\pdfoutput=1

\documentclass[english,11pt]{article}
\input{macros}

\newcommand{\repourl}{\url{https://github.com/nathanweill/pseudolabel-uda-experiments}}

\crefname{fact}{Fact}{Facts}
\renewcommand{\theHtheorem}{\thesection.\arabic{theorem}}
\renewcommand{\theHlemma}{\thesubsection.\arabic{lemma}}
\renewcommand{\theHcorollary}{\thesection.\arabic{corollary}}

\renewcommand{\theHfact}{\thesection.\arabic{fact}}

\begin{document}
\title{\Large Pseudo-Labeling for Unsupervised Domain Adaptation with Kernel GLMs}

\author{
    Nathan Weill\thanks{Department of IEOR, Columbia University. Email: \texttt{nmw2142@columbia.edu}}
    \quad 
    Kaizheng Wang\thanks{Department of IEOR and Data Science Institute, Columbia University. Email: \texttt{kw2934@columbia.edu}}
}

\date{\vspace{-1em} This version: November 2026}

\maketitle

\begin{abstract}
We propose a principled framework for unsupervised domain adaptation under covariate shift in kernel Generalized Linear Models (GLMs), encompassing kernelized linear, logistic, and Poisson regression with ridge regularization. Our goal is to minimize prediction error in the target domain by leveraging labeled source data and unlabeled target data, despite differences in covariate distributions. We partition the labeled source data into two batches: one for training a family of candidate models, and the other for building an imputation model. This imputation model generates pseudo-labels for the target data, enabling robust model selection. We establish non-asymptotic excess-risk bounds that characterize adaptation performance through an ``effective labeled sample size,'' explicitly accounting for the unknown covariate shift. Experiments on synthetic and real datasets demonstrate consistent performance gains over source-only and importance weighting baselines.
\end{abstract}

\noindent \textbf{Keywords:} Covariate shift, kernel GLMs, pseudo-labeling, unsupervised domain adaptation.

\section{Introduction}

The standard machine learning framework relies on the fundamental assumption that training and test datasets are drawn from the same underlying distribution \citep{Vapnik1999-VAPTNO}. However, in realistic deployment scenarios, this assumption is frequently violated. A canonical instance of distribution shift is \emph{covariate shift} \citep{SHIMODAIRA2000227, ben-david_theory_2010}, where the marginal distribution of features varies between the source and target domains, while the conditional distribution of the output label remains stable. This phenomenon is pervasive in high-stakes applications, from personalized medicine to computer vision. It is well-documented that covariate shift can severely degrade predictive performance \citep{koh2021wildsbenchmarkinthewilddistribution}.

In many applications, obtaining labels for the target population is prohibitively expensive, whereas collecting unlabeled covariates is straightforward. This motivates the study of unsupervised domain adaptation (UDA), where the learning algorithm must generalize to the target domain using labeled source data and only unlabeled target covariates. We investigate UDA for ridge-regularized Generalized Linear Models (GLMs) and their kernel extensions. In this regime, simply applying a model trained on the source is often suboptimal, even if the model is well-specified. Standard learning algorithms tune regularization hyperparameters to balance bias and variance on the \emph{source} distribution, ignoring regions of the covariate space where the target density is high but source data is scarce. Consequently, achieving optimal target performance requires an adaptive strategy that leverages unlabeled target data to recalibrate the estimator. 

\paragraph{Contributions}
We address this challenge by developing a pseudo-labeling framework tailored for kernel GLMs. The approach enables model selection for the target domain using only unlabeled data. Our contributions are two-fold:

\begin{itemize}
    \item \textbf{Methodology:} We propose a procedure that partitions the labeled source data into two batches. The first batch is used to train a family of candidate models, while the second trains an imputation model to generate pseudo-labels for the target covariates. These pseudo-labels serve as a proxy for ground truth, allowing us to select the candidate model that minimizes the estimated target risk. We provide precise guidance for tuning the imputation model to optimize the final selection.
    
    \item \textbf{Theory:} We derive non-asymptotic excess risk bounds for ridge-regularized kernel GLMs in the UDA setting. Our bounds quantify the information transfer via an ``effective labeled sample size'' that accounts for the discrepancy between source and target spectral properties. This theoretical result demonstrates that our method automatically adapts to the unknown covariate shift.
\end{itemize}

\paragraph{Related Work}

The theoretical landscape of transfer learning largely focuses on settings where the learner has access to labeled target data or a known target covariate distribution \citep{ben-david_theory_2010,kpotufe2021marginal}. Such assumptions simplify adaptation by allowing for direct risk estimation or precise density ratio calculation. In contrast, we address the strictly unsupervised regime, where the target distribution is unknown and observed only through a finite set of unlabeled covariates. This necessitates inferring the structure of the shift and adapting the estimator solely from the geometry of the unlabeled data, without the guidance of target labels.

The necessity of adaptation under covariate shift depends on the modeling regime. In the \emph{parametric} setting, if the model is well-specified, standard Maximum Likelihood Estimation (MLE) on the source is asymptotically efficient for the target, rendering adaptation unnecessary \citep{mleisallyouneed}. However, if the parametric model is mis-specified, the source MLE is biased for the target; the standard remedy is importance weighting \citep{SHIMODAIRA2000227, sugiyama_direct_2008} or its doubly robust variants utilizing pseudo-outcomes \citep{JMLR:v24:22-0700}. While statistically consistent, these methods often suffer from high variance and the difficulty of estimating density ratios in high dimensions.

In the \emph{nonparametric} regime we consider, the situation differs. Even if the model is well-specified (e.g., the true function lies in an RKHS), regularization is required to ensure finite-sample generalization. Recent works have established theoretical guarantees for reweighting-based adaptation in kernel ridge regression (KRR) \citep{ma_optimally_2023} and kernel methods with more general losses \citep{feng_towards_2023}. However, these approaches remain dependent on the existence and accurate estimation of density ratios. An alternative that avoids density estimation is to adapt the regularization directly. \citet{wang2023pseudo} addressed this for KRR using pseudo-labels to estimate the target error for selecting the optimal regularization. Our work operates in this same nonparametric regime and deliberately keeps the split--impute--select template of \citet{wang2023pseudo}; our contribution is to prove that it remains valid for kernel GLMs (logistic, Poisson, \dots), a case left open there, where the estimator has no closed form and the analysis of the KRR proof no longer applies (see \cref{sec: wang comparison} for a detailed comparison).

The strategy of filling missing labels with model predictions has driven empirical success in deep learning \citep{lee2013pseudolabel,sohn2020fixmatch}. Theoretical analysis of pseudo-labeling has mostly focused on classification under cluster assumptions or margin conditions \citep{liu2021cycle,selftrainingwei,cai2021theory}, which do not translate to the continuous responses covered by our GLM framework.

\paragraph{Notations}
The constants $c_1, c_2, C_1, C_2, \cdots$ may differ from line to line. We use the symbol $[n]$ as a shorthand for $\{1,2, \cdots, n\}$ and $|\cdot|$ to denote the absolute value of a real number or cardinality of a set. For nonnegative sequences $\left\{a_n\right\}_{n=1}^{\infty}$ and $\left\{b_n\right\}_{n=1}^{\infty}$, we write $a_n \lesssim b_n$ or $a_n=O\left(b_n\right)$ if there exists a positive constant $C$ such that $a_n \leq C b_n$. In addition, we write $a_n \asymp b_n$ if $a_n \lesssim b_n$ and $b_n \lesssim a_n ; a_n=o\left(b_n\right)$ if $a_n=O\left(c_n b_n\right)$ for some $c_n \rightarrow 0$. Notations with tildes (e.g., $\widetilde{O}$) hide logarithmic factors. For a matrix $\boldsymbol{A}$, we use $\|\boldsymbol{A}\|_2=\sup _{\|\boldsymbol{x}\|_2=1}\|\boldsymbol{A} \boldsymbol{x}\|_2$ to denote its spectral norm. For a bounded linear operator $\boldsymbol{A}$ between two Hilbert spaces $\mathbb{H}_1$ and $\mathbb{H}_2$, we define its operator norm $\|\boldsymbol{A}\|=\sup _{\|\boldsymbol{u}\|_{\mathbb{H}_1}=1}\|\boldsymbol{A} \boldsymbol{u}\|_{\mathbb{H}_2}$. For a bounded, self-adjoint linear operator $\boldsymbol{A}$ mapping a Hilbert space $\mathbb{H}$ to itself, we write $\boldsymbol{A} \succeq 0$ if $\langle\boldsymbol{u}, \boldsymbol{A} \boldsymbol{u}\rangle \geq 0$ holds for all $\boldsymbol{u} \in \mathbb{H}$; in that case, we define $\|\boldsymbol{u}\|_{\boldsymbol{A}}=\sqrt{\langle\boldsymbol{u}, \boldsymbol{A} \boldsymbol{u}\rangle}$. For any $\boldsymbol{u}$ and $\boldsymbol{v}$ in a Hilbert space $\mathbb{H}$, their tensor product $\boldsymbol{u} \otimes \boldsymbol{v}$ is a linear operator from $\mathbb{H}$ to itself that maps any $\boldsymbol{w} \in \mathbb{H}$ to $\langle\boldsymbol{v}, \boldsymbol{w}\rangle \boldsymbol{u}$. Define $\|X\|_{\psi_\alpha}=\sup _{p \geq 1}\left\{p^{-1 / \alpha} \mathbb{E}^{1 / p}|X|^p\right\}$ for a random variable $X$ and $\alpha \in\{1,2\}$. If $\boldsymbol{X}$ is a random element in a separable Hilbert space $\mathbb{H}$, then we let $\|\boldsymbol{X}\|_{\psi_\alpha}=\sup _{\|\boldsymbol{u}\|_{\mathbb{H}}=1}\|\langle\boldsymbol{u}, \boldsymbol{X}\rangle\|_{\psi_\alpha}$, $\alpha \in\{1,2\}$.

\section{Problem setup}

\subsection{Generalized linear models under covariate shift}

Let $\{(\bx_i,y_i)\}_{i=1}^n$ and $\{(\bx_{0i},y_{0i})\}_{i=1}^{n_0}$ denote samples drawn from the source and target distributions, respectively. Here, $\bx_i, \bx_{0i} \in \mathcal{X}$ are covariates in a compact metric space $\cX$, and $y_i, y_{0i} \in \mathbb{R}$ are real-valued responses. In the Unsupervised Domain Adaptation (UDA) setting, the target labels $\{y_{0i}\}_{i=1}^{n_0}$ are unobserved. Our goal is to learn a predictive model using the labeled source data $\{(\bx_i,y_i)\}_{i=1}^n$ and the unlabeled target covariates $\{\bx_{0i}\}_{i=1}^{n_0}$ that achieves low error on the target population.

We assume that both domains share a common Generalized Linear Model (GLM) \citep{glm_article} in a Reproducing Kernel Hilbert Space (RKHS) $\cF$. Specifically, conditioned on $\bx$, the response $y$ is drawn from an exponential family distribution whose density function has the form:
\begin{equation}\label{eq:glm-loglik}
p(y | \bx) = h(y,\tau) \exp \bigg(
\frac{ y f^*(\bx) - \logpar(f^*(\bx))}{d(\tau)}
\bigg)
,
\end{equation}
where $\logpar: \mathbb{R} \to \mathbb{R}$ is a known, strictly convex log-partition function; $h$ and $d$ are known functions; $\tau>0$ is a (possibly unknown) dispersion parameter; and $f^*$ belongs to the class $\cF$.
The conditional mean of the response is given by the canonical link:
\begin{equation}\label{eq:glm-mean}
    \mathbb{E}[y \mid \bx] = \logpar'(f^*(\bx)).
\end{equation}
Up to an affine transform,
the corresponding negative log-likelihood function is:
\begin{equation}\label{eq:glm-loss}
    \ell(f(\bx), y) := \logpar(f(\bx)) - y f(\bx), \quad \forall f \in \cF.
\end{equation}
Classical examples include \emph{linear regression} [$\logpar(u)=u^2/2$], \emph{logistic regression} [$\logpar(u)=\log(1+e^u)$], and \emph{Poisson regression} [$\logpar(u)=e^u$].

The RKHS $\cF$ is generated by a continuous, symmetric, positive semi-definite kernel $K: \mathcal{X} \times \mathcal{X} \rightarrow \mathbb{R}$. By the Moore-Aronszajn Theorem \citep{arons}, there exists a feature space (Hilbert space) $\mathbb{H}$ with inner product $\langle\cdot, \cdot\rangle_{\mathbb{H}}$ and a mapping $\Phi: \mathcal{X} \rightarrow \mathbb{H}$ such that $K(\bx, \bx') = \langle\Phi(\bx), \Phi(\bx')\rangle_{\mathbb{H}}$ for all $\bx, \bx' \in \mathcal{X}$. The RKHS is defined as the class of linear functionals on the feature space:
\begin{equation}\label{eq: rkhs isometry}
    \cF = \left\{ f_{\btheta}: \mathcal{X} \to \mathbb{R} \mid f_{\btheta}(\bx) = \langle \Phi(\bx), \btheta \rangle_{\mathbb{H}}, \, \btheta \in \mathbb{H} \right\}.
\end{equation}

Equipped with the norm $\|f_{\btheta}\|_{\cF} := \|\btheta\|_{\mathbb{H}}$, $\cF$ is a Hilbert space isometric to $\mathbb{H}$. Below are some frequently used kernels and associated RKHSs \citep{wainwright}.

\begin{example}[Linear and affine kernels]\label{ex: linear kernel}
Let $\cX \subset \RR^d$. The linear kernel $K(\bz, \bw) = \bz^\top \bw$ and the affine kernel $K(\bz, \bw) = 1+\bz^\top \bw$ yield the RKHS $\cF$ of all linear and affine functions respectively, of dimension $D=d$ and $D=d+1$ respectively.
\end{example}

\begin{example}[Polynomial kernels]\label{ex: polynomial kernel}
Let $\cX \subset \RR^d$. The inhomogeneous polynomial kernel of degree $m\ge 2$ is $K(\bz, \bw) = (1+\bz^\top \bw)^m$. It yields the RKHS $\cF$ of all polynomials of degree $m$ or less, of dimension $D=\binom{d+m}{m}$.
\end{example}

The function spaces in \Cref{ex: linear kernel,ex: polynomial kernel} are finite-dimensional. The following example is infinite-dimensional. 

\begin{example}[First-order Sobolev kernels]\label{ex: sobolev kernel}
Let $\cX = [0,1]$. The kernel $K(\bz, \bw) = \min\{\bz,\bw\}$ and its constant-augmented version $K(\bz, \bw) = 1+\min\{\bz,\bw\}$ yield respectively the first-order Sobolev spaces
\begin{align*}
\cF &= \acc{f : [0,1] \rightarrow \RR \mid f(0)=0, \int_0^1 \abs{f'(x)}^2dx < \infty}, & \norm{f}_{\cF}^2&=\int_0^1 \abs{f'(x)}^2dx,\\
\cF &= \acc{f : [0,1] \rightarrow \RR \mid \int_0^1 \abs{f'(x)}^2dx < \infty}, & \norm{f}_{\cF}^2&=f(0)^2+\int_0^1 \abs{f'(x)}^2dx.
\end{align*}
\end{example}

We adopt the following standard assumptions for the UDA setting:

\begin{assumption}[GLM in RKHS and Random Design]\label{setup}
The source dataset $\cD = \{(\bx_i, y_i)\}_{i=1}^n$ and target dataset $\cD_0 = \{(\bx_{0i}, y_{0i})\}_{i=1}^{n_0}$ are independent, with samples drawn i.i.d. within each domain.
\begin{itemize}
    \item \textbf{Covariate Shift:} The covariate distributions for source and target are $\cP$ and $\cQ$, respectively, with $\supp(\cQ) \subseteq \supp(\cP)$. The second-moment operators $\bSigma = \mathbb{E}_{\bx \sim \cP}[\Phi(\bx) \otimes \Phi(\bx)]$ and $\bSigma_0 = \mathbb{E}_{\bx \sim \cQ}[\Phi(\bx) \otimes \Phi(\bx)]$ are trace-class.
    \item \textbf{Well-Specified Model:} There exists a true function $f^* = f_{\btheta^*} \in \cF$ that generates the responses in both domains according to the distribution in \eqref{eq:glm-loglik}.
    \item \textbf{Finite Noise Variance:} Define noise variables $\res_i=\logpar'\paren{f^*(\boldsymbol{x}_i)}-y_i$ and $\res_{0 i}=\logpar'\paren{f^*(\boldsymbol{x}_{0i})}-y_{0 i}$. Conditioned on the covariates, $\{\varepsilon_i\}_{i=1}^n$ and $\{\varepsilon_{0i}\}_{i=1}^{n_0}$ have finite second moments.
\end{itemize}
\end{assumption}

\Cref{setup} is a \emph{pure} covariate shift assumption: the same conditional law $p(y\mid\bx)$ holds in both domains. Our guarantees do not cover conditional or concept shift, under which source labels need not carry information about the target conditional and the pseudo-label risk proxy may be biased; this restriction is shared by the covariate-shift literature we build on \citep{SHIMODAIRA2000227,ma_optimally_2023,wang2023pseudo}.

\subsection{Ridge regularization for kernel GLMs}

For any predictive model $f \in \mathcal{F}$, we evaluate its performance via the expected negative log-likelihood on the target population:
\begin{align}\label{eq: glm risk}
    R(f) = \mathbb{E}_{(\bx, y) \sim \cQ} [\logpar(f(\bx)) - y f(\bx)]
         = \mathbb{E}_{\bx \sim \cQ} [\logpar(f(\bx)) - f(\bx) \logpar'(f^*(\bx))]. 
\end{align}
Under \cref{setup}, $R(f)$ is minimized by $f^*$. We define the \emph{excess risk} as:
\begin{align}\label{eq: glm excess risk}
    \cR(f) = R(f) - R(f^*)
           = \mathbb{E}_{\bx \sim \cQ} \Big[ \logpar(f(\bx)) - \logpar(f^*(\bx))
       - \logpar'(f^*(\bx))(f(\bx) - f^*(\bx)) \Big].
\end{align}
This quantity, which is the expected Bregman divergence defined by $\logpar$, measures the discrepancy between the predicted conditional distribution and the truth.

A standard approach to estimating $f^*$ is to minimize the regularized empirical risk on the source data:
\begin{equation}\label{regularized glm bis}
    \hat{f}_\lambda = \argmin_{f \in \mathcal{F}} \left\{ \frac{1}{n} \sum_{i=1}^n \left( \logpar(f(\bx_i)) - y_i f(\bx_i) \right) + \frac{\lambda}{2}\|f\|_{\mathcal{F}}^2 \right\},
\end{equation}
where $\lambda > 0$ is a regularization hyperparameter balancing data fit and smoothness. By the Representer Theorem \citep{scholkopf2001generalized}, the solution lies in the span of the source data: $\hat{f}_\lambda(\cdot) = \sum_{i=1}^n \hat{\alpha}_i K(\bx_i, \cdot)$. The coefficient vector $\hat{\balpha} \in \mathbb{R}^n$ is the solution to the convex optimization problem:
\begin{equation}\label{eq: alpha space}
    \min_{\balpha \in \RR^n} \left\{ \frac{1}{n} \left( \mathbf{1}^{\top}\ba(\balpha) - \by^{\top}\bK \balpha \right) + \frac{\lambda}{2} \balpha^{\top} \bK \balpha \right\},
\end{equation}
where $\bK_{ij} = K(\bx_i, \bx_j)$, $\by = (y_1, \dots, y_n)^\top$, and $\ba(\balpha) = (\logpar([\bK \balpha]_1), \dots, \logpar([\bK \balpha]_n))^\top$.

For source-only learning, $\lambda$ is typically selected to minimize the source validation error. However, under covariate shift, an optimal penalty for the source may be suboptimal for the target, particularly if the target distribution is concentrated in directions where the source provides weak coverage. 
Ideally, if we had access to target labels $\{y_{0i}\}_{i=1}^{n_0}$, we would select the parameter $\lambda$ minimizing the empirical target risk:
\begin{equation}\label{eq: hold-out validation}
    \frac{1}{n_0}\sum_{i=1}^{n_0}\paren{\logpar(\hat f_\lambda(\bx_{0i}))-y_{0i}\hat f_\lambda(\bx_{0i})}.
\end{equation}
Due to missing labels, we require a strategy to estimate the target performance using only the unlabeled covariates.

\section{Methodology}\label{section: methodology}

We propose an adaptive method for unsupervised domain adaptation with kernel ridge GLM (KRGLM). The core idea is to construct a computable proxy for the empirical risk \eqref{eq: hold-out validation}. Note that the conditional mean response $\EE (y_{0i} | \bx_{0i} ) = \logpar'(f^* (\bx_{0i} ) )$ is given by the true model $f^*$. We will build an auxiliary ``imputation model'' $\tilde{f}$ trained on part of the labeled source data and then use the estimated conditional mean $\tilde{y}_{0i} = \logpar'(\tilde{f} (\bx_{0i} ) )$ to replace $y_{0i}$ in \eqref{eq: hold-out validation}. The full procedure is detailed in \cref{alg:pseudo-krr}.

\begin{algorithm}[ht]
\caption{Pseudo-Labeling for Kernel GLM}
\label{alg:pseudo-krr}
\begin{algorithmic}[1]
\STATE \textbf{Input}: Source data $\cD=\{(\bx_i, y_i)\}_{i=1}^n$, unlabeled target dataset $\cD_0=\{\bx_{0i}\}_{i=1}^{n_0}$, training split size $n_1 \in [n]$, candidate penalty grid $\Lambda \subset (0,\infty)$, imputation penalty $\tilde{\lambda} > 0$.

\STATE \textbf{Step 1 (Data Splitting)}. Randomly partition the source data into two disjoint subsets: $\mathcal{D}_1$ of size $n_1$ (for candidate training) and $\mathcal{D}_2$ of size $n_2 = n - n_1$ (for imputation).

\STATE \textbf{Step 2 (Training)}. 
\STATE \quad (a) \textbf{Candidate Models}: For each $\lambda \in \Lambda$, train a candidate $\hat{f}_{\lambda}$ on $\mathcal{D}_1$ by solving the program
\[
\min_{f \in \mathcal{F}} \bigg\{ \frac{1}{n_1}\sum_{(\bx,y)\in \mathcal{D}_1}(\logpar(f(\bx)) - y f(\bx)) + \frac{\lambda}{2}\|f\|_{\mathcal F}^2 \bigg\}.
\]
\STATE \quad (b) \textbf{Imputation Model}: Train the imputation model $\tilde f$ on $\mathcal{D}_2$ with penalty $\tilde{\lambda}$ by solving the program
\[
\min_{f \in \mathcal{F}} \bigg\{ \frac{1}{n_2}\sum_{(\bx,y)\in \mathcal{D}_2}(\logpar(f(\bx)) - y f(\bx)) + \frac{\tilde\lambda}{2}\|f\|_{\mathcal F}^2 \bigg\}.
\]

\STATE \textbf{Step 3 (Pseudo-Labeling)}. Generate pseudo-labels for the target data using the conditional mean of the imputation model:
\[
\tilde y_{0i} = \logpar'(\tilde f(\bx_{0i})), \quad \text{for } i = 1, \dots, n_0.
\]

\STATE \textbf{Step 4 (Model Selection)}. Select the candidate $\hat{\lambda} \in \Lambda$ that minimizes the pseudo-target risk:
\begin{equation}\label{eq: risk proxy}
\frac{1}{n_0}\sum_{i=1}^{n_0}\paren{\logpar(\hat f_\lambda(\bx_{0i})) - \tilde{y}_{0i} \hat f_\lambda(\bx_{0i})} .
\end{equation}

\STATE \textbf{Output}: The selected model $\hat{f} = \hat{f}_{\hat\lambda}$.
\end{algorithmic}
\end{algorithm}

\paragraph{Soft vs. hard labeling.} 
For classification tasks such as logistic regression, our method utilizes ``soft'' pseudo-labels $\tilde{y}_{0i} = 1/ (1 + e^{-\widetilde{f}(\bx_{0i}) }) \in (0, 1)$---representing the estimated conditional probability---rather than ``hard'' labels $
\mathbf{1}(\widetilde{f}(\bx_{0i}) \geq 0) \in 
\{0, 1\}$. This distinction is critical. While hard pseudo-labeling is common in semi-supervised learning to enforce low-entropy decision boundaries \citep{lee2013pseudolabel,sohn2020fixmatch,liu2021cycle,selftrainingwei,cai2021theory}, our objective is to minimize the negative log-likelihood to recover the true conditional distribution. Using hard labels would destroy the calibration information necessary for minimizing the log-likelihood, introducing severe bias into the risk proxy \eqref{eq: risk proxy} by forcing candidate models to fit extreme probabilities (0 or 1) rather than the true underlying uncertainty.

\paragraph{Implementation and tuning strategy.}
The procedure separates model training (Step 2a) from validation (Steps 2b--4): the candidates are scored on the target covariates $\mathcal{D}_0$ with pseudo-labels produced by an imputer trained on $\mathcal{D}_2$, which is independent of the data $\mathcal{D}_1$ used to train them. The computational cost is dominated by solving the kernel GLM optimization, which is efficient for standard losses.
Our theoretical analysis (presented in \cref{theory}) provides specific guidance on parameter choices:
\begin{itemize}
    \item \textbf{Split ratio}: We typically set $n_1 = \lfloor n/2 \rfloor$, using half the source data for training candidates and half for imputation. This balanced split is standard in split-sample validation methods.

    \item \textbf{Candidate grid ($\Lambda$)}: The grid should be sufficiently dense to cover the bias-variance spectrum. We typically set $\Lambda$ as a geometric sequence ranging from $O(n^{-1})$ (low bias, high variance) to $O(1)$ (high bias, low variance), with a cardinality of roughly $O(\log n)$. In practice, we construct it heuristically from source data only: we inspect the source regularization curve, choose a range spanning its high-variance and high-bias regimes, and extend it by about one order of magnitude beyond the low-penalty end.

    \item \textbf{Imputation penalty ($\tilde{\lambda}$)}: This is the critical hyperparameter. Unlike standard supervised learning, where $\tilde{\lambda}$ would be tuned to minimize source prediction error, here it must be chosen to optimize the \emph{model selection} capability of the pseudo-labels. Our theory (\cref{theory}) suggests setting $\tilde{\lambda} \asymp n^{-1}$ up to logarithmic factors. This rate typically corresponds to ``undersmoothing'' the imputation model---prioritizing low bias in the pseudo-labels over low variance---which is crucial for accurately ranking the candidate models on the target domain. Concretely, we use the rule $\tilde\lambda=\min\Lambda$, which requires no target labels and no knowledge of the shift or of problem-dependent constants. Our experiments in \cref{app: ablation} confirm the need for undersmoothing and show that the final risk is insensitive to the exact value of $\tilde\lambda$ within the undersmoothed range.
\end{itemize}

\section{Theoretical guarantees on the adaptivity}\label{theory}

This section develops theoretical guarantees for \cref{alg:pseudo-krr}. We show that the resulting estimator adapts both to the target distribution's structure and to the unknown covariate shift. Our analysis introduces an effective sample size that quantifies how informative the labeled source data are for the target regression task.

\subsection{Assumptions}
Our analysis of kernel ridge GLM under covariate shift (\cref{setup}) necessitates only mild regularity conditions, that one can verify are satisfied in most common scenarios, as discussed below.

First, we need some assumptions on the kernel function $K$. Mercer's theorem (see for instance \cite{wainwright}) guarantees the existence of an orthonormal eigen-expansion of the PSD kernel with respect to a measured function class when it is continuous and satisfies the Hilbert-Schmidt condition.

\begin{assumption}[Kernel regularity]\label{kernel assumptions}
We assume Mercer's theorem holds under both source and target distributions, impose a uniform boundedness assumption on the eigenfunctions, and a polynomial decay of the eigenvalues:
    \begin{itemize}
        \item For any $\bx, \bz \in \cX$, $K(\bx, \bz)$ can be expanded as
        \begin{equation*}
        \sum_{j=1}^{\infty} \mu_j \phi_j(\bx) \phi_j(\bz) = \sum_{j=1}^{\infty} \mu_{0j} \phi_{0j}(\bx) \phi_{0j}(\bz)
        \end{equation*}
        where $\left(\mu_j\right)_{j=1}^{\infty}$ and $\left(\mu_{0j}\right)_{j=1}^{\infty}$ are sequences of positive eigenvalues \footnote{We order them as \(\mu_1\ge\mu_2\ge\cdots>0\) and \(\mu_{01}\ge\mu_{02}\ge\cdots>0\). Note that they are also the eigenvalues of the second moment operators $\bSigma$ and $\bSigma_0$ respectively.}, $\left(\phi_j\right)_{j=1}^{\infty}$ is an orthonormal sequence of eigenfunctions in $L^2(\mathcal{X} ; \cP)$, and $\left(\phi_{0j}\right)_{j=1}^{\infty}$ is an orthonormal sequence of eigenfunctions in $L^2(\mathcal{X} ; \cQ)$. The convergence of the infinite series holds absolutely and uniformly.
        \item $\exists M_1<\infty, \forall j, \norm{\phi_j}_{\infty} \le M_1, \norm{\phi_{0j}}_{\infty} \le M_1$. 
        \item $\exists A, A'>0,\, \exists \alpha \ge1, \,\forall j\ge 1, \mu_j \le A j^{-2}, \, \mu_{0j} \le A' j^{-2\alpha} $.
    \end{itemize}
\end{assumption}

The polynomial upper-bound on the spectral decay is a smoothness condition for the kernel: a faster decay (larger $\alpha$) corresponds to a smoother RKHS, as
high-frequency eigendirections contribute less to the kernel. It is satisfied by the first-order Sobolev kernel on $\cX=[0,1]$ (\cref{ex: sobolev kernel}) with $\alpha=1$ under $\cQ = \cU[0,1]$, and so is its constant-augmented version, since adding a constant is a rank-one perturbation of the integral operator, which preserves the eigenvalue decay by Weyl's inequality. It is also trivially satisfied for any $\alpha$ by finite-rank kernels such as linear and polynomial kernels (\Cref{ex: linear kernel,ex: polynomial kernel}) which have at most $D$ non-zero eigenvalues.

The uniform boundedness of eigenfunctions holds for linear and polynomial kernels on compact domains, as well as for the first-order Sobolev kernel and its constant-augmented version, whose eigenfunctions are sinusoidal. 


Under \cref{kernel assumptions}, the feature map induced by the kernel $K$ can be written as 
\begin{equation*}
\Phi :\bx \mapsto \left(\sqrt{\mu_1} \phi_1(\bx), \, \sqrt{\mu_2} \phi_2(\bx), \, \ldots\right).
\end{equation*}

These assumptions on the kernel imply that the covariates are bounded almost surely: there exists a constant $M_2 < \infty$ such that $\norm{\Phi(\bx)} \le M_2$ holds almost surely for $\bx$ from $\cP$ and $\cQ$.

We also assume bounded noise and signal. The noise is only assumed sub-exponential instead of sub-Gaussian to include Poisson regression in our analysis.

\begin{assumption}[Sub-exponential noise]\label{subexponentialassumption}
Conditioned on $\bX$, the noise vector \(\bres\) satisfies $\norm{\bres}_{\psi_1} \le \resv < \infty$.
\end{assumption}


\begin{assumption}[Bounded signal]\label{trueparam}
$\norm{f^*}_{\cF} \le R < \infty$.
\end{assumption}

\begin{remark}
\Cref{kernel assumptions,trueparam} imply that for any $\bx \in \supp(\cP)$, $|f^*(\bx) | \le C:=RM_2$.
\end{remark} 

Finally, we assume some smoothness and convexity conditions for the log-partition function:

\begin{assumption}[Local strong convexity and boundedness]\label{convexityassumption}
    We assume that the log-partition function $\logpar$ is strictly convex, smooth, and strongly convex in compact regions: for any $\rho >0$, we have $k_{\rho}:= \inf_{|z| \le (\rho +1)C} \logpar''(z) >0$. Define $k=k_1$. Moreover, $\logpar''$ is assumed to be upper bounded and Lipschitz on compact regions: let $K_{\rho}:= \sup_{|z| \le (\rho +1)C} \logpar''(z)<\infty$, $L_{\rho}:= \sup_{|z| \le (\rho +1)C} \logpar'''(z)<\infty$, and define $K=K_1, \, L=L_1$. Then, $\forall |z| \le 2C, k\le \logpar''(z) \le K$ and $\forall |z_1|, |z_2| \le 2C, \abs{\logpar''(z_1)-\logpar''(z_2)} \le L\abs{z_1-z_2}$. 
\end{assumption}

This assumption is satisfied by common GLMs. \Cref{tab:glm_properties} provides simple bounds for the corresponding GLM log-partition constants in the three canonical examples.

\begin{table}[ht]
  \caption{Curvature and smoothness constants for the log-partition function $a$ on $[-2C, 2C]$.}
  \label{tab:glm_properties}
  \centering
  \small
  \begin{tabular*}{\columnwidth}{@{\extracolsep{\fill}} lcccc @{}}
    \toprule
    Model & $a(z)$ & $k$ & $K$ & $L$ \\
    \midrule
    Linear
      & $\frac{z^2}{2}$
      & $1$
      & $1$
      & $0$ \\
    Logistic
      & $\log(1 + e^z)$
      & $\frac{e^{2C}}{(1+e^{2C})^2}$
      & $\frac{1}{4}$
      & $\frac{1}{10}$ \\
    Poisson
      & $e^z$
      & $e^{-2C}$
      & $e^{2C}$
      & $e^{2C}$ \\
    \bottomrule
  \end{tabular*}
\end{table}

\subsection{Main results}\label{sec: main results}

Our main theorem shows that our method adapts to the unknown covariate shift.

\begin{theorem}[Excess risk]\label{main theorem}
    Let \Cref{setup,kernel assumptions,subexponentialassumption,convexityassumption,trueparam} hold. Define for $\mu^2=M_1^6\sigma^2$
    \begin{equation}\label{effective sample size}
        n_{\text{eff}} = \sup \bigg\{ t\le n ~\bigg|~ t\bSigma_0 \preceq n\bSigma + \frac{\mu^2}{k}\bI \bigg\}.
    \end{equation}
Choose any $\delta \in (0, 1/e]$. Consider \cref{alg:pseudo-krr} with $n_1=n/2$,  $\tilde \lambda =\mu^2\log^7(n)\log(n_0/\delta)/n$, and 
\begin{equation*}\Lambda = \acc{2^{j-1}\mu^2/n : 1\le j\le \lceil \log_2n\rceil +1}.
\end{equation*}
    Then there exists a function $\zeta$ polylogarithmic in $(n, n_0, \delta^{-1})$ multiplied by a polynomial in $\frac{K}{k}$, such that \wpd:

    \begin{align}\label{eq: main theorem}
    R(\hat f) -R(f^*) 
     \le \zeta \cdot \inf_{\rho >0} \left\{ \rho R^2 + \frac{\resv^2}{n_{\text{eff}}} \sum_{j=1}^{\infty} \frac{\mu_{0j}}{k\mu_{0j}+\rho} \right\} + \zeta \cdot R^2(1+L^2)\mu^2 \left( \frac{1}{n_{\text{eff}}} + \frac{1}{n_0} \right).
    \end{align}

\end{theorem}

The proof is deferred to \cref{proof of main theorem}.

\begin{remark}[Interpretation of Effective Sample Size]
The quantity $n_{\text{eff}}$ defined in \cref{effective sample size} serves as a data-dependent measure of the covariate shift. It quantifies the extent to which the source covariance $\bSigma$ ``covers'' the target covariance $\bSigma_0$. Crucially, this alignment is not merely a check of input support, but is measured through the \textit{geometry of the kernel}: the covariance inequality dictates how well the source data captures variance along the principal directions of the RKHS relative to the target distribution.
 
\begin{itemize}
    \item  If the source dominates the target (i.e., $\bSigma_0 \preceq B \bSigma$), we have $n_{\text{eff}} \ge n/B$, recovering standard fast rates.
    \item If the target $\bSigma_0$ has high energy in eigen-directions where the source $\bSigma$ decays rapidly (a mismatch in kernel geometry), $n_{\text{eff}} \ll n$, reflecting the difficulty of adaptation. For instance, consider the first-order Sobolev space with $\cP = \cU[0,1]$ and $\cQ = \delta_{0.5}$. It is proved in Section B.2 of \citet{wang2023pseudo} that $n_{\text{eff}} \asymp \sqrt{n}$.  
\end{itemize}
\end{remark}

The infimum over $\rho$ in \eqref{eq: main theorem} captures the optimal bias-variance tradeoff for KRGLM on $n_{\mathrm{eff}}$ labeled target samples: the term $\rho R^2$ is the regularization bias, which increases with $\rho$, while $\frac{\sigma^2}{n_{\mathrm{eff}}}\sum_{j=1}^{\infty}\frac{\mu_{0j}}{k\mu_{0j}+\rho}$ is the variance, which decreases with $\rho$ and reflects how well the target eigendirections are captured given $n_{\mathrm{eff}}$ effective samples. The infimum over $\rho$ selects the regularization level that optimally balances these two competing terms, yielding the same rate as if one had direct access to $n_{\mathrm{eff}}$ labeled target samples. The second term in \eqref{eq: main theorem} reflects the extra error incurred by the imperfect pseudo-labels. The following corollary shows that this overhead is negligible compared to the first term in our non-parametric setting.

\begin{corollary}[Convergence Rate with Effective Sample Size]
\label{cor:poly_decay_neff}
Given the polynomial spectral decay assumption $\mu_{0j} \le A j^{-2\alpha}$, and provided $n$ and $n_0/\log n$ are sufficiently large, the excess risk satisfies with probability at least $1 - n^{-1}$:
\begin{align}
    R(\hat f) - R(f^*) \lesssim n_{\text{eff}}^{-\frac{2\alpha}{2\alpha + 1}} + n_0^{-1}.
\end{align}
\end{corollary}

The $\lesssim$ only hides polylogarithmic factors in $n$ and $n_0$, and constants depending on the GLM parameters. The proof is deferred to \cref{proof of main coro}.

\begin{remark}[Required number of target covariates]\label{rem: n0}
The pseudo-label term $n_0^{-1}$ is negligible as soon as $n_0 \gtrsim n_{\mathrm{eff}}^{2\alpha/(2\alpha+1)}$. This requirement is sublinear in $n_{\mathrm{eff}}$, hence in $n$, and it moves in the opposite direction to the difficulty of the problem: a harder shift (smaller $n_{\mathrm{eff}}$) slows the estimation rate and \emph{weakens} the demand on $n_0$. For $\alpha=1$ and $n=10^4$, a hard shift with $n_{\mathrm{eff}}\asymp\sqrt n$ only needs $n_0\gtrsim 22$, whereas a mild shift with $n_{\mathrm{eff}}\asymp n$ needs $n_0\gtrsim 464$. The constraint thus binds for mild shifts with scarce target covariates (e.g., a rare subpopulation), where the bound degrades to $O(n_0^{-1})$ whatever the amount of source data; see \cref{app: n0 requirement}.
\end{remark}

\begin{remark}[Optimality and adaptivity]
The bound established in \cref{cor:poly_decay_neff} is sharp. It is easily seen in the bounded density ratio regime $\mathrm{d}\cQ / \mathrm{d}\cP \le B$ for some constant $B \ge 1$, implying $n_{\text{eff}} \ge n/B$. In this setting, we recover the rate $(n/B)^{-\frac{2\alpha}{2\alpha+1}}$, which matches the minimax lower bound for KRR with square loss derived in Theorem~2 of \citet{ma_optimally_2023}. While their lower bound holds for estimators with full knowledge of $\cQ$ and $B$, our Algorithm~\ref{alg:pseudo-krr} achieves this rate without any prior information on the target distribution $\cQ$ or the shift magnitude $B$, demonstrating its adaptivity.
\end{remark}

\section{Proof sketch and key intermediate results}\label{sec: proof sketch}

This section outlines the proof of \cref{main theorem}; the complete proofs are deferred to the appendix. Our main result hinges on two ingredients.
\Cref{generic oracle inequality} establishes an in-sample oracle inequality for model
selection via pseudo-labels under the GLM loss, characterizing how the
imperfection of the pseudo-labels propagates into the excess risk of the
selected model. This result is
generic and independent of the RKHS structure. \Cref{generic GLM estimator}
then shows how, in our kernel GLM setting, each term of this oracle inequality
--- the imputation error decomposition and the candidate excess risk bounds
--- can be controlled via a Taylor expansion of the regularized empirical
loss's gradient around $f^*$, yielding the pointwise convergence rate of
\cref{consistency rate} as a key intermediate result. \Cref{sec: putting things
together} explains how these two ingredients are combined to yield
\cref{main theorem}.

\subsection{Generic model selection analysis with GLM loss}\label{generic oracle inequality}

At the core of our analysis, we need to understand how the pseudo-labeler impacts the model selection. We consider the task of selecting an optimal candidate from a set of $m$ deterministic models $f_1, \dots, f_m$ to approximate an unknown regression function $f^*$. In our GLM framework, $f^*$ characterizes the conditional mean of the response through the canonical link $\logpar'(f^*(\bx))$. We work under a fixed design of unlabeled samples $\{\bx_i\}_{i=1}^n$, for which pseudo-labels are generated via an imperfect imputation model $\tilde{f}$. By treating the candidate models and design points as deterministic, and the imputer as the sole source of randomness (reflecting its dependence on external training data), we isolate the effect of pseudo-label errors from the finite-sample fluctuations of the target covariates, which are handled separately as described in \cref{sec: putting things together}.

We decompose the pseudo-labels as 
\begin{equation}
    \tilde{f}(\bx_i):= f^*(\bx_i) + \biasi + \varii,
\end{equation} 
where $\vari = (\varii)_i$ is a centered random vector and $\bias=(\biasi)_i$ is a bias term. They represent the random error and bias of the imputation model. 

Recall that the in-sample risk in a GLM framework is given by the negative log-likelihood:

\begin{equation}
    R_{in}(f) = \frac{1}{n} \sum_{i=1}^n (\logpar(f(\bx_i)) - \logpar'(f^*(\bx_i)) f(\bx_i)),
\end{equation}
and the in-sample excess risk is thus given by
\begin{align}\label{eq: in-sample risk}
\cR_{in}(f)=\frac{1}{n} \sum_{i=1}^n\ \Big[\logpar(f(\bx_i))-\logpar(f^*(\bx_i))-\logpar'(f^*(\bx_i))[f(\bx_i)-f^*(\bx_i)]
\Big].
\end{align}

Since the true function $f^*$ is unknown, we will use as in \cref{eq: risk proxy} a proxy for the in-sample excess risk, a pseudo excess risk for model selection purposes:
\begin{align}
\widehat{\cR}_{in}(f) = \frac{1}{n} \sum_{i=1}^n\ \Big[\logpar(f(\bx_i))- 
\logpar(\tilde f(\bx_i)) -\logpar'(\tilde f(\bx_i))[f(\bx_i)-\tilde f(\bx_i)]
\Big].
\end{align}
The pseudo-labeling strategy is thus to choose
\begin{equation}
    \hat{j}=\argmin_{j\in[m]}\hat{\cR}_{in}(f_j),
\end{equation}
and we denote the selected model by $\hat{f} = f_{\hat{j}}$.
\bigskip

We will only impose some regularity conditions, which hold in the kernel GLM setting of \cref{theory}, as verified in \cref{proof of in sample oracle}.

\begin{assumption}\label{assumptions oracle}
We make the following assumptions:
\begin{enumerate}
    \item The image of the covariates under the true model is bounded almost surely, i.e. there exists a constant $C>0$ such that $|f^*(\bx_i)|\le C$, $\forall i$.

    \item $\logpar$ is twice differentiable, locally strongly convex and $\logpar''$ is locally upper bounded and Lipschitz, i.e. there exist finite constants $k>0, K>0, L >0$ such that $\forall x, y \in [-2C, 2C]$, $K \ge \logpar''(x) \ge k$ and $|\logpar''(x)-\logpar''(y)|\le L|x-y|$.

    \item With probability at least $1-\delta/2$, both $\max_{i\in [n]}|\tilde{f}(\bx_i) - f^*(\bx_i)| \le \eta$ and $\norm{\bias}_2\le b$, for constants $b\ge 0$ and $\eta < \min\{k / L, C, \sqrt{kK}/L\}$.

    \item The imputation random error $\vari = (\tilde\varepsilon_1,\dots,\tilde\varepsilon_n)^\top$ is a zero-mean sub-exponential random vector: $\exists \variv>0$, $\norm{\vari}_{\psi_1} \le \variv$.

\end{enumerate}
\end{assumption}

\bigskip

Now the objective is to bound the excess risk of the selected model $\hat{f}$, relative to the excess risk of the model selected by an oracle that knows the true function $f^*$. We denote by $f_{j^*}$ this oracle model, with
\begin{equation*}
j^* = \argmin_{j\in[m]} \cR_{in}(f_j).
\end{equation*}

The following theorem is an oracle inequality for the in-sample excess risk of the selected model:
\begin{restatable}[Oracle inequality]{theorem}{oraclenonapprox}\label{oracle general non approx}
Let $\delta \in (0,1)$. Then there exists a universal constant $C_1>0$ such that under \cref{assumptions oracle}, the following holds \wpd:
\begin{equation}\label{eq: oracle inequality}
\begin{split}
  \cR_{in}(\hat f) \le \inf_{\epsilon>0} \left \{(1+\epsilon)\cR_{in}(f_{j^*}) + \frac{C_1K}{nk}\paren{1+\frac{1}{\epsilon}}\xi(\variv, b, \eta, \delta) \right \},  
\end{split}
\end{equation}
with 
\begin{equation*}
\xi(\variv, b, \eta, \delta) = K\variv^2\log^2(\frac{m}{\delta}) + Kb^2+ \frac{L^2n\eta^4}{k}.
\end{equation*}

\end{restatable}

The proof is deferred to \cref{poracle}. It proceeds by applying the law of cosines for Bregman divergences, which decomposes the true excess risk of any candidate into its pseudo excess risk plus a cross term involving the imputation error $\tilde{f} - f^*$, controlled by the local curvature of $a$ from \cref{assumptions oracle}. The minimality of the selected model's pseudo risk then transfers to the true risk up to this cross term --- a strategy that mirrors the square loss case, but where the Bregman
structure of the GLM loss requires additional care to obtain a sharp bound.

\medskip
The additive overhead in the oracle inequality \eqref{eq: oracle inequality} corresponds to the error of the pseudo-labels. It comprises a bias term proportional to $b^2/n$, where $b$ bounds $\|\bias\|_2$, a variance term proportional to $\variv^2\log^2(m)/n$, and a perturbation term proportional to $L^2\eta^4$, which arises from the nonlinearity of the GLM loss and is shown to be negligible compared to the other two terms in \cref{sec: putting things together}. Note that in the case of the square loss, $L=0$, so the perturbation term vanishes and the third item in \cref{assumptions oracle} is no longer necessary. The key asymmetry between the bias and variance terms is that $\variv$ is a scalar sub-exponential proxy for the noise of the imputer, whereas $\|\bias\|^2 = \sum_{i=1}^n B_i^2$ accumulates the squared pointwise biases over all $n$ design points. Consequently, as long as $\log^2(m)/n$ is small, even a moderate pointwise bias of the imputation model can dominate the variance contribution, making bias the critical quantity to control.

This settles the bias-variance trade-off for the imputation model: we should prioritize an \emph{undersmoothed} imputation model, i.e. with small bias, in order for the pseudo-labels to select a more faithful model.
In the case of kernel GLM, the imputation model should thus be trained with a small regularization parameter.

\begin{remark}[Random bias]\label{rem: random bias}
\Cref{assumptions oracle} only requires a high-probability bound on $\norm{\bias}_2$: the bias may be random and correlated with $\vari$. This is needed in our application. For KRR, the estimator is linear in the responses and the bias of the imputer is deterministic under fixed design. For kernel GLMs, the estimator has no closed form, so its error is centered around a deterministic surrogate rather than its expectation, and the resulting bias term depends on the noise (\cref{generic GLM estimator}). This is harmless, because $\bias$ only enters the proof of \cref{oracle general non approx} through inequalities that hold for every realization, while concentration is only used for $\vari$ (see \cref{rem: pathwise}).
\end{remark}

\subsection{Analysis of the KRGLM estimator}\label{generic GLM estimator}

We now show how the oracle inequality of \cref{oracle general non approx} can be applied in our kernel GLM setting by bounding each of its
terms --- the imputation error decomposition and the candidate excess risk
bounds --- via a Taylor expansion of the regularized
empirical loss's gradient. 

For any positive $\lambda$, the estimator $\hat{f}_\lambda$ of \cref{regularized glm bis} is only defined implicitly as the minimizer of a nonlinear objective, unlike the closed-form solution available in kernel ridge regression. For ease of presentation, we derive the following expansion in the finite-dimensional setting with $\cX \subset \RR^d$ and the linear kernel, writing $\bX = (\bx_1, \dots, \bx_n)^\top \in \RR^{n\times d}$ for the source design matrix, $\bvarepsilon = (\varepsilon_1, \dots, \varepsilon_n)^\top$ for the noise vector, and $f_{\btheta}(\bx) = \btheta^\top \bx$. The KRGLM minimization problem \cref{regularized glm bis} becomes
\begin{equation}\label{eq: loss theta}
    \hat{\btheta}_\lambda \in \arg\min_{\btheta \in \RR^d} \left\{ \frac{1}{n} \sum_{i=1}^n \left( \logpar(\bx_i^\top \btheta) - y_i \bx_i^\top \btheta \right) + \frac{\lambda}{2}\|\btheta\|^2 \right\}.
\end{equation}
Setting the gradient of the associated loss to zero at $\hat{\btheta}_\lambda$ and Taylor-expanding around $\btheta^*$ yields:
\begin{equation}\label{eq: taylor expansion}
    \hat{\btheta}_\lambda - \btheta^* = -\bH^{-1}\!\left(\frac{1}{n}\bX^\top \bvarepsilon + \lambda \btheta^*\right),
\end{equation}
where the data-dependent Hessian is
\begin{equation}\label{eq: hessian}
     \bH = \frac{1}{n}\sum_{i=1}^n D_{ii}\, \bx_i\bx_i^\top + \lambda \bI, \qquad D_{ii} = \int_0^1 \logpar''(\bx_i^\top(\btheta^* + t(\hat{\btheta}_\lambda - \btheta^*)))\,dt.
\end{equation}
To isolate the dependence on $\hat{\btheta}_\lambda$, we introduce the deterministic Hessian $\bar{\bH} = \frac{1}{n}\sum_{i=1}^n \logpar''(\bx_i^\top\btheta^*)\, \bx_i\bx_i^\top + \lambda \bI$, obtained by evaluating $\logpar''$ at $\btheta^*$ rather than along the path. Splitting \eqref{eq: taylor expansion} accordingly gives the decomposition:
\begin{equation}\label{eq: bvp decomposition}
    \hat{\btheta}_\lambda - \btheta^* = - \bigg[\underbrace{\lambda\bH^{-1}\btheta^*}_{\text{bias}} + \underbrace{\frac{1}{n}\bar{\bH}^{-1}\bX^\top \bvarepsilon}_{\text{variance}} + \underbrace{\frac{1}{n}(\bH^{-1}-\bar{\bH}^{-1})\bX^\top \bvarepsilon}_{\text{perturbation}} \bigg].
\end{equation}
An equivalent decomposition holds in the general setting via the isometry between the RKHS $\mathcal{F}$ and the feature space $\mathbb{H}$ \cref{eq: rkhs isometry}, with $\bX$ interpreted as a bounded linear operator; see \cref{subsection notations}.

\medskip
Recall that our \cref{alg:pseudo-krr} trains a grid
$\{\hat{f}_\lambda\}_{\lambda \in \Lambda}$ of candidate models on the first half $\cD_1$ of the source data and an imputation model $\tilde{f} = \hat{f}_{\tilde\lambda}$ on the second half $\cD_2$, which generates pseudo-labels $\tilde{f}(\bx_{0i})$ for each target covariate in $\cD_0$. To bound the terms entering the oracle inequality of \cref{oracle general non approx} applied to $\cD_0$, we project decomposition \cref{eq: bvp decomposition} onto the target second-moment operator $\bSigma_0$ for the candidate excess risk $R(\hat{f}_\lambda) - R(f^*)$, and onto the target design points $\{\bx_{0i}\}$ for the imputation error $\{\tilde{f}(\bx_{0i}) - f^*(\bx_{0i})\}_{i=1}^{n_0}$. The main difficulty is analyzing the perturbation term, since $\bH$ depends on $\hat{\btheta}_\lambda$ itself and is correlated with the noise. We resolve this in two steps. First, a localization argument (\cref{crude bound}) establishes, using only the global strict convexity of $\logpar$, that $\hat{\btheta}_\lambda$ lies in a neighborhood of $\btheta^*$ with high probability. Once in this neighborhood, the local properties of $\logpar''$ from \cref{convexityassumption} --- strong convexity, boundedness, and Lipschitz continuity --- all become available. In particular, the Lipschitz continuity ensures that the relative operator error between $\bH$ and $\bar{\bH}$ is controlled by the pointwise deviations $|\hat{f}_\lambda(\bx_i) - f^*(\bx_i)|$, reducing the perturbation analysis to a pointwise consistency problem. The following theorem provides the required bound.

\begin{restatable}[Pointwise convergence rate]{theorem}{consistency}\label{consistency rate}
    Consider a ridge-regularized kernel GLM estimator $\hat{f}_\lambda$ trained on $n$ i.i.d.\ samples from the source distribution $\cP$ via \cref{regularized glm bis}. Let $\beta \in (0,8)$, and choose $\delta \in (0, 1/e]$. Then, under \Cref{setup,kernel assumptions,subexponentialassumption,convexityassumption,trueparam}, there exist constants $c, c', c_1, c_2 >0$, such that for any given $\bx_0 \in \supp(\cP)$, when:
    \[
  \begin{cases}
    \lambda \le \frac{c_1}{R^4\log^{\beta}(n)} \\
    \lambda  \ge \frac{c_2M_1^2\resv^{2}\log(\frac{8}{\delta})(\log n)^{2+\beta/2}}{n}
  \end{cases},
\]
    the following inequality holds \wpd:
    \begin{equation}\label{eq: glm rate bias variance}
        \abs{\hat f_\lambda(\bx_0)-f^*(\bx_0)} \le cM_1 \left(\resv\log(\frac{n}{\delta})\frac{\lambda^{-1/4}}{\sqrt{n}} + \lambda^{1/4}\norm{f^*}_{\cF}\right).
    \end{equation}

When setting $\lambda$ to its lower bound, the result becomes
\begin{equation}\label{eq: glm rate}
        \abs{\hat f_\lambda(\bx_0)-f^*(\bx_0)} \le c'\sqrt{M_1^3\sigma} n^{-1/4}\paren{\log (\frac{n}{\delta})}^3.
    \end{equation}
\end{restatable}

The proof is deferred to \cref{proofres}.

\paragraph{Bounding the oracle inequality terms via the pointwise rate.} 
We can now carry out the program described above. Projecting \eqref{eq: bvp decomposition} and applying \cref{consistency rate} to control the perturbation yields:

\begin{enumerate}
    \item \textbf{Candidate excess risk bounds.} Contracting
    \eqref{eq: bvp decomposition} against the target second-moment operator
    $\bSigma_0$ yields a bound on the excess risk
    $R(\hat{f}_\lambda) - R(f^*)$ for each candidate. Informally, after applying union bounds, we obtain up to constant and $\log$ factors:
    \begin{equation}
        \min_{\lambda \in \Lambda}\{R(\hat{f}_\lambda) - R(f^*)\} \lesssim \lambda R^2 \norm{\bS_\lambda} + \frac{\sigma^2}{n_1}\Tr(\bS_\lambda),
    \end{equation}
    where $\bS_{\lambda} = (k\bSigma + \lambda \bI)^{-1/2}\bSigma_0(k\bSigma + \lambda \bI)^{-1/2}$. A formal statement can be found in \cref{epsilon risk lemma}. Using the effective sample size to relate $\bSigma$ to $\bSigma_0$, one can easily verify that $\Tr(\bS_\lambda) \le \frac{2n}{n_\text{eff}}\Tr((k\bSigma_0 + \lambda \frac{n}{n_\text{eff}}\bI)^{-1}\bSigma_0)$ and $\norm{\bS_\lambda} \le \frac{n}{kn_\text{eff}}$, as shown in \cref{sec: use of neff}. Thus taking $n_1=n/2$, it reduces to
    \begin{equation}\label{eq: out of sample excess risk}
        \min_{\lambda \in \Lambda}\{R(\hat{f}_\lambda) - R(f^*)\} \lesssim \frac{n}{n_{\text{eff}}}\lambda R^2 + \frac{\sigma^2}{n_{\text{eff}}}\Tr((k\bSigma_0+\frac{n}{n_{\text{eff}}}\lambda \bI)^{-1}\bSigma_0),
    \end{equation}
    as proved in \cref{proof of main theorem}.
    \item \textbf{Imputation error decomposition.} 
    Following the insight of \cref{generic oracle inequality}, our \cref{alg:pseudo-krr} trains the imputation model with a small bias with respect to our $\Lambda$ grid, by choosing $\tilde \lambda \asymp n^{-1}$, which also corresponds to the optimal pointwise convergence rate as per \cref{consistency rate}. Applying \eqref{eq: bvp decomposition} to $\tilde{\btheta}=\hat{\btheta}_{\tilde \lambda}$ and projecting onto the target covariates
    $\{\bx_{0i}\}_{i=1}^{n_0}$ yields a decomposition of $\{\tilde{f}(\bx_{0i}) - f^*(\bx_{0i})\}_{i=1}^{n_0}$ into
    the bias, variance, and perturbation terms entering $\xi(\variv, b, \eta, \delta)$ in
    \cref{oracle general non approx}, and allows us to bound them in \cref{fixed design lemma}. Unlike in KRR, the resulting bias term is random, which \cref{assumptions oracle} allows (\cref{rem: random bias}). Informally, for this choice of $\tilde \lambda$, we find that the squared variance proxy and the square norm of the bias are of the same order:
    \begin{equation}\label{eq: bias variance bounds}
        \tilde \sigma^2 \lesssim \frac{n_0}{n_2}\norm{\hat \bS_{2,\tilde \lambda}},\quad \norm{\bB}^2\lesssim \frac{n_0}{n_2}\norm{\hat \bS_{2,\tilde \lambda}},
    \end{equation}
    where $\hat \bS_{2, \tilde\lambda} = (k\hat \bSigma_2 + \tilde\lambda \bI)^{-1/2}\hat \bSigma_0(k\hat \bSigma_2 + \tilde\lambda \bI)^{-1/2}$. To bound the perturbation term $n_0 \eta^4$, we recall that $\eta$ bounds $\max_{i\in[n_0]}\abs{\tilde{f}(\bx_{0i}) - f^*(\bx_{0i})}$ with high probability. Substituting the rate $\eta \lesssim n_2^{-1/4}$ from \cref{consistency rate} yields
    \begin{equation}\label{eq: perturbation term bound}
        n_0 \eta^4 \lesssim \frac{n_0}{n_2}.
    \end{equation}
\end{enumerate}

\paragraph{Pointwise rate.}
We conclude this section by commenting on the rate obtained in \cref{consistency rate}. We obtain a $O(n^{-1/4})$ pointwise convergence rate of the ridge-regularized kernel GLM estimator when $\lambda \asymp n^{-1}$. This behavior is consistent with the sharp asymptotic theory for linear functionals of kernel ridge regression developed by \citet{TuoZouSS2024}, who analyze point evaluation in Sobolev-equivalent RKHSs: when the input space $\cX$ is a convex compact subset of $\RR^d$ and $\cF$ is equivalent to a Sobolev space of order $m>d/2$, they show in Theorem 7 that $\lambda \asymp n^{-1}$ yields a rate of $n^{-1/2 + d/4m}$ for linear functionals, with matching upper and lower bounds (hence optimality in their framework). In particular, the first-order Sobolev (Brownian) kernel $K(u,v) = \min\{u,v\}$ on $[0,1]$ ($d=m=1$) satisfies the quadratic eigendecay $\mu_j \lesssim j^{-2}$ and has uniformly bounded eigenfunctions, so both their results and ours predict the same $O(n^{-1/4})$ pointwise rate.

We also emphasize that this pointwise rate is slower than the $L^2(\mathcal{X} ; \cP)$-prediction rate, reflecting the additional difficulty of controlling a point evaluation functional compared to an integrated $L^2$ loss. We show in \cref{proof of L2 GLM rate} how to recover the optimal prediction rate $\norm{\hat f_\lambda - f^*}_{L^2} = O(n^{-1/3})$ in our setting, obtained for $\lambda \asymp n^{-2/3}$. This in turn matches the fast rate of $O(n^{-\frac{2r\alpha+\alpha}{2r\alpha+\alpha+1}})$ obtained by \citet{bachfastrates} with capacity condition $\alpha=2$ (corresponding to our eigendecay $\alpha=1$) and source condition $r=0$ (i.e. without any additional smoothness assumption on $f^*$ beyond membership in the RKHS) in their Corollary 9. The penalty $\lambda\asymp n^{-2/3}$ lies inside the admissible window of \cref{consistency rate}, where it yields the weaker pointwise rate $O(n^{-1/6})$: the two results optimize different functionals, hence different optimal penalties.

\subsection{Putting things together}\label{sec: putting things together}

We now combine the two previous ingredients. First, we can substitute the bounds on the imputation error
decomposition from \cref{eq: bias variance bounds,eq: perturbation term bound} with $n_2=n/2$ in the oracle inequality of \cref{oracle general non approx}. We obtain informally and up to $\log$ factors:
\begin{align*}
    R_{in}(\hat{f})-R_{in}(f^*) \lesssim \min_{\lambda \in \Lambda}
    \{R_{in}(\hat{f}_\lambda)-R_{in}(f^*)\} + \frac{1}{n}\norm{\hat \bS_{2,\tilde \lambda}}.
\end{align*}

A formal statement can be found in \cref{lemma in sample oracle}. By further bridging the empirical second-moment operator $\hat \bS_{2,\tilde \lambda}$ to its population counterparts $\bSigma$ and $\bSigma_0$ through $n_{\mathrm{eff}}$, and relating the in-sample and out-of-sample excess risks via concentration inequalities, it implies:
\begin{align*}
    R(\hat{f})-R(f^*) \lesssim \min_{\lambda \in \Lambda}
    \{R(\hat{f}_\lambda)-R(f^*)\} + O\!\left(\frac{1}{n_{\mathrm{eff}}}
    +\frac{1}{n_0}\right).
\end{align*}
A formal statement is presented in \cref{oracle inequality for pseudo-labeling}. To obtain \cref{main theorem}, it then suffices to
substitute the candidate excess risk bounds from \cref{eq: out of sample excess risk} into the first term. As seen in \cref{cor:poly_decay_neff},
the overhead due to pseudo-label errors, of order
$O(1/n_{\mathrm{eff}}+1/n_0)$, is negligible compared to the resulting
candidate excess risks of order
$O(n_{\mathrm{eff}}^{-\frac{2\alpha}{2\alpha+1}})$, determined by the complexity of the target distribution.

\subsection{Comparison with \texorpdfstring{\citet{wang2023pseudo}}{Wang (2026)}}\label{sec: wang comparison}
\Cref{tab:wang comparison} compares our analysis with the KRR analysis of \citet{wang2023pseudo}. The split--impute--select pipeline is intentionally inherited, and our $n_{\mathrm{eff}}$ reduces to theirs at $k=1$; our contribution is to prove that this procedure remains valid beyond least squares, a case that \citet{wang2023pseudo} leaves open, as the analytical expressions of the KRR estimator are central to their proof. Without them, the perturbation term of \cref{generic GLM estimator} can only be controlled once the estimator is localized where $\logpar''\ge k$, which is itself what must be established; our localization argument and pointwise rate break this circularity, and the Bregman structure of the GLM loss extends the pseudo-label oracle inequality. While reusing the loss-agnostic parts of their analysis (e.g., \cref{bridge risks}), we obtain the same rate for logistic and Poisson regression, substantially enlarging the scope of the method: many UDA problems, including both of our real-data tasks, are classification problems.

\begin{table}[t]
  \caption{Comparison with \citet{wang2023pseudo}. Both use the same split--impute--select pipeline.}
  \label{tab:wang comparison}
  \centering
  \footnotesize
  \setlength{\tabcolsep}{4pt}
  \begin{tabular}{@{} p{0.15\textwidth} p{0.36\textwidth} p{0.45\textwidth} @{}}
    \toprule
    & \citet{wang2023pseudo} -- KRR & This paper -- KRGLM \\
    \midrule
    Effective size & $\sup\{t\le n: t\bSigma_0\preceq n\bSigma+\mu^2\bI\}$ & Same, with curvature: $t\bSigma_0\preceq n\bSigma+\mu^2\bI/k$ \\
    Loss / task & Square loss; regression & GLM loss: linear, logistic, Poisson (regression, classification, counts) \\
    Curvature & Quadratic everywhere ($\logpar''\equiv 1$) & Locally: $k\le \logpar''\le K$ on $[-2C,2C]$, $\logpar''$ $L$-Lipschitz \\
    Noise & Sub-Gaussian & Sub-exponential (needed for Poisson) \\
    Kernel & Trace-class $\bSigma,\bSigma_0$, bounded feature map; decay only in the corollary & Mercer expansion with bounded eigenfunctions in the main theorem, for pointwise control (strictly stronger) \\
    Estimator & Closed form; exact bias + linear noise split & Implicit; extra perturbation term, as the Hessian depends on the estimator \\
    Imputer error & Immediate from the closed form & Localization + pointwise rate (\cref{consistency rate}) \\
    Oracle ineq. & General classes, squared loss & Same generality, GLM excess risk via the Bregman structure (\cref{oracle general non approx}) \\
    \bottomrule
  \end{tabular}
\end{table}

\section{Numerical experiments}\label{sec: experiments}

We conduct experiments on synthetic and real data to show the benefit of our approach over source-only and reweighting baselines. The code for implementing our experiments is available at \repourl{}. In particular, we implemented an ad hoc KRGLM solver that can be of independent interest; it is explained in detail in \cref{app: algo details}.

\paragraph{Kernel logistic regression specifics.} In the synthetic and Raisin experiments, we focus on kernel logistic regression, for which the log-partition function is $\logpar(u)=\log(1+e^u)$, and its first derivative is the sigmoid function $\logpar'(u)=\frac{e^u}{1+e^u}=\sigma(u)$ that maps real values to $[0, 1]$. If the true underlying function is $f^*$, then $y_i \mid \bx_i \sim \text{Bernoulli}(\sigma(f^*(\bx_i)))$. We can thus denote the conditional mean as a probability $p^*_i := \sigma(f^*(\bx_i))$. Equivalently, $f^*(\bx_i) = \log(\frac{p^*_i}{1-p^*_i})=\text{logit}(p^*_i)$. As such, for any estimator $f$ of $f^*$, if we denote for each sample $p_i = \sigma(f(\bx_i))$, the in-sample version of the GLM risk \cref{eq: glm risk} recovers the usual log-loss with ``soft labels'' or cross-entropy between the true distribution $p^*$ and the estimate $p$:
\begin{equation*}
    R_{in}(f)=\frac{1}{n}\sum_{i=1}^n \left[\logpar(f(\bx_i))-\logpar'(f^*(\bx_i))f(\bx_i)\right]=\frac{1}{n}\sum_{i=1}^n \left[-(1-p^*_i)\log(1-p_i)-p^*_i\log(p_i)\right].
\end{equation*}

\subsection{Synthetic data}\label{sec: experiments synth}

For simulation experiments, our feature space $\cX$ is the unit interval $[0,1]$, and the kernel $K$ is the first-order Sobolev kernel augmented with a constant, $K(u,v) = 1+\min \{ u, v \}$ (\cref{ex: sobolev kernel}), whose RKHS contains the function $f^*$ defined below. Let $f^{*}(x)=1.5\cos (2 \pi x), \nu_0=\mathcal{U}[0,1 / 2]$ and $\nu_1=\mathcal{U}[1 / 2,1]$. For any even integer $n \geq 2$, we follow the procedure below:
\begin{itemize}
    \item Let $B=n^{0.4}, \mathcal{P}=\frac{B}{B+1} \nu_0+\frac{1}{B+1} \nu_1$ and $\mathcal{Q}=\frac{1}{B+1} \nu_0+\frac{B}{B+1} \nu_1$. Generate $n$ i.i.d. source samples $\left\{\left(x_i, y_i\right)\right\}_{i=1}^n$ with $x_i \sim \mathcal{P}$ and $y_i \mid x_i \sim \text{Bernoulli}\left(\sigma(f^{*}\left(x_i\right))\right)$, where $\sigma$ is the sigmoid function. Generate $n_0=n$ i.i.d. unlabeled target samples $\mathcal{D}_0=\left\{x_{0 i}\right\}_{i=1}^{n_0}$ from $\mathcal{Q}$.
    \item Run \cref{alg:pseudo-krr} with $n_1=n / 2, \Lambda=\left\{2^k /(10 n): k=0,1, \cdots,\left\lceil\log _2(10 n)\right\rceil\right\}$, and $\widetilde{\lambda}=1 /(10 n)$ to get an estimate $\widehat{f}$. Use the log-partition function associated to logistic regression $\logpar(u)=\log(1+e^u)$, which verifies $\logpar'(u)=\sigma(u)$.
\end{itemize}

The $1.5$ amplitude of $f^*$ ensures that the source contains enough points near the decision boundary to learn decent pseudo-labels. The quantity $B$ controls the strength of the covariate shift; letting it grow with $n$ allows us to test larger shifts at larger sample sizes.

Our method selects a candidate model that minimizes the (target) risk estimate
\begin{equation*}
f \mapsto \frac{1}{\left|\mathcal{D}_0\right|} \sum_{x \in \mathcal{D}_0}\paren{\logpar\paren{f(x)}- f(x)\logpar'(\tilde f(x))}
\end{equation*}
based on pseudo-labels. We compare it with an oracle selection approach whose (target) risk estimate 
\begin{equation*}
f \mapsto\frac{1}{\left|\mathcal{D}_0\right|} \sum_{x \in \mathcal{D}_0}\paren{\logpar\paren{f(x)}- f(x)\logpar'(f^{*}(x))}
\end{equation*}
is constructed from the true conditional probabilities, and a naive approach that uses the empirical (source) risk 
\begin{equation*}
f \mapsto\frac{1}{\left|\mathcal{D}_2\right|} \sum_{(x, y) \in \mathcal{D}_2}\paren{\logpar\paren{f(x)}- yf(x)}
\end{equation*}
on $\mathcal{D}_2$. The latter is the standard hold-out validation method that ignores the covariate shift.

These three criteria are used only for selecting $\lambda$ from the candidate grid. Final performance is evaluated separately on an independent target test set, consisting of 10,000 new samples from $\mathcal Q$. We then report the corresponding Monte Carlo estimate of the target excess risk, computed via \cref{eq: glm excess risk}.


For every $n \in \{4000,8000,16000,32000\}$, we conduct 100 independent runs of the experiment. The results are shown in \cref{fig:reslog}. The horizontal axis represents the source sample size $n$, and the vertical axis represents the average estimated target excess risk across the 100 trials.

\begin{figure}[ht]
  \centering
  \subfigure[$B=n^{0.4}$]{\includegraphics[width=0.48\textwidth]{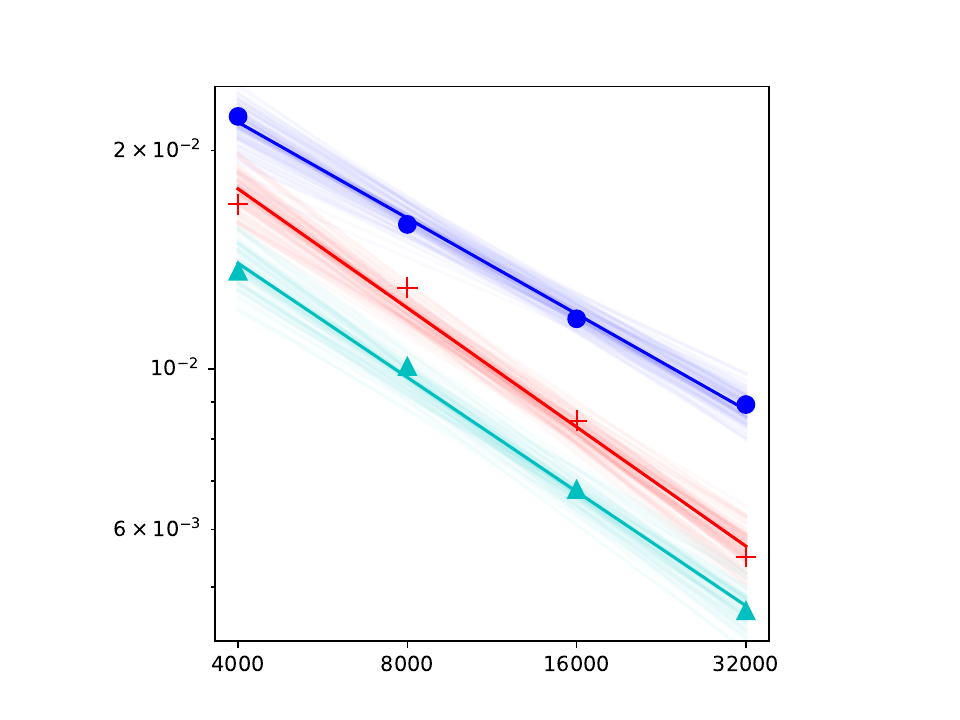}\label{fig:reslog}}
  \hfill
  \subfigure[$B=n^{0.45}$]{\includegraphics[width=0.48\textwidth]{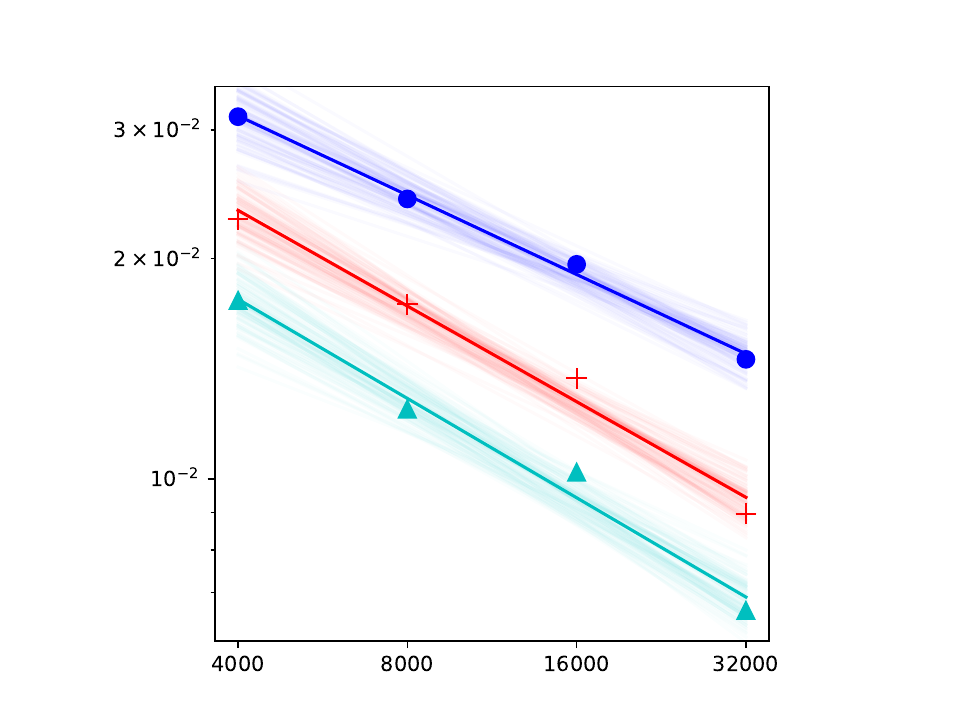}\label{fig:reslog2}}

  \caption{Comparisons of the three approaches on a log-log plot. $x$-axis: $n$. $y$-axis: excess risk. Red crosses: pseudo-labeling. Cyan triangles: the oracle method. Blue circles: the naive method. Solid lines: fitted power laws; thin shaded regions: linear fits of 100 bootstrap replicates.} 
  \label{fig:reslog both}
\end{figure}

First, the oracle is lowest, as expected; pseudo-labeling closely tracks it; and the naive method is uniformly higher, reflecting the cost of selecting $\lambda$ under $\mathcal P$ while evaluating under $\mathcal Q$. Second, the approximately linear alignment of the data points on the log-log scale suggests a power-law decay of the target excess risk, of the form $O(n^{-\kappa})$ for some exponent $\kappa>0$. Linear fits to the log-transformed data give $\kappa=0.546(0.047)$, $0.523(0.051)$, and $0.439(0.043)$ for the pseudo-labeling, oracle, and naive methods, respectively, represented by the solid lines in the figure. The standard errors reported in parentheses were computed via a cluster bootstrap procedure \citep{bootstrap}, resampling the 100 trials associated with each $n$ over 10,000 iterations. Our proposed method achieves an error exponent that does not differ significantly from the oracle benchmark. Furthermore, our approach significantly outperforms the naive method, which suffers from poor source coverage of the target domain. \Cref{fig:reslog2} shows the same experiment with the stronger covariate shift $B=n^{0.45}$.

\paragraph{Varying the hardness of the shift.} To test the adaptivity of \cref{alg:pseudo-krr} directly in the regime $n_{\mathrm{eff}}\ll n$, we repeat the experiment with $B=n^{\gamma}$ for $\gamma\in\{0.40,0.45,\dots,0.65\}$, which corresponds to $n_{\mathrm{eff}}\asymp n^{1-\gamma}$, with 100 trials for each $n$. \Cref{tab:gamma sweep} reports the fitted exponents; the rows $\gamma=0.40$ and $\gamma=0.45$ correspond to \cref{fig:reslog,fig:reslog2}. As the shift becomes harder, naive source validation deteriorates until its excess risk is essentially flat at $\gamma=0.65$, whereas pseudo-labeling keeps improving, without knowledge of $B$ or $n_{\mathrm{eff}}$: at $n=32{,}000$ and $\gamma=0.65$, the target excess risks are $0.0429$ (pseudo-labeling), $0.1361$ (naive) and $0.0265$ (oracle). The fitted exponents of pseudo-labeling and oracle exceed the worst-case minimax exponent $2(1-\gamma)/3$ of \cref{cor:poly_decay_neff} (with $\alpha=1$); this should not be over-interpreted, as the slopes are fitted from four sample sizes only, and a fixed $f^*$ can converge faster than the worst case without contradicting the lower bound. Mean risks, paired slope comparisons and plots for every $\gamma$ are available in the code repository.

\begin{table}[ht]
  \caption{Fitted decay exponents $\kappa$ of the target excess risk (bootstrap SE) for $B=n^\gamma$, i.e., $n_{\mathrm{eff}}\asymp n^{1-\gamma}$.}
  \label{tab:gamma sweep}
  \centering
  \small
  \begin{tabular*}{\columnwidth}{@{\extracolsep{\fill}} ccccc @{}}
    \toprule
    $\gamma$ & Minimax $2(1-\gamma)/3$ & Pseudo-labeling & Oracle & Naive \\
    \midrule
    0.40 & 0.400 & 0.546 (0.047) & 0.523 (0.051) & 0.439 (0.043) \\
    0.45 & 0.367 & 0.434 (0.049) & 0.451 (0.053) & 0.360 (0.047) \\
    0.50 & 0.333 & 0.386 (0.059) & 0.369 (0.058) & 0.285 (0.048) \\
    0.55 & 0.300 & 0.402 (0.059) & 0.342 (0.057) & 0.231 (0.048) \\
    0.60 & 0.267 & 0.362 (0.060) & 0.273 (0.059) & 0.111 (0.046) \\
    0.65 & 0.233 & 0.301 (0.060) & 0.176 (0.064) & 0.008 (0.042) \\
    \bottomrule
  \end{tabular*}
\end{table}

\subsection{Sensitivity to the imputation penalty}\label{app: ablation}

Our theory requires an undersmoothed imputation model (\cref{generic oracle inequality}). To validate this prescription and the practical rule $\tilde\lambda=\min\Lambda$ of \cref{section: methodology}, we run an ablation on the imputation penalty $\tilde\lambda$ in the setting of \cref{sec: experiments synth} with covariate shift strength $B=n^{0.45}$. We use $n=16{,}000$, large enough for the leading-order effects to dominate finite-sample noise, and vary only $\tilde\lambda$: for each value, the pseudo-labels select a candidate among the same source-trained models, and we report the target excess risk of the selected estimator over 50 independent trials.

\begin{figure}[ht]
    \centering    \includegraphics[width=0.65\linewidth]{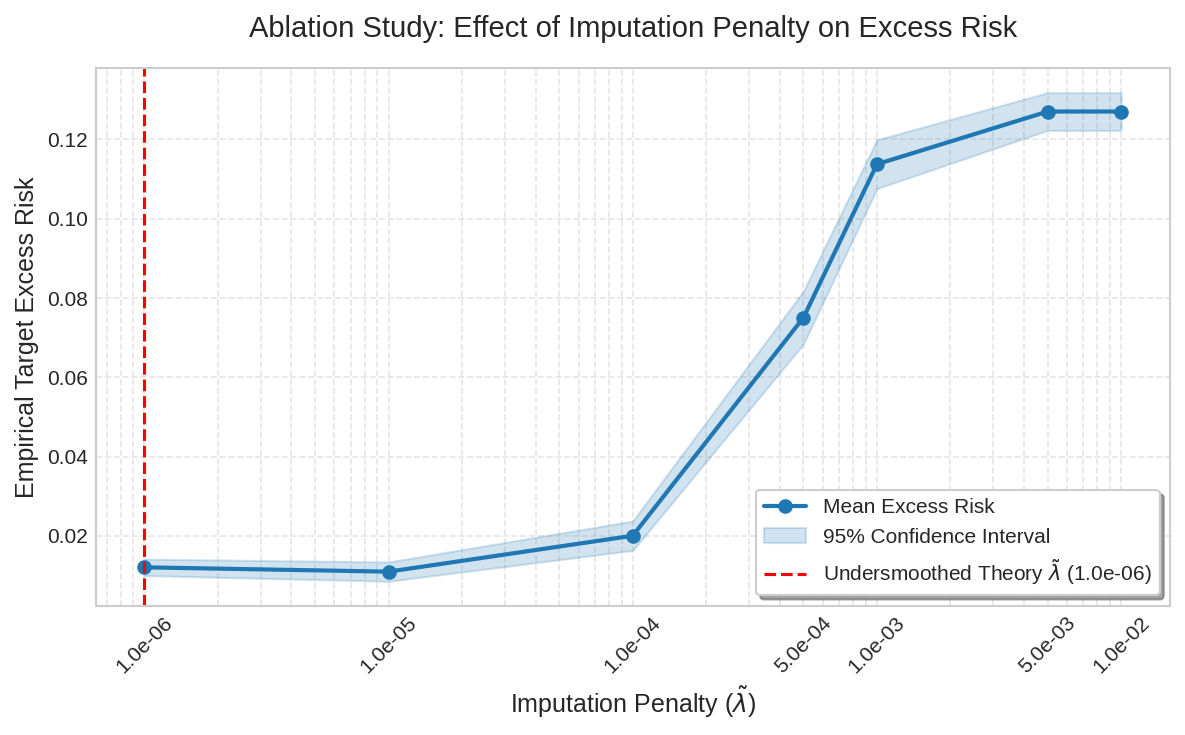}
    \caption{Target excess risk of the pseudo-labeling estimator as a function of the imputation penalty $\tilde\lambda$ ($n=16{,}000$, 50 trials, 95\% confidence band). The rule $\tilde\lambda=\min\Lambda=1/(10n)\approx 6.3\times10^{-6}$ used in \cref{sec: experiments synth} lies between the two leftmost points, in the flat region.}
    \label{fig:ablation}
\end{figure}

\Cref{fig:ablation} confirms the need for undersmoothing: the empirically optimal penalty is small, and it yields a tenfold reduction in excess risk compared with a standard, more strongly regularized imputer. Within the undersmoothed regime, the final estimator is insensitive to $\tilde\lambda$: its target excess risk is $0.012\pm0.001$ at $\tilde\lambda=10^{-6}$, $0.011\pm0.001$ at the empirical optimum $10^{-5}$, and still $0.020\pm0.002$ at $10^{-4}$, whereas $\tilde\lambda=10^{-3}$ gives $0.114\pm0.003$. What matters is thus to avoid oversmoothing, which the rule $\tilde\lambda=\min\Lambda$ achieves without target labels. This rule is motivated by \cref{main theorem}, where $\min\Lambda=\mu^2/n$ and $\tilde\lambda=\mu^2\log^7(n)\log(n_0/\delta)/n$ share the unknown scale $\mu^2=M_1^6\sigma^2$, so that their ratio is a known logarithmic factor, which we drop. It remains a heuristic, since in practice $\Lambda$ is built from source data rather than from the theoretical grid, which depends on unknown constants. At small sample sizes, the variance of the imputer also matters, which motivates the repeated folds used for Raisin in \cref{sec: experiments real}. Finally, undersmoothing cannot correct the part of the imputer's error due to a change in the conditional law, which lies outside our theory.

\subsection{The Raisin dataset}\label{sec: experiments real}
The Raisin dataset\footnote{available at \url{https://archive.ics.uci.edu/dataset/850/raisin}} \citep{raisin_850} contains 900 instances of two raisin varieties with 7 numerical features. The task is binary classification. To introduce covariate shift, we follow a usual sub-sampling technique on the first covariate \citep{book_datasetshift_ml}: after scaling the data, for each sample $\bx_i$, look at its first component $\bx_i[0]$ and compute the rejection probability $p_i = \min(1, (\bx_i[0] - c)^2/l)$, where $c = \min_i (\bx_i[0])$ and $l>0$ is a parameter to adjust the shift strength (the smaller it is, the larger the shift). Then sample $i$ is assigned to the target out-of-distribution (OOD) set if a draw of the random variable $\text{Bernoulli}(p_i)$ returns $1$, otherwise, it is assigned to the in-distribution (ID) source set. In the rest of this experiment, we selected a shift level $l=3$. Since the shift is induced by covariate-dependent subsampling, $p(y\mid\bx)$ is unchanged and this benchmark satisfies the pure covariate shift assumption of our theory.

To first illustrate the effect of this shift, we train a ridge-regularized logistic regression baseline only on source data, with a grid of ridge regularization parameters $\lambda$. We keep 20\% of the ID data and all of the OOD data for evaluation. 

\begin{figure}[ht]
    \begin{center}
        \centerline{\includegraphics[width=0.85\columnwidth]{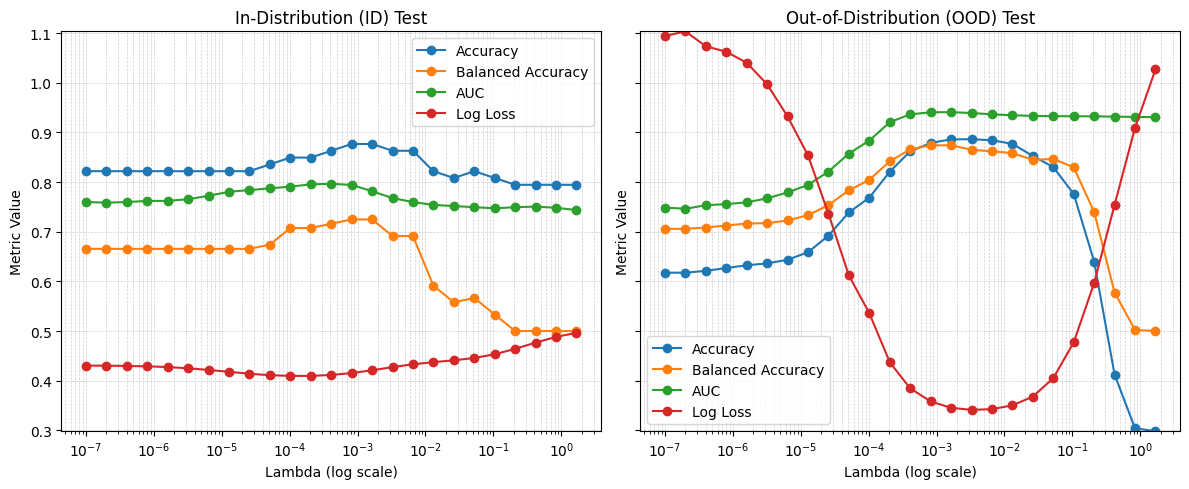}}
        \caption{Baseline logistic regression performance on ID and OOD test data, when trained only on ID data, as a function of the ridge regularization parameter $\lambda$.}
        \label{fig:logistic baseline}
    \end{center}
\end{figure}

\cref{fig:logistic baseline} plots the results for four different metrics. We can observe that only a specific range of regularization (between $10^{-3}$ and $10^{-2}$) can achieve good performance on the target data, different from the optimal one for the source. The goal of our method is to find a near-optimal value without seeing the target labels used to plot these curves. Among ridge logistic regression models trained on source only, the target-optimal regularization parameter is around $100$ times larger than the source-optimal one.

We run our pseudo-labeling experiment in the same fashion as for the synthetic data (\cref{sec: experiments synth}), by comparing it to a naive method that selects the best model on source validation data, and to an oracle method that uses the target labels for selection. We added comparisons against importance-weighting (IW) methods using two density-ratio estimation procedures: Kernel Mean Matching \citep{NIPS2006_a2186aa7} and Kullback-Leibler Importance Estimation Procedure \citep{sugiyama_direct_2008}. Since \citet{feng_towards_2023}, our closest competitor in the kernel covariate-shift setting, do not provide a target-label-free selection rule, we extended their KLIEP implementation with importance-weighted cross-validation \citep{sugiyama_covariate_nodate} for a fair UDA comparison on the Raisin benchmark. 

We run it for 100 different random seeds for significance and report the results in \cref{tab:risk_results}. As our objective is to estimate the conditional probability function, we evaluate performance with the empirical risk (i.e., the log-loss), a strictly proper scoring rule. For each seed, we first split the OOD data into training (for pseudo, density-ratio and oracle methods) and test set for evaluation purposes.

Given the limited sample size, we adopted a more robust model selection strategy: instead of a simple hold-out split, we employed repeated 2-fold cross-validation (see for instance \cite{crossvalbible, crossvalrepeat}) (6 repetitions) on the In-Distribution (ID) data. Candidate models were selected based on risk estimates averaged across all folds, and subsequently refitted on the full ID dataset using the chosen hyperparameters. This refit is a standard empirical variance-reduction step, used only here: our guarantees cover the output of \cref{alg:pseudo-krr}, trained on half of the source data, not the refitted model; it is applied identically to all methods, so it cannot create the reported gaps (the synthetic and Camelyon17 experiments run \cref{alg:pseudo-krr} as stated, without refitting). Our pseudo-labeling approach consistently outperforms the other baselines; the higher variance of the importance-weighted methods reflects the practical instability of density-ratio estimation in finite samples. See \cref{app: real data exp} for method details, hyperparameters, and further results.

\begin{table}[t]
  \caption{Risk mean estimate across seeds with confidence interval (CI).}
  \label{tab:risk_results}
  \centering
  \small
  \begin{tabular*}{\columnwidth}{@{\extracolsep{\fill}} lcc @{}}
    \toprule
    Method & Mean & 95\% CI \\
    \midrule
    Naive           & 0.502 & [0.463, 0.541] \\
    KLIEP + IW-CV          & 0.478 & [0.438, 0.518] \\
    KMM + IW-CV          & 0.469 & [0.442, 0.496] \\
    Pseudo-labeling & 0.428 & [0.412, 0.444] \\
    Oracle          & 0.373 & [0.361, 0.385] \\
    \bottomrule
  \end{tabular*}
\end{table}

\subsection{The Camelyon17 medical images
}\label{sec: wilds}

We finally test the selection principle well outside the scope of our theory, on early stopping when fine-tuning a deep network on the Camelyon17 \citep{bandi2018detection} benchmark from WILDS \citep{koh2021wildsbenchmarkinthewilddistribution}. This experiment is \emph{not} evidence for \cref{main theorem}: the candidates are neither kernel GLMs nor ridge-penalized, and scanner and staining differences across hospitals plausibly induce conditional, not merely covariate, shift. It tests whether the selection principle---impute the conditional mean on target covariates, then rank candidates by pseudo-target risk---transfers to a modern workflow, here tumor detection in histopathological patches (\cref{fig:camelyon}). A pre-trained DINOv2 ViT-B/14 \citep{oquab2023dinov2} backbone is fine-tuned end-to-end on labeled source data ($\approx 300$k samples from hospitals 1--3), with the epoch number serving as the regularization parameter. Following \cref{alg:pseudo-krr}, an imputer is trained on a held-out source validation dataset ($\approx 34$k samples) and used to generate soft pseudo-labels on an unlabeled target selection set ($\approx 42$k samples from hospital 5); the selected checkpoints are evaluated on a held-out target test set ($\approx 42$k samples). Across three training seeds, naive source validation selects late epochs (8, 18, 21), after the model has overfit to source-specific features, with a mean target log-loss of $2.09$. Pseudo-labeling does not always recover the exact oracle epoch, but it consistently selects an early-stopping region and beats naive validation in every run, with a mean target log-loss of $0.82$ (oracle: $0.43$), i.e., a mean regret of $0.39$ versus $1.66$ for naive validation. Results are stable across target selection/test splits. Per-seed results are reported in \cref{tab:early stopping seeds}, and further details are in \cref{app: wilds exp}.

\begin{figure}[ht]
    \centering
    \includegraphics[width=0.8\linewidth]{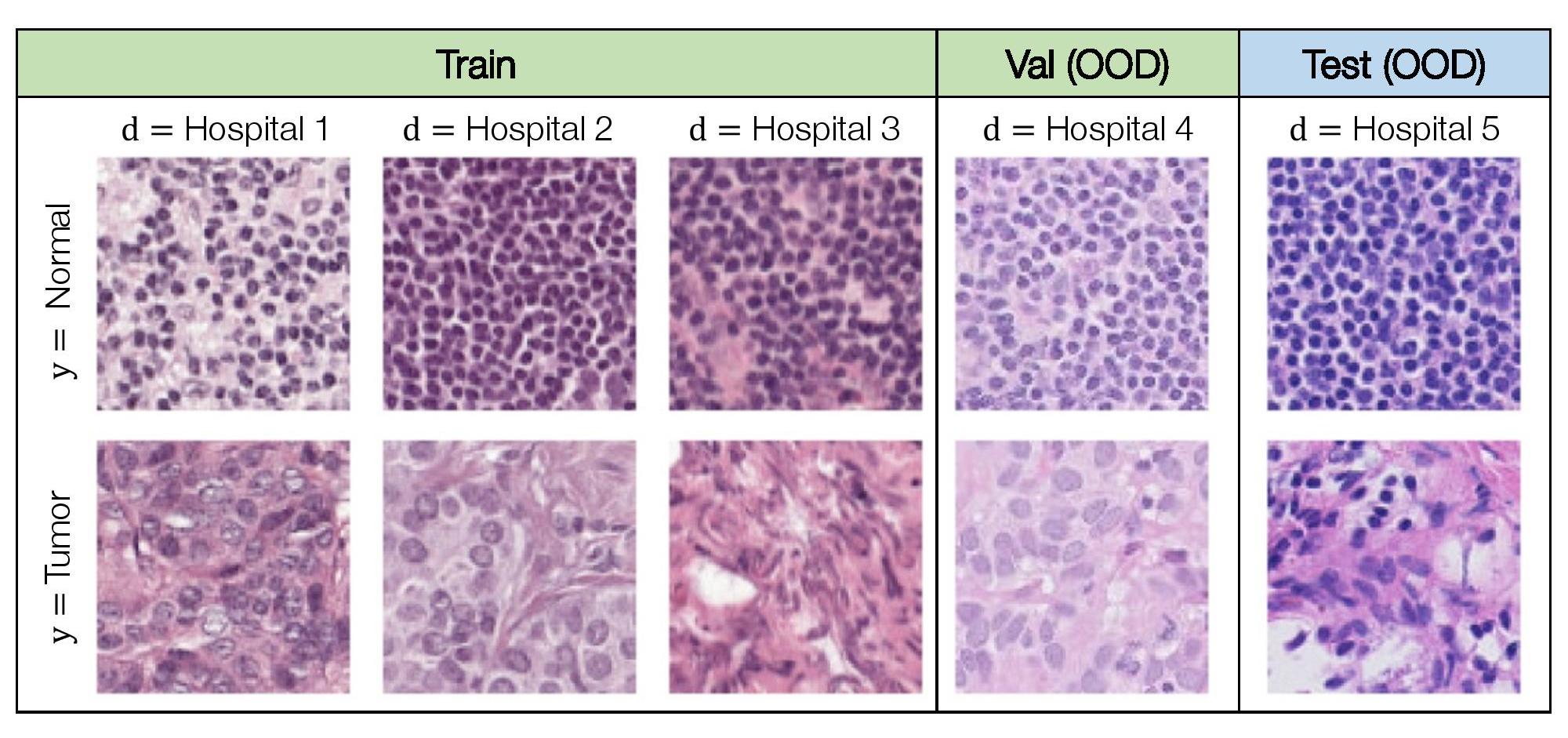}
    \caption{Camelyon17 patches across hospitals, from \citet{koh2021wildsbenchmarkinthewilddistribution}. We use hospitals 1--3 as the labeled source domain and hospital 5 as the unlabeled target domain; hospital 4 is not used.}
    \label{fig:camelyon}
\end{figure}

\begin{table}[ht]
  \caption{Camelyon17 early stopping across three training seeds: selected epochs, target test NLL, and regret with respect to oracle selection.}
  \label{tab:early stopping seeds}
  \centering
  \small
  \begin{tabular*}{\columnwidth}{@{\extracolsep{\fill}} l ccc ccc cc @{}}
    \toprule
    & \multicolumn{3}{c}{Selected epoch} & \multicolumn{3}{c}{Target test NLL} & \multicolumn{2}{c}{Regret} \\
    \cmidrule(lr){2-4}\cmidrule(lr){5-7}\cmidrule(lr){8-9}
    Training seed & Naive & Pseudo & Oracle & Naive & Pseudo & Oracle & Naive & Pseudo \\
    \midrule
    Run 1 & 8  & 1  & 1 & 1.94 & 0.66 & 0.66 & 1.29 & 0.00 \\
    Run 2 & 18 & 2  & 5 & 2.28 & 1.08 & 0.32 & 1.96 & 0.76 \\
    Run 3 & 21 & 13 & 2 & 2.05 & 0.71 & 0.31 & 1.74 & 0.40 \\
    \midrule
    Mean  &    &    &   & 2.09 & 0.82 & 0.43 & 1.66 & 0.39 \\
    \bottomrule
  \end{tabular*}
\end{table}


\section{Conclusion}
We presented a principled framework for unsupervised domain adaptation in kernel GLMs, utilizing a split-and-fit strategy to perform target-aware model selection via pseudo-labeling. Our non-asymptotic excess risk bounds characterize adaptation through an ``effective labeled sample size'', and our experiments confirm consistent gains over source-only and reweighting baselines. Our pure covariate shift assumption remains a limitation, as real-world data typically exhibit a mixture of distribution shifts. Several future directions are worth exploring. Extending our framework beyond ridge regularization and RKHS estimators would broaden its applicability. Also, our analysis focuses on one-step imputation; analyzing iterative refinement, where pseudo-labels are updated over multiple rounds, could connect our bounds to the empirical success of self-training.

\section*{Acknowledgment}
The research is supported by NSF grants DMS-2210907 and DMS-2515679, and a Data Science Institute seed grant SF-181 at Columbia University.

{
\bibliographystyle{ims}
\bibliography{bib}
}

\newpage

\begin{appendices}
    \section*{Appendices} 
    
    \startcontents[appendices]
    \printcontents[appendices]{l}{1}{\setcounter{tocdepth}{2}}
    \vspace{1em}
    \hrule
    \vspace{1em}

    \renewcommand{\theHequation}{\thesection.\arabic{equation}}
    
    \renewcommand{\theHlemma}{\thesubsection.\arabic{lemma}}
    \renewcommand{\theHtheorem}{\thesection.\arabic{theorem}}
    \renewcommand{\theHcorollary}{\thesection.\arabic{corollary}}
    \renewcommand{\theHfact}{\thesection.\arabic{fact}}

\section{Preliminaries}
\subsection{Notations}\label[appendix]{subsection notations}

To facilitate the discussion and analysis, we introduce some additional notations. When $\mathcal{X}=\mathbb{R}^d$ for some $d<\infty$ and $K(\boldsymbol{z}, \boldsymbol{w})=\boldsymbol{z}^{\top} \boldsymbol{w}, \Phi$ is the identity mapping. We can construct design matrices $\boldsymbol{X}=\left(\boldsymbol{x}_1, \cdots, \boldsymbol{x}_n\right)^{\top} \in$ $\mathbb{R}^{n \times d}, \boldsymbol{X}_0=\left(\boldsymbol{x}_{01}, \cdots, \boldsymbol{x}_{0 n_0}\right)^{\top} \in \mathbb{R}^{n_0 \times d}$ and the response vector $\boldsymbol{y}=\left(y_1, \cdots, y_n\right)^{\top} \in \mathbb{R}^n$. Define the index set $\mathcal{T}=\left\{i \in[n]:\left(\boldsymbol{x}_i, y_i\right) \in \mathcal{D}_1\right\}$ of the data for training candidate models. In addition, let $n_1=|\mathcal{T}|$ and $n_2=n-n_1$. Denote by $\boldsymbol{X}_1 \in \mathbb{R}^{|\mathcal{T}| \times d}$ and $\boldsymbol{X}_2 \in \mathbb{R}^{(n-|\mathcal{T}|) \times d}$ the sub-matrices of $\boldsymbol{X}$ by selecting rows whose indices belong to $\mathcal{T}$ and $[n] \backslash \mathcal{T}$, respectively. Similarly, denote by $\boldsymbol{y}_1 \in \mathbb{R}^{|\mathcal{T}|}$ and $\boldsymbol{y}_2 \in \mathbb{R}^{n-|\mathcal{T}|}$ the sub-vectors of $\boldsymbol{y}$ induced by those index sets.

The design matrices can be generalized to operators when $\mathcal{X}$ and $K$ are general. In that case, $\boldsymbol{X}$ and $\boldsymbol{X}_0$ are bounded linear operators defined through

$$
\begin{aligned}
& \boldsymbol{X}: \mathbb{H} \rightarrow \mathbb{R}^n, \boldsymbol{\btheta} \mapsto\left(\left\langle\Phi\left(\boldsymbol{x}_1\right), \boldsymbol{\btheta}\right\rangle, \cdots,\left\langle\Phi\left(\boldsymbol{x}_n\right), \boldsymbol{\btheta}\right\rangle\right)^{\top}, \\
& \boldsymbol{X}_0: \mathbb{H} \rightarrow \mathbb{R}^{n_0}, \boldsymbol{\btheta} \mapsto\left(\left\langle\Phi\left(\boldsymbol{x}_{01}\right), \boldsymbol{\btheta}\right\rangle, \cdots,\left\langle\Phi\left(\boldsymbol{x}_{0 n_0}\right), \boldsymbol{\btheta}\right\rangle\right)^{\top} .
\end{aligned}
$$

We can define $\boldsymbol{X}_1$ and $\boldsymbol{X}_2$ similarly. With slight abuse of notation, we use $\boldsymbol{X}^{\top}$ to refer to the adjoint of $\boldsymbol{X}$, i.e. $\boldsymbol{X}^{\top}: \mathbb{R}^n \rightarrow \mathbb{H}, \boldsymbol{u} \mapsto \sum_{i=1}^n u_i \Phi\left(\boldsymbol{x}_i\right)$. Similarly, we can define $\boldsymbol{X}_0^{\top}, \boldsymbol{X}_1^{\top}$ and $\boldsymbol{X}_2^{\top}$. Let $\widehat{\boldsymbol{\bSigma}}_0=\frac{1}{n_0} \sum_{i=1}^{n_0} \Phi\left(\boldsymbol{x}_{0 i}\right) \otimes \Phi\left(\boldsymbol{x}_{0 i}\right), \widehat{\boldsymbol{\bSigma}}_1=\frac{1}{n_1} \sum_{i \in \mathcal{T}} \Phi\left(\boldsymbol{x}_i\right) \otimes \Phi\left(\boldsymbol{x}_i\right)$ and $\widehat{\boldsymbol{\bSigma}}_2=\frac{1}{n_2} \sum_{i \in[n] \backslash \mathcal{T}} \Phi\left(\boldsymbol{x}_i\right) \otimes \Phi\left(\boldsymbol{x}_i\right)$. We have $\widehat{\boldsymbol{\bSigma}}_j=n_j^{-1} \boldsymbol{X}_j^{\top} \boldsymbol{X}_j$ for all $j \in\{0,1,2\}$. 

We will also use the following notation for any $\lambda >0$:
\begin{equation*}
\hat f_\lambda (\cdot) = \ip{\Phi(\cdot)}{\hthetal}\,, \quad f^*(\cdot) = \ip{\Phi(\cdot)}{\btheta^*}\,, \quad \tilde f(\cdot) = \ip{\Phi(\cdot)}{\tilde \btheta},
\end{equation*}

where $\tilde \btheta = \hat \btheta_{\tilde{\lambda}}.$

Also denote $g(\btheta) := (\logpar'(\ip{\Phi(\bx_i)}{\btheta})-y_i)_i^{\top}$ the vector of residuals such that $\bres_i = g_i(\btheta^*)$.

\begin{remark}[Convention]
    We will adopt the convention that the minimum of a function $f$ over an empty set is infinite: \begin{equation*}
    \min_{\lambda \in \emptyset} f(\lambda) = +\infty.
    \end{equation*}
\end{remark}


\subsection{Useful observations on the GLM excess risk}\label[appendix]{sec: excess risk obs}
Note that, on the domain where \cref{convexityassumption} holds, such that the second derivative of the log-partition function is bounded, $k\le \logpar'' \le K$, then, for the true excess risk, we have by Taylor-Lagrange:
\begin{equation}\label{excess risk bound}
    \frac{k}{2}\mathbb{E}_{\bx \sim \mathcal{Q}}\abs{f(\bx)-f^*(\bx)}^2 \le \cR(f) \le \frac{K}{2}\mathbb{E}_{\bx \sim \mathcal{Q}}\abs{f(\bx)-f^*(\bx)}^2,
\end{equation}

and we recognize the usual mean squared estimation error under the target distribution $\mathcal{Q}$.

Likewise, for the in-sample excess risk \eqref{eq: in-sample risk}, we get
\begin{equation}\label{in-sample excess risk bound}
\frac{k}{2n} \sum_{i=1}^n\paren{f(\bx_i)-f^*(\bx_i)}^2 \le \cR_{in}(f) \le \frac{K}{2n}\sum_{i=1}^n\paren{f(\bx_i)-f^*(\bx_i)}^2.
\end{equation}

\begin{remark}
    The excess risk under the GLM loss is a type of Bregman divergence. The Bregman divergence of the strictly convex and continuously differentiable log-partition function $\logpar$ is defined as:
    \begin{equation*}
    D_{\logpar}(x, y):= \logpar(x)-\logpar(y)-\logpar'(y)(x-y), \forall x, y \in \Omega,
    \end{equation*}
    where $\Omega$ is a closed convex set in $\RR$. Hence, we can rewrite the excess risk and in-sample excess risk respectively as:
    \begin{equation*}
    \cR(f)= \mathbb{E}_{\bx \sim \mathcal{Q}}\cro{D_{\logpar}(f(\bx), f^*(\bx))},
    \end{equation*}
    and 
    \begin{equation*}
    \cR_{in}(f) = \frac{1}{n} \sum_{i=1}^n D_{\logpar}(f(\bx_i), f^*(\bx_i)).
    \end{equation*}
    
\end{remark}

\begin{remark}[Excess risk bounds]\label{remark excess risk}
    Based on \cref{excess risk bound,in-sample excess risk bound} and the notations introduced, we have for any $\lambda >0$:
    \begin{equation*}
    \frac{k}{2}\norm{\hthetal-\btheta^*}_{\Sigma_0}^2 \le \cR(\hat f_\lambda) \le \frac{K}{2}\norm{\hthetal-\btheta^*}_{\bSigma_0}^2,
    \end{equation*}
    and
    \begin{equation*}
    \frac{k}{2}\norm{\hthetal-\btheta^*}_{\hat \Sigma_0}^2 \le \cR_{\textit{in}}(\hat f_\lambda) \le \frac{K}{2}\norm{\hthetal-\btheta^*}_{\hat \bSigma_0}^2.
    \end{equation*}
    
\end{remark}

\subsection{Proof structure}\label[appendix]{app: proof structure}

\Cref{sec: proof sketch} presents the two building blocks of our proof and explains how to combine them informally: those are the in-sample oracle-inequality (\cref{oracle general non approx}) proved in \cref{poracle} and the pointwise convergence rate of the KRGLM estimator (\cref{consistency rate}) proved in \cref{proofres}. \Cref{proof of theorems} formalizes this reasoning to prove our main \cref{main theorem}. To keep the logic of the proof clear, \cref{proof of theorems} states and combines a number of results that are proved in \cref{proof of appendix B}. Throughout, we make use of technical lemmas presented in \cref{lemmas}.

\section{Proof of \texorpdfstring{\cref{main theorem}}{Theorem \getrefnumber{main theorem}} and corollaries}\label[appendix]{proof of theorems}

In this section, we will use the following notations for $j\in \{1,2\}$:
\begin{equation*}
\hat \bC_{j, \lambda}=(k\hat \bSigma_j + \lambda \bI)^{-1}\hat \bSigma_j\,,\quad \hat \bS_{j, \lambda} = (k\hat \bSigma_j + \lambda \bI)^{-1/2}\hat \bSigma_0(k\hat \bSigma_j + \lambda \bI)^{-1/2},
\end{equation*}
their population version
\begin{equation*}
\bC_{\lambda}=(k\bSigma + \lambda \bI)^{-1}\bSigma\,,\quad \bS_{\lambda} = (k\bSigma + \lambda \bI)^{-1/2}\bSigma_0(k\bSigma + \lambda \bI)^{-1/2},
\end{equation*}
as well as the hybrid
\begin{equation*}
\hat \bU_{j, \lambda}=(k\hat \bSigma_j + \lambda \bI)^{-1/2}\bSigma_0(k\hat \bSigma_j + \lambda \bI)^{-1/2}.
\end{equation*}

To be able to use the in-sample oracle inequality, we first need to conduct a fixed-design analysis by treating the covariates and the data split as deterministic.

\subsection{Fixed-design analysis}

\begin{lemma}[Fixed-design KRGLM]\label{fixed design lemma}
    Let \Cref{setup,kernel assumptions,subexponentialassumption,convexityassumption,trueparam} hold except that all covariates are deterministic and the data split is fixed. Choose any $\lambda$ satisfying \cref{consistency rate} conditions:
    \[
  \begin{cases}
    \lambda \le \frac{g(n_1)}{R^4}=\frac{1}{R^4\log^{\beta}(n_1)} \\
    \lambda  \ge \paren{\frac{M_1\resv\sqrt{\log(16/\delta)}\log n_1}{\sqrt{n_1}}}^{2}\frac{1}{g(n_1)^{1/2}} = \frac{M_1^2\resv^{2}\log(\frac{16}{\delta})(\log n_1)^{2+\beta/2}}{n_1}
  \end{cases}
\]
Choose any positive semi-definite trace-class linear operator $\bQ : \HH \rightarrow \HH$ and define $\hat \bT_{1, \lambda} = (k\hat \bSigma_1 + \lambda \bI)^{-1/2}\bQ(k\hat \bSigma_1 + \lambda \bI)^{-1/2}$. Then, learning with the first half of the source dataset $\cD_1$, there is a universal constant $C$ such that for all $\delta \in (0, 1/e]$, \wpd:
\begin{equation*}
\norm{\hthetal-\btheta^*}_{\bQ}^2 \le 4\paren{\lambda R^2\norm{\hat \bT_{1, \lambda}}+\frac{C\sigma^2\log(1/\delta)\log^2(n_1)}{kn_1}\paren{L^2R^2 M_1^4\norm{\hat \bT_{1, \lambda}} + \Tr(\hat \bT_{1, \lambda})}}.
\end{equation*}

\medskip
On the other hand, for $\tilde \lambda$ set to its lower bound
\begin{equation*}
\frac{M_1^2\resv^{2}\log(\frac{16n_0}{\delta})(\log n_2)^{2+\beta/2}}{n_2}.
\end{equation*}
Then, learning on the second half of the source dataset $\cD_2$, $\bX_0(\tilde \btheta-\btheta^*)$, where $\hat \btheta_{\tilde \lambda} = \tilde \btheta$, can be decomposed into three terms $\bX_0(\tilde \btheta-\btheta^*) = \bB_1 + \bB_2 + \vari$, respectively bias, perturbation and variance, satisfying simultaneously, with probability at least $1-\delta/2$, for all $\delta \in (0, 1/e]$:
\begin{itemize}
    \item $\norm{\bB_1}_2^2 \le n_0 \tilde\lambda R^2\norm{\hat \bS_{2, \tilde \lambda}} \le C M_1^2\resv^{2}R^2\frac{n_0}{n_2}\norm{\hat \bS_{2, \tilde \lambda}}\log(\frac{n_0}{\delta})(\log n_2)^{6}$,
    \item $\norm{\bB_2}_2^2 \le C M_1^4 \sigma^2\frac{L^2n_0}{k^2n_2}\norm{\hat \bS_{2, \tilde \lambda}}\log^7(n_2/\delta)$,
    \item $\norm{\vari}_{\psi_1}^2 \le \frac{n_0}{n_2}\frac{\sigma^2}{k}\norm{\hat \bS_{2, \tilde\lambda}}$ (this is deterministic), and $\norm{\vari}_2^2 \le C\frac{n_0}{n_2}\frac{\sigma^2}{k}\Tr(\hat \bS_{2, \tilde\lambda})\log^3(n_2/\delta)$.
\end{itemize}

The variance term is sub-exponential (conditionally on $\bX_2$). $C$ is universal constant.
\end{lemma}

The proof is deferred to \cref{proof of fixed design}. The next lemma allows us to convert the previous empirical error bounds to population guarantees, its proof is in \cref{proof of concentration}.

\begin{lemma}[Concentration of second-moment operators]\label{concentration of second lemma}
    For any $\lambda$ satisfying \cref{consistency rate} conditions when $n_1=n_2=n/2$:
    \[
  \begin{cases}
    \lambda \le \frac{c_1 g(n)}{R^4}=\frac{c_1}{R^4\log^{\beta}(n)} \\
    \lambda  \ge c_2 \paren{\frac{M_1\resv\sqrt{\log(32/\delta)}\log n}{\sqrt{n}}}^{2}\frac{1}{g(n)^{1/2}} = \frac{c_2M_1^2\resv^{2}\log(\frac{32}{\delta})(\log n)^{2+\beta/2}}{n}
  \end{cases}
\], and $\lambda_0 = \frac{16 M_2^2 \log(28n_0/\delta)}{n_0}$, we have simultaneously \wpd:
\begin{equation*}
\hat \bSigma_0 \preceq \frac{3}{2}\bSigma_0 + \frac{\lambda_0}{2}\bI\,, \quad  k\hat \bSigma_j + \lambda \bI \succeq \frac{1}{2}(k\bSigma + \lambda \bI), \forall j \in \{1,2\}.
\end{equation*}
Under this event, we have for $j\in\{1,2\} $:
\begin{equation*}
\Tr(\hCj)\le 2M_1^2\Tr(C_\lambda),
\end{equation*}
and
\begin{equation*}
\norm{\hUj} \le 2 \norm{\bS_\lambda}\,, \quad \Tr(\hUj) \le 2 \Tr(\bS_\lambda),
\end{equation*}
and 
\begin{equation*}
\norm{\hSj} \le 3\norm{\bS_\lambda} + \frac{\lambda_0}{\lambda}.
\end{equation*}

\end{lemma}

\subsection{Analysis of the selected candidate}
From the two previous lemmas, we obtain an excess risk bound for the best candidate in $\{\hat f_\lambda\}_{\lambda \in \Lambda}$. The proof can be found in \cref{proof of excess risk bound}.

\begin{lemma}[Excess risk of the best candidate]\label{epsilon risk lemma}
Let \Cref{setup,kernel assumptions,subexponentialassumption,convexityassumption,trueparam} hold. Choose any $\delta \in (0, 1/e]$, and define for any $\lambda >0$
\begin{equation*}
\cE(\lambda, \delta) = K\paren{\lambda R^2 \norm{\bS_\lambda} + \frac{\sigma^2\log(m/\delta)\log^2(n_1)}{kn_1}\paren{\Tr(\bS_\lambda)+ L^2R^2M_1^4\norm{\bS_\lambda}}}.
\end{equation*}
Then there exists a universal constant $C$ such that \wpd,
\begin{equation*}
\min_{\lambda \in \Lambda}\{R(\hat f_\lambda)-R(f^*)\} \le C \min_{\lambda \in \Lambda \cap J(\delta)}\cE(\lambda, \delta),
\end{equation*}
where
\begin{equation*}
J(\delta)=\cro{\frac{M_1^2\sigma^2\log(1/\delta)(\log n)^{2+\beta/2}}{n}\,, \frac{1}{R^4\log^\beta (n)}},
\end{equation*}
with the convention $\min{\emptyset}=\infty$.
    
\end{lemma}

The general oracle inequality from \cref{oracle general non approx} combined with \cref{fixed design lemma} yields the following result for the in-sample risk, and the proof is in \cref{proof of in sample oracle}.

\begin{lemma}[Oracle inequality for in-sample risk]\label{lemma in sample oracle}
Let \Cref{setup,kernel assumptions,subexponentialassumption,convexityassumption,trueparam} hold, except that all covariates are deterministic and the data split is fixed. Choose any $\delta \in (0, 1/e]$, and let 
$ \tilde \lambda = \frac{M_1^2\sigma^2\log(16n_0/\delta)(\log n_2)^{2+\beta/2}}{n_2}$, then define
\begin{equation*}
\hat \xi(\tilde \lambda, \delta)=\log^2(\frac{m}{\delta})\frac{(1+L^2)K^2R^2\sigma^2M_1^6\log^7(n_2/\delta)\log(n_0)}{k^2n_2}\paren{\norm{\hat \bS_{2, \tilde \lambda}} + 1}.
\end{equation*}
There exists a universal constant C such that \wpd,
\begin{equation*}
\sqrt{R_\textit{in}(\hat f)-R_\textit{in}(f^*)}\le \min_{\lambda \in \Lambda}\sqrt{R_\textit{in}(\hat f_\lambda)-R_\textit{in}(f^*)}+C\sqrt{\hat \xi(\tilde \lambda, \delta)}.
\end{equation*}
\end{lemma}

From \citep{wang2023pseudo}, we apply lemma D.5. The modified proof is in \cref{proof of bridge risk}.
\begin{lemma}[Bridging in-sample and out-of-sample risks]\label{bridge risks}
Let \Cref{setup,kernel assumptions,subexponentialassumption,convexityassumption,trueparam} hold, except that the source covariates are deterministic and the data split is fixed. Choose any $\delta \in (0, 1/e]$ and assume 
\begin{equation*}
\Lambda \subseteq \cro{\frac{M_1^2\sigma^2\log(8/\delta)(\log n_1)^{2+\beta/2}}{n_1}\,, \frac{1}{R^4\log^\beta (n_1)}}.
\end{equation*}
Then there exists a universal constant $C$ such that \wpd,
\begin{equation*}
\max_{\lambda \in \Lambda}\abs{\norm{\hthetal-\btheta^*}_{\Sigma_0}-\norm{\hthetal-\btheta^*}_{\hat \Sigma_0}} \le C\sqrt{M_1\sigma}\sqrt{\frac{\log^2(n_0m/\delta)\log(m/\delta)}{n_0}}.
\end{equation*}

\end{lemma}

With the same assumptions and definitions, the previous lemmas lead to:
\begin{corollary}\label{corollary out-of-sample oracle inequality}
    There exists a universal constant C such that with probability at least $1-2\delta$,
\begin{equation*}
\sqrt{R(\hat f)-R(f^*)}\le \frac{K}{k}\min_{\lambda \in \Lambda}\sqrt{R(\hat f_\lambda)-R(f^*)}+C\sqrt{\frac{K}{k}}\paren{\sqrt{\hat \xi(\tilde \lambda, \delta)}+\sqrt{\frac{M_1\sigma\log^2(n_0m/\delta)\log(m/\delta)}{n_0}}}.
\end{equation*}
\end{corollary}

The proof is in \cref{proof of out of sample oracle}. Further combining, we get:

\begin{corollary}[Oracle inequality for pseudo-labeling]\label{oracle inequality for pseudo-labeling}
Let \Cref{setup,kernel assumptions,subexponentialassumption,convexityassumption,trueparam} hold. Assume $n_1=n_2=n/2$. When $\tilde \lambda \asymp \frac{M_1^2\sigma^2\log(n_0/\delta)\log(n)^{2+\beta/2}}{n}$ and 
\begin{equation*}\Lambda \subseteq \cro{\frac{M_1^2\sigma^2\log(1/\delta)(\log n)^{2+\beta/2}}{n}\,, \frac{1}{R^4\log^\beta (n)}},
\end{equation*}
\wpd,
\begin{equation*}
\sqrt{R(\hat f)-R(f^*)}\le \frac{K}{k}\min_{\lambda \in \Lambda}\sqrt{R(\hat f_\lambda)-R(f^*)}+\zeta \cdot\sqrt{\frac{(1+L^2)K^3R^2M_1^6\resv^2}{k^3}(\frac{1}{n_{\text{eff}}}+\frac{1}{n_0})},
\end{equation*}

where $\zeta$ is a polylogarithmic factor in $n, n_0, m, \delta^{-1}$.
    
\end{corollary}

It implies up to a polylog factor, 
\begin{equation*}
R(\hat f)-R(f^*)\lesssim \frac{K^2}{k^2}\min_{\lambda \in \Lambda}\{R(\hat f_\lambda)-R(f^*)\}+\frac{(1+L^2)K^3R^2M_1^6\resv^2}{k^3}(\frac{1}{n_{\text{eff}}}+\frac{1}{n_0}).
\end{equation*}
Plugging the bound on the second term from \cref{epsilon risk lemma} yields the main theorem.

\begin{proof}

In this proof, we use $\lesssim$ to hide polylogarithmic factors in $(n, n_0, m, \delta^{-1})$. We relate the in-sample quantities from \cref{corollary out-of-sample oracle inequality} above to their population version via \cref{concentration of second lemma}.

On the one hand, from \cref{corollary out-of-sample oracle inequality}, we have

\begin{equation*}
\hat \xi(\tilde \lambda, \delta/4)=\log^2(\frac{4m}{\delta})\frac{(1+L^2)K^2R^2\sigma^2M_1^6\log^7(n/\delta)\log(n_0)}{k^2n}\paren{\norm{\hat \bS_{2, \tilde \lambda}} + 1}.
\end{equation*}


On the other hand, the concentration bounds from \cref{concentration of second lemma} give

\begin{equation*}
\norm{\hat \bS_{2, \tilde \lambda}} \le 3\norm{\bS_{\tilde \lambda}} + \frac{\lambda_0}{\tilde \lambda} \lesssim \frac{n}{n_{\text{eff}}} + \frac{n}{n_0},
\end{equation*}
where $\lambda_0 = \frac{16M_2^2\log(56n_0/\delta)}{n_0}$, and we used \cref{eq: operator norm s_lambda} since $\tilde \lambda > \mu^2/n$.
Putting everything together, since by construction $n_{\text{eff}} \le n$, we obtain:
\begin{equation*}
\hat \xi(\lambda, \delta/4) \lesssim \frac{(1+L^2)K^2R^2\sigma^2M_1^6}{k^2}\paren{\frac{1}{n_{\text{eff}}} + \frac{1}{n_0}}.
\end{equation*}

Since the additional term $\frac{M_1\sigma\log^2(n_0m/\delta)\log(m/\delta)}{n_0}$ is just another polylog divided by the target sample size, we deduce that \wpd,

\begin{equation*}
\sqrt{R(\hat f)-R(f^*)}\le \frac{K}{k}\min_{\lambda \in \Lambda}\sqrt{R(\hat f_\lambda)-R(f^*)}+\zeta \cdot\sqrt{\frac{(1+L^2)K^3R^2M_1^6\resv^2}{k^3}(\frac{1}{n_{\text{eff}}}+\frac{1}{n_0})}.
\end{equation*}
\end{proof}

\subsection{Proof of the main \texorpdfstring{\cref{main theorem}}{Theorem \getrefnumber{main theorem}}}\label[appendix]{proof of main theorem}
Combining via union bound \cref{oracle inequality for pseudo-labeling} and \cref{epsilon risk lemma}, we get that, for $\Lambda=\{\lambda_j\}_{j=1}^m$ defined in the theorem, the following holds \wpd:

\begin{equation*}R(\hat f)-R(f^*)\lesssim \frac{K^2}{k^2}\min_{\lambda \in \Lambda \cap J(\delta)}\{\mathcal{E}(\lambda, \delta)\}+\frac{(1+L^2)K^3R^2M_1^6\resv^2}{k^3}(\frac{1}{n_{\text{eff}}}+\frac{1}{n_0}),
\end{equation*}

and

\begin{align*}
\mathcal{E}(\lambda, \delta) & \le \frac{K}{k}\paren{\lambda R^2r^{-1} + \frac{\sigma^2\log^2(m/\delta)\log^2(n)}{n}(2r^{-1}\Tr((k\bSigma_0+\lambda r^{-1}\bI)^{-1}\bSigma_0)+L^2R^2M_1^4r^{-1}k^{-1})} \\
& \lesssim \frac{K\log^2(n)\log^2(m/\delta)}{k} \paren{\lambda R^2 \frac{n}{n_{\text{eff}}} + \frac{\sigma^2}{n_{\text{eff}}}\Tr((k\bSigma_0+\lambda r^{-1}\bI)^{-1}\bSigma_0)},
\end{align*}

where we used again that $\norm{\bS_\lambda}\le r^{-1}k^{-1}$ and $\Tr(\bS_\lambda) \le 2r^{-1}\Tr((k\bSigma_0+\lambda r^{-1}\bI)^{-1}\bSigma_0)$, and that the last perturbation term is negligible compared to the trace term of order $n_{\text{eff}}^{1/2\alpha_0}$.

The last step is to bridge the discrete grid with the continuous interval. We adapt the proof of theorem 4.1 \citep{wang2023pseudo}, by replacing the decision set $\Lambda$ by $\tilde\Lambda:=\Lambda \cap J(\delta)=\{\lambda_1,...,\lambda_p\}$, with $\lambda_1 \ge \frac{\mu^2}{n}\log(1/\delta)\log(n)^{2+\beta/2}$ and $\lambda_p \le \frac{1}{R^4\log(n)^\beta}$.

Denote 
\begin{equation*}
\min _{\lambda \in \tilde \Lambda}\{\underbrace{n R^2 \lambda}_{A(\lambda)}+\underbrace{\sigma^2 \operatorname{Tr}\left[\left(k\boldsymbol{\Sigma}_0+\mathrm{r}^{-1} \lambda \boldsymbol{I}\right)^{-1} \boldsymbol{\Sigma}_0\right]}_{B(\lambda)}\}.
\end{equation*}

We will finish the proof by establishing the following connection between the finite grid and the infinite interval:

\begin{equation}\label{eq: final claim}
\min _{\lambda \in \tilde\Lambda}\{A(\lambda)+B(\lambda)\} \lesssim \inf _{\lambda>0}\{A(\lambda)+B(\lambda)\}+R^2 \mu^2 .
\end{equation}

The functions $A$ and $B$ are increasing and decreasing, respectively. From the bound on $\lambda_1$ we get $A\left(\lambda_1\right)\asymp R^2 \mu^2$ and

\begin{equation}\label{eq: below interval}
\inf _{0<\lambda \leq \lambda_1}\{A(\lambda)+B(\lambda)\}+R^2 \mu^2 \geq B\left(\lambda_1\right)+A\left(\lambda_1\right).
\end{equation}

From $\lambda_p \asymp \frac{\mu^2}{R^4\log^\beta n}$ we get

\begin{equation*}
\inf _{\lambda \geq \lambda_p}\{A(\lambda)+B(\lambda)\} \geq A\left(\lambda_p\right) = nR^2\lambda_p \asymp \frac{n \mu^2}{R^2\log^\beta n}.
\end{equation*}

Meanwhile,

\begin{equation*}
B\left(\lambda_p\right)=\sigma^2 \operatorname{Tr}\left[\left(k\boldsymbol{\Sigma}_0+\mathrm{r}^{-1} \lambda_p \boldsymbol{I}\right)^{-1} \boldsymbol{\Sigma}_0\right] \lesssim \lambda_p^{-1/2\alpha_0} \asymp \log^{\beta/2\alpha_0}(n) \lesssim A(\lambda_p),
\end{equation*}
since $A(\lambda_p)$ grows linearly in $n$ up to logarithmic factors.

Therefore,

\begin{equation}\label{eq: above interval}
\inf _{\lambda \geq \lambda_p}\{A(\lambda)+B(\lambda)\} \gtrsim A\left(\lambda_p\right)+B\left(\lambda_p\right).
\end{equation}

Next, we bridge the grid $\tilde \Lambda$ and the interval $\left[\lambda_1, \lambda_p\right]$. For any $\lambda \in\left[\lambda_j, \lambda_{j+1}\right]$, we have $A(\lambda) \geq A\left(\lambda_j\right)$ and $k\boldsymbol{\Sigma}_0+\mathrm{r}^{-1} \lambda \boldsymbol{I} \preceq\left(\lambda / \lambda_j\right)\left(k\boldsymbol{\Sigma}_0+\mathrm{r}^{-1} \lambda_j \boldsymbol{I}\right)$. Hence,

\begin{equation*}
A(\lambda)+B(\lambda) \geq A\left(\lambda_j\right)+\frac{\lambda_j}{\lambda} B\left(\lambda_j\right) \geq \frac{A\left(\lambda_j\right)+B\left(\lambda_j\right)}{2}.
\end{equation*}

We get 

\begin{equation}\label{eq: in interval}
\inf_{\lambda_1\le \lambda \le\lambda_p}\{A(\lambda)+B(\lambda)\} \ge \frac{1}{2}\min_{\lambda \in \tilde \Lambda}\{A(\lambda)+B(\lambda)\}.
\end{equation}

The claim \eqref{eq: final claim} and the theorem follow from \cref{eq: above interval,eq: below interval,eq: in interval}.

Note that since $m=\lceil \log_2(n) \rceil + 1$, we make the pre-factor $\zeta$ a polylog in $(n, n_0, \delta^{-1})$ only.

\bigskip
\subsection{Proof of \texorpdfstring{\cref{cor:poly_decay_neff}}{Corollary \getrefnumber{cor:poly_decay_neff}}}\label[appendix]{proof of main coro}
\begin{proof}
Throughout the proof, we use $\lesssim$ to hide polylogarithmic factors and constants independent of $n$ and $n_{\text{eff}}$. From \cref{main theorem}, the excess risk is bounded by:
\begin{equation}
    R(\hat f) - R(f^*) \lesssim \inf_{\rho > 0} \underbrace{\left\{ \rho + \frac{1}{n_{\text{eff}}} \sum_{j=1}^{\infty} \frac{\mu_{0j}}{k\mu_{0j} + \rho} \right\}}_{A(\rho)} + \frac{1}{n_0} + \frac{1}{n_{\text{eff}}}.
\end{equation}
We analyze the term inside the sum. Observe that for any $j$:
\begin{equation}
    \frac{\mu_{0j}}{k\mu_{0j} + \rho} = \frac{1}{k} \left( \frac{k\mu_{0j}}{k\mu_{0j} + \rho} \right) \le \frac{1}{k} \min\left\{1, \frac{k\mu_{0j}}{\rho}\right\}.
\end{equation}
Substituting this bound and using the decay assumption $\mu_{0j} \le c j^{-2\alpha}$, we split the sum at an integer index $J$:
\begin{align*}
    \sum_{j=1}^{\infty} \frac{\mu_{0j}}{k\mu_{0j} + \rho} 
    &\le \frac{1}{k} \sum_{j=1}^{J} 1 + \frac{1}{k} \sum_{j=J+1}^{\infty} \frac{k c j^{-2\alpha}}{\rho} \\
    &= \frac{J}{k} + \frac{c}{\rho} \sum_{j=J+1}^{\infty} j^{-2\alpha} \\
    &\lesssim \frac{J}{k} + \frac{J^{1-2\alpha}}{\rho}.
\end{align*}
Substituting this back into $A(\rho)$:
\begin{equation}
    A(\rho) \lesssim \rho + \frac{1}{n_{\text{eff}}} \left( \frac{J}{k} + \frac{J^{1-2\alpha}}{\rho} \right).
\end{equation}
We choose the regularization parameter $\rho = n_{\text{eff}}^{-\frac{2\alpha}{2\alpha+1}}$ and the cutoff index $J = \lceil \rho^{-\frac{1}{2\alpha}} \rceil \asymp n_{\text{eff}}^{\frac{1}{2\alpha+1}}$. 
Checking the balance of terms (ignoring $k$ and constants for the rate):
\begin{itemize}
    \item $\rho = n_{\text{eff}}^{-\frac{2\alpha}{2\alpha+1}}$.
    \item $\frac{J}{n_{\text{eff}}} \asymp n_{\text{eff}}^{\frac{1}{2\alpha+1} - 1} = n_{\text{eff}}^{-\frac{2\alpha}{2\alpha+1}}$.
    \item $\frac{J^{1-2\alpha}}{\rho n_{\text{eff}}} \asymp \frac{(n_{\text{eff}}^{\frac{1}{2\alpha+1}})^{1-2\alpha}}{n_{\text{eff}}^{-\frac{2\alpha}{2\alpha+1}} n_{\text{eff}}} = n_{\text{eff}}^{-\frac{2\alpha}{2\alpha+1}}$.
\end{itemize}
Thus, $A(\rho) \lesssim n_{\text{eff}}^{-\frac{2\alpha}{2\alpha+1}}$, completing the proof.
\end{proof}

\bigskip
\subsection{The implicit requirement on the number of target covariates}\label[appendix]{app: n0 requirement}

We detail \cref{rem: n0}. In \cref{cor:poly_decay_neff}, the pseudo-label term $n_0^{-1}$ is negligible compared with the estimation rate if and only if $n_0\gtrsim n_{\mathrm{eff}}^{2\alpha/(2\alpha+1)}$. Since $2\alpha/(2\alpha+1)<1$, this requirement is sublinear in $n_{\mathrm{eff}}$, hence in $n$; for $\alpha=1$, it reads $n_0\gtrsim n_{\mathrm{eff}}^{2/3}$, and the exponent only approaches $1$ as $\alpha\to\infty$. As the requirement decreases with $n_{\mathrm{eff}}$, it is easiest to meet under hard shifts: for $\alpha=1$ and $n_{\mathrm{eff}}\asymp\sqrt n$ (the Sobolev example of \cref{sec: main results}), it reduces to $n_0\gtrsim n^{1/3}$. It binds under mild shifts, $n_{\mathrm{eff}}\asymp n$, with scarce target covariates (e.g., a rare subpopulation): if $n_0\lesssim n_{\mathrm{eff}}^{2\alpha/(2\alpha+1)}$, the bound is driven by the $O(n_0^{-1})$ pseudo-label error regardless of how much source data is available. The bound remains valid in all regimes.

\subsection{Optimal \texorpdfstring{$L^2(\mathcal{X} ; \cP)$}{L2} prediction rate}\label[appendix]{proof of L2 GLM rate}
We proved in \cref{consistency rate} a pointwise convergence rate of $O(n^{-1/4})$ for a generic ridge-regularized kernel GLM estimator, with polynomial eigendecay $\mu_j \lesssim j^{-2}$. Our analysis also allows us to derive the optimal $L^2$ prediction rate of this estimator. We will omit the details of the proof because it follows the exact same steps as our main proof. Below we present a proof sketch. $\lesssim$ hides constants and $\log$ factors.

We can use part 1 of \cref{fixed design lemma} with $\bQ = \bSigma$, combined with the concentration results of \cref{concentration of second lemma} to obtain an equivalent form of \cref{epsilon risk lemma} for one distribution only. 
Further using the fact that $\norm{\bC_\lambda} \lesssim 1$ and $\Tr(\bC_\lambda)\lesssim \lambda^{-1/2}$ (as per \cref{remark clambda}), we obtain
\begin{equation*}\norm{\hthetal-\btheta^*}_{\bSigma}^2 \lesssim \lambda R^2 + \frac{\sigma^2 \lambda^{-1/2}}{n}.
\end{equation*}

Minimizing on $\lambda$ yields the announced $\lambda^* \asymp n^{-2/3}$ and $\norm{\hat \btheta_{\lambda^*}-\btheta^*}_{\bSigma}^2 \lesssim n^{-2/3}$. 

Since by definition
\begin{equation*}\norm{\hat f_\lambda - f^*}_{L^2}^2 = \EE_{\bx \sim \cP} \left[\ip{\Phi(\bx)}{\hthetal-\btheta^*}^2\right] = \ip{\hthetal-\btheta^*}{\bSigma (\hthetal-\btheta^*)},
\end{equation*}
we directly get:
\begin{equation*}\norm{\hat f_{\lambda^*} - f^*}_{L^2} \lesssim n^{-1/3},
\end{equation*}
and the GLM excess risk
\begin{equation*}
\cR(f_{\lambda^*}) \lesssim n^{-2/3}.
\end{equation*}

\section{Proofs of \texorpdfstring{\cref{proof of theorems}}{Appendix \getrefnumber{proof of theorems}}}\label[appendix]{proof of appendix B}

\subsection{Proof of \texorpdfstring{\cref{fixed design lemma}}{Lemma \getrefnumber{fixed design lemma}}}\label[appendix]{proof of fixed design}
The key decomposition to use for this lemma comes from the proof of \cref{consistency rate} in \cref{proof of consistency rate}. With the same notations \footnote{In what follows, we denote by $\bH_i$, $\bD_i$, $\bDelta_i$ the corresponding quantities defined in the proof of \cref{consistency rate} when working with dataset $\cD_i$.}:
\begin{equation*}\forall \lambda>0, -(\hthetal-\btheta^*) = \lambda \bH^{-1}\btheta^* + (\bH^{-1} - \bar\bH^{-1})\frac{1}{n}\bX^{\top}\bres + \bar\bH^{-1}\frac{1}{n}\bX^{\top}\bres.
\end{equation*}

The first term is the bias term, the last term is a sub-exponential variance term, and the middle term is a perturbation that we regroup with the bias.

For the first part of the lemma, let $\bQ$ be any PSD trace-class operator and $\lambda$ in the specified range, working on the first half of the dataset. Per \cref{consistency rate} (and intermediate results obtained during its proof), with probability at least $1-\delta/2$, the following inequalities hold simultaneously:
\begin{equation}\label{H bound}
    \bH_1 \succeq k\hat \bSigma_1 + \lambda \bI,
\end{equation}

\begin{equation}\label{perturbation decomposition}
    \bH_1^{-1} - \bar\bH_1^{-1} = -\bar\bH_1^{-1/2}(\bDelta_1(\bI + \bDelta_1)^{-1})\bar\bH_1^{-1/2},
\end{equation}

\begin{equation}\label{delta bound}
    \norm{\bDelta_1} \lesssim \frac{LM_1}{k}\paren{\frac{\resv\log(n_1/\delta)\lambda^{-1/4}}{\sqrt{n_1}} + R\lambda^{1/4}},
\end{equation}

\begin{equation}\label{included tail bound}
    \norm{\frac{1}{\sqrt{n_1}}\bar\bH_1^{-1/2}\bX_1^{\top}\bres_1}^2 \le \resv^2\log^2(n_1)\log(8/\delta)\Tr(\hCone).
\end{equation}

Then:
\begin{equation*}\norm{\hthetal-\btheta^*}_Q^2 \le 4 \cro{\norm{\lambda \bH_1^{-1}\btheta^*}_Q^2 +\norm{(\bH_1^{-1} - \bar\bH_1^{-1})\frac{1}{n_1}\bX_1^{\top}\bres_1}_Q^2+\norm{\bar\bH_1^{-1}\frac{1}{n_1}\bX_1^{\top}\bres_1}_Q^2}.
\end{equation*}

We study the three terms separately.

\begin{itemize}
    \item For the first term, use \cref{H bound} 
    \begin{align*}\norm{\lambda \bH_1^{-1}\btheta^*}_Q^2 &\le \lambda^2 R^2 \norm{\bQ^{1/2}\bH_1^{-1}}^2 \le \lambda^2 R^2 \norm{\bH_1^{-1}}\,\norm{\bQ^{1/2}\bH_1^{-1}\bQ^{1/2}} \\
    &\le \lambda R^2 \norm{\bQ^{1/2}(k\hat \bSigma_1 + \lambda \bI)^{-1}\bQ^{1/2}} = \lambda R^2 \norm{\hat \bT_{1, \lambda}},\end{align*}
    where we wrote $\norm{\bQ^{1/2}\bH_1^{-1}}^2=\norm{(\bQ^{1/2}\bH_1^{-1/2})\bH_1^{-1}(\bH_1^{-1/2}\bQ^{1/2})}$, used $\norm{\bH_1^{-1}}\le 1/\lambda$, and used $\bH_1\succeq k\hat\bSigma_1+\lambda\bI$ together with the operator monotonicity of $t\mapsto -t^{-1}$ (note that $t\mapsto t^{-2}$ is not operator monotone, so $\bH_1^{-2}$ cannot be compared directly with $(k\hat\bSigma_1+\lambda\bI)^{-2}$). The last equality holds because $\norm{\bA\bA^\top}=\norm{\bA^\top\bA}$ for $\bA=\bQ^{1/2}(k\hat \bSigma_1 + \lambda \bI)^{-1/2}$.

    \item For the second term, use \cref{perturbation decomposition,included tail bound}:
    \begin{align*}
        \norm{(\bH_1^{-1} - \bar\bH_1^{-1})\frac{1}{n_1}\bX_1^{\top}\bres_1}_Q^2
        & = \frac{1}{n_1^2}\norm{\bQ^{1/2}\bar\bH_1^{-1/2}(\bDelta_1(\bI + \bDelta_1)^{-1})\bar\bH_1^{-1/2}\bX_1^{\top}\bres_1}^2 \\
        & \le \frac{1}{n_1}\norm{\bQ^{1/2}\bar\bH_1^{-1/2}}^2\cdot \norm{(\bDelta_1(\bI + \bDelta_1)^{-1})}^2\cdot \norm{\frac{1}{\sqrt{n_1}}\bar\bH_1^{-1/2}\bX_1^{\top}\bres_1}^2 \\
        & \lesssim \frac{1}{n_1}\norm{\hat \bT_{1, \lambda}}\norm{\bDelta_1}^2\resv^2\log^2(n_1)\log(8/\delta)\Tr(\hCone) \\
        & \lesssim \frac{1}{n_1}\norm{\hat \bT_{1, \lambda}}\norm{\bDelta_1}^2\resv^2\log^2(n_1)\log(8/\delta)M_1^2\Tr(\bC_\lambda),
    \end{align*}

    where the last inequality uses $\Tr(\hCone) \le 2M_1^2\Tr(\bC_\lambda)$.

Recall that $\Tr(\bC_\lambda) \lesssim \lambda^{-1/2}$ and from \cref{delta bound}, we get:
\begin{equation*}\norm{\bDelta_1}^2\Tr(\bC_\lambda) \lesssim \frac{L^2M_1^2}{k^2}\paren{\frac{\resv^2\log^2(n_1/\delta)}{n_1 \lambda} + R^2} \lesssim \frac{L^2M_1^2}{k^2}\paren{\frac{1}{M_1^2\log^{\beta/2}(n_1/\delta)}+R^2}\lesssim \frac{L^2M_1^2R^2}{k^2},\end{equation*}
so that 
\begin{equation*}\norm{(\bH_1^{-1} - \bar\bH_1^{-1})\frac{1}{n_1}\bX_1^{\top}\bres_1}_Q^2 \lesssim \frac{L^2}{k^2n_1}\norm{\hat \bT_{1, \lambda}}\resv^2\log^2(n_1)\log(8/\delta)R^2M_1^4.\end{equation*}

    \item For the third term, use a sub-exponential quadratic form tail bound that holds with probability larger than $1-\delta/2$:
    \begin{align*}
        \norm{\bar\bH_1^{-1}\frac{1}{n}\bX_1^{\top}\bres_1}_Q^2 
        & = \frac{1}{n_1^2}\ip{\bres_1}{\bX_1\bar\bH_1^{-1}\bQ\bar\bH_1^{-1}\bX_1^{\top}\bres_1} \\
        & \le \frac{\resv^2}{n_1^2} \Tr(\bX_1\bar\bH_1^{-1}\bQ\bar\bH_1^{-1}\bX_1^{\top})\log^2(n_1)\log(4/\delta) \\
        & \le \frac{\resv^2}{n_1} \Tr(\bar\bH_1^{-1}\bQ\bar\bH_1^{-1}\hat \bSigma_1)\log^2(n_1)\log(4/\delta) \\
        & = \frac{\resv^2}{n_1} \Tr(\bQ\bar\bH_1^{-1}\hat \bSigma_1\bar\bH_1^{-1})\log^2(n_1)\log(4/\delta) \\
        & \le \frac{\resv^2}{kn_1} \Tr(\bQ\bar\bH_1^{-1})\log^2(n_1)\log(4/\delta) \\
        & \le \frac{\resv^2}{kn_1} \Tr(\bQ(k\hat \bSigma_1 + \lambda \bI)^{-1})\log^2(n_1)\log(4/\delta) \\
        & = \frac{\resv^2}{kn_1} \Tr(\hat \bT_{1, \lambda})\log^2(n_1)\log(4/\delta).
    \end{align*}
    
\end{itemize}

The bound follows by union bound and redefining constant $C$.

\medskip
For the second part of the lemma, we work with the second half of the source dataset and $\tilde \lambda$ set to the lower bound of the range, so that \cref{consistency rate} now guarantees a $n^{-1/4}$ consistency rate (up to logarithmic factors), otherwise the three inequalities \cref{H bound,perturbation decomposition,included tail bound} are still valid with probability at least $1-\delta/2$ with $\bX_2$, $\bH_2$ $\bres_2$... and $\norm{\bDelta_2}^2\lesssim \frac{L^2}{k^2}M_1^3\sigma n_2^{-1/2}\log^6(n_2/\delta)$. We use the same decomposition multiplied by the target design matrix $\bX_0$:

\begin{equation*}\bX_0(\tilde \btheta-\btheta^*) =\bB_1 + \bB_2 + \vari,\end{equation*}
where
\begin{align*}
    \norm{\bB_1}_2^2 &= \norm{\tilde\lambda \bX_0 \bH_2^{-1}\btheta^*}^2 \\
    & \le \tilde\lambda^2 R^2 \norm{\bX_0 \bH_2^{-2}\bX_0^{\top}} \\
    & \le \tilde\lambda R^2 \norm{\bX_0 \bH_2^{-1}\bX_0^{\top}} \\
    & \le \tilde\lambda R^2 \norm{\bX_0 (k\hat \bSigma_2+\tilde\lambda \bI)^{-1}\bX_0^{\top}} \\
    & = n_0 \tilde\lambda R^2 \norm{\hat \bS_{2,\tilde \lambda}},
\end{align*}

\begin{align*}
    \norm{\bB_2}_2^2 &= \norm{\bX_0\bar\bH_2^{-1/2}(\bDelta_2(\bI + \bDelta_2)^{-1})\bar\bH_2^{-1/2}\frac{1}{n_2}\bX_2^{\top}\bres_2}_2^2 \\
    & \le \frac{1}{n_2}\norm{\bX_0\bar\bH_2^{-1/2}}_2^2 \cdot 2\norm{\bDelta_2}_2^2 \cdot \norm{\frac{1}{\sqrt{n_2}}\bar\bH_2^{-1/2}\bX_2^{\top}\bres_2}_2^2 \\
    & \lesssim \frac{L^2n_0}{k^2n_2}\norm{\hat \bS_{2,\tilde \lambda}} \cdot M_1^3\sigma n_2^{-1/2}\log^6(n_2) \cdot \resv^2\log^2(n_2)\log(8/\delta)\Tr(\hat \bC_{2,\tilde \lambda}) \\
    & \lesssim \frac{L^2n_0}{k^2n_2}\norm{\hat \bS_{2,\tilde \lambda}} \cdot M_1^3 n_2^{-1/2}\log^9(n_2/\delta) \cdot \resv^3 M_1^2 \tilde \lambda^{-1/2} \quad \text{since} \Tr(\hat \bC_{2,\tilde \lambda}) \le 2M_1^2 \Tr(\bC_\lambda) \lesssim 2M_1^2 \tilde \lambda^{-1/2} \\
    & \lesssim \frac{L^2n_0}{k^2n_2}\norm{\hat \bS_{2,\tilde \lambda}} \cdot M_1^3 n_2^{-1/2}\log^9(n_2/\delta) \cdot \resv^3 M_1^2 \cdot  n_2^{1/2} M_1^{-1}\sigma^{-1}\log^{-2}(n_2/\delta)\log^{-1}(n_0/\delta) \\
    & \lesssim \frac{L^2n_0}{k^2n_2}\norm{\hat \bS_{2,\tilde \lambda}}\log^7(n_2/\delta) \resv^2 M_1^4,
\end{align*}
and finally

\begin{align*}
    \norm{\vari}_{\psi_1}^2 & = \norm{\frac{1}{n_2}\bX_0\bar\bH_2^{-1}\bX_2^{\top}\bres_2}_{\psi_1}^2 \\
    & \le \sigma^2 \norm{\frac{1}{n_2}\bX_0\bar\bH_2^{-1}\bX_2^{\top}}_2^2 \\
    & = \sigma^2 \norm{\bX_0\bar\bH_2^{-1}\frac{1}{n_2^2}\bX_2^{\top}\bX_2\bar\bH_2^{-1}\bX_0^{\top}}_2 \\
    & \le \sigma^2 \norm{\bX_0\bar\bH_2^{-1}\frac{1}{n_2k}\bX_0^{\top}}_2 \quad \text{because} \quad \frac{1}{n_2}\bX_2^{\top}\bX_2 = \hat \bSigma_2 \preceq \frac{1}{k}\bar\bH_2 \\
    & \le \frac{\sigma^2 n_0}{k n_2} \norm{\hat \bS_{2,\tilde \lambda}},
\end{align*}
with the associated tail bound for $\norm{\vari}_{2}^2$, obtained through the quadratic form tail bound of sub-exponential vectors result of \cref{subexp hanson wright}.

\subsection{Proof of \texorpdfstring{\cref{concentration of second lemma}}{Lemma \getrefnumber{concentration of second lemma}}}\label[appendix]{proof of concentration}
Apply \cref{covarariance concentration} with $\gamma=1/2$ and note that when $n$ is large enough the new lower bound is larger than the one required on the lemma.

\subsection{Proof of \texorpdfstring{\cref{epsilon risk lemma}}{Lemma \getrefnumber{epsilon risk lemma}}}\label[appendix]{proof of excess risk bound}

Start by leveraging the observation made in \cref{excess risk bound}: $R(\hat f_\lambda) - R(f^*)\le \frac{K}{2}\norm{\hthetal-\btheta^*}_{\Sigma_0}^2$. 
Then apply \cref{fixed design lemma} with $\bQ=\bSigma_0$ and union bounds on all candidates to get error bounds with empirical second-moment operators. Further, apply \cref{concentration of second lemma} to relate them to the population versions. The two first lemmas are only valid in the range of $J(\delta)$.
Finally, use the fact that $\min_{\lambda \in \Lambda} \le \min_{\lambda \in \Lambda \cap J}$.

\subsection{Proof of \texorpdfstring{\cref{lemma in sample oracle}}{Lemma \getrefnumber{lemma in sample oracle}}}\label[appendix]{proof of in sample oracle}
We want to apply the in-sample oracle inequality from \cref{oracle general non approx} to our context. The samples of interest are the target covariates $\{\bx_{0i}\}_{i=1}^{n_0}$, and the pseudo-labels are obtained via the imputation model $\tilde f$ trained on the second half of the source dataset $\cD_2$. In our vector notation, the target pseudo-labels can be decomposed as in \cref{fixed design lemma}:
\begin{equation*}\bX_0\tilde \btheta = \bX_0\btheta^* + (\bB_1 + \bB_2) + \vari.\end{equation*}

We can verify that \cref{assumptions oracle} needed to apply the oracle inequality is verified directly per \Cref{kernel assumptions,subexponentialassumption,trueparam,convexityassumption}, with $C=RM_2$. The bias $\bB=\bB_1+\bB_2$ is random, which \cref{assumptions oracle} allows (\cref{rem: random bias}): we take $b^2 := 2(\bar b_1^2+\bar b_2^2)$, where $\bar b_1^2$ and $\bar b_2^2$ are the bounds on $\norm{\bB_1}_2^2$ and $\norm{\bB_2}_2^2$ from \cref{fixed design lemma}. They hold on the same event of probability at least $1-\delta/2$ as the bound $\eta$ on $\norm{\bX_0(\tilde \btheta - \btheta^*)}_\infty$, so item 3 of \cref{assumptions oracle} holds without changing the probability budget.
Therefore the oracle inequality gives that \begin{equation*}\sqrt{R_\textit{in}(\hat f)-R_\textit{in}(f^*)} - \min_{\lambda \in \Lambda}\sqrt{R_\textit{in}(\hat f_\lambda)-R_\textit{in}(f^*)}\end{equation*} (this intermediate inequality is obtained directly in the proof of the theorem) is smaller up to a constant than
\begin{equation}\label{eq: oracle inequality applied}
    \frac{\sqrt{K}}{\sqrt{n_0k}}\sqrt{K\variv^2 \log^2(\frac{m}{\delta}) + Kb^2 + \frac{L^2n_0\eta^4}{k}},
\end{equation}
with $b$ as above, $\variv$ is an upper bound on $\norm{\vari}_{\psi_1}$, and $\eta$ is an upper bound on $\norm{\bX_0(\tilde \btheta - \btheta^*)}_\infty$.

We take the square and consider each term separately, bounding them using \cref{fixed design lemma} (for $\bB_1, \bB_2, \norm{\vari}_{\psi_1}$) and \cref{consistency rate} (for $\eta$).

The variance term:
\begin{align*}
    \frac{K}{n_0k}K\norm{\vari}_{\psi_1}^2 \log^2(\frac{m}{\delta})
    & \le \frac{K^2\sigma^2}{k^2}\norm{\hat \bS_{2, \tilde\lambda}}\frac{\log^2(m/\delta)}{n_2}.
\end{align*}

The bias term:
\begin{align*}
    \frac{K}{n_0k}Kb^2
    & = \frac{2K^2}{n_0k}(\bar b_1^2+\bar b_2^2) \\
    & \lesssim \frac{K^2}{k}(\tilde \lambda R^2 \norm{\hat \bS_{2, \tilde\lambda}} + \frac{L^2M_1^4 \sigma^2}{k^2n_2}\norm{\hat \bS_{2, \tilde \lambda}}\log^7(n_2/\delta)) \\
    & \lesssim \frac{(1+L^2)K^2M_1^4\sigma^2R^2}{k^2}\norm{\hat \bS_{2, \tilde \lambda}}\frac{\log^7(n_2/\delta)\log(n_0)}{n_2}.
\end{align*}

Finally, the perturbation term:
\begin{align*}
    \frac{K}{n_0k}\frac{L^2n_0\eta^4}{k}
    & \lesssim \frac{L^2KM_1^6\sigma^2}{k^2}\frac{\log^{6}(n_2/\delta)}{n_2}.
\end{align*}

Combining the three terms together yields the desired result.

\subsection{Proof of \texorpdfstring{\cref{bridge risks}}{Lemma \getrefnumber{bridge risks}}}\label[appendix]{proof of bridge risk}
We adapt the proof of lemma 6.5 in \citep{wang2023pseudo}. The only difference is checking the assumptions of their lemma D.5. Use \cref{consistency rate} that gives for any $\epsilon \in (0, 1/2]$
\begin{equation*}\Pro{\abs{\ip{\Phi(\bx_0)}{\hthetal - \btheta^*}} \le cM_1 \left(\resv\log(\frac{n_1}{\epsilon})\frac{\lambda^{-1/4\alpha}}{\sqrt{n_1}} + \lambda^{1/4}\norm{\btheta^*}\right)} \ge 1-\epsilon.\end{equation*}

Plugging in the two bounds on $\lambda$, we obtain with probability larger than $1-\epsilon$:
\begin{equation*}\abs{\ip{\Phi(\bx_0)}{\hthetal - \btheta^*}} \lesssim \sqrt{M_1\resv}\log(\frac{1}{\epsilon})n_1^{-1/4}\log^5(n_1) +g(n_1)^{1/4} \le C \sqrt{M_1\resv}\log(\frac{1}{\epsilon})\end{equation*}
for $n$ large enough.

This exponential tail also yields finite moments such that 
\begin{equation*}\Ex{\abs{\ip{\Phi(\bx_0)}{\hthetal - \btheta^*}}^4}^{1/4} \le C\sqrt{M_1\resv}.\end{equation*}

Thus the inner product satisfies the conditions of lemma 6.5 with $r=C\log(1/\epsilon)\sqrt{M_1\resv}$ and $\bar U=C\sqrt{M_1\resv}$.
The rest follows as in the original proof.

\subsection{Proof of \texorpdfstring{\cref{corollary out-of-sample oracle inequality}}{Corollary \getrefnumber{corollary out-of-sample oracle inequality}}}\label[appendix]{proof of out of sample oracle}

We recall from \cref{remark excess risk} that, since $a$ is $k$-strongly convex and $K$-smooth, for every $\lambda\in\Lambda$,
\begin{equation}\label{out sample risk relat eq}
    \sqrt{\frac{k}{2}}\norm{\hthetal-\btheta^*}_{\Sigma_0}
    \le
    \sqrt{R(\hat f_\lambda)-R(f^*)}
    \le
    \sqrt{\frac{K}{2}}\norm{\hthetal-\btheta^*}_{\Sigma_0},
\end{equation}
and
\begin{equation}\label{in sample risk relat eq}
    \sqrt{\frac{k}{2}}\norm{\hthetal-\btheta^*}_{\hat \Sigma_0}
    \le
    \sqrt{R_{\textit{in}}(\hat f_\lambda)-R_{\textit{in}}(f^*)}
    \le
    \sqrt{\frac{K}{2}}\norm{\hthetal-\btheta^*}_{\hat \Sigma_0}.
\end{equation}
For compactness, define
\[
    c_1 := \sqrt{\frac{k}{2}},
    \qquad
    c_2 := \sqrt{\frac{K}{2}},
    \qquad
    \rho := \frac{c_2}{c_1}=\sqrt{\frac{K}{k}}.
\]

Moreover, \cref{bridge risks} gives, with high probability,
\begin{equation}\label{bridge ris eq}
    \max_{\lambda \in \Lambda}
    \abs{
    \norm{\hthetal-\btheta^*}_{\Sigma_0}
    -
    \norm{\hthetal-\btheta^*}_{\hat \Sigma_0}
    }
    \le x,
\end{equation}
where
\[
    x
    =
    C\sqrt{M_1\resv}
    \sqrt{
    \frac{
    \log^2(n_0m/\delta)\log(m/\delta)
    }{n_0}
    }.
\]
Finally, \cref{lemma in sample oracle} gives an inequality of the form
\begin{equation}\label{oracle eq}
    \sqrt{R_\textit{in}(\hat f)-R_\textit{in}(f^*)}
    \le
    \min_{\lambda \in \Lambda}
    \sqrt{R_\textit{in}(\hat f_\lambda)-R_\textit{in}(f^*)}
    + b,
\end{equation}
where
\[
    b=C\sqrt{\hat \xi(\tilde \lambda, \delta)}.
\]

We first compare the in-sample excess risk to the out-of-sample excess risk. For every $\lambda\in\Lambda$, by \cref{in sample risk relat eq} and \cref{bridge ris eq},
\begin{align*}
    \sqrt{R_{\textit{in}}(\hat f_\lambda)-R_{\textit{in}}(f^*)}
    &\le
    c_2 \norm{\hthetal-\btheta^*}_{\hat \Sigma_0} \\
    &\le
    c_2 \norm{\hthetal-\btheta^*}_{\Sigma_0}
    +
    c_2 x.
\end{align*}
Using the lower bound in \cref{out sample risk relat eq}, we obtain
\begin{equation}\label{in to out comparison}
    \sqrt{R_{\textit{in}}(\hat f_\lambda)-R_{\textit{in}}(f^*)}
    \le
    \rho
    \sqrt{R(\hat f_\lambda)-R(f^*)}
    +
    c_2 x.
\end{equation}
Therefore,
\begin{equation}\label{min in to min out comparison}
    \min_{\lambda\in\Lambda}
    \sqrt{R_{\textit{in}}(\hat f_\lambda)-R_{\textit{in}}(f^*)}
    \le
    \rho
    \min_{\lambda\in\Lambda}
    \sqrt{R(\hat f_\lambda)-R(f^*)}
    +
    c_2 x.
\end{equation}

We also compare the out-of-sample excess risk of the selected estimator to its in-sample excess risk. By \cref{out sample risk relat eq} and \cref{bridge ris eq}, with $\hat f=\hat f_{\hat\lambda}$,
\begin{align*}
    \sqrt{R(\hat f)-R(f^*)}
    &\le
    c_2 \norm{\hat\theta_{\hat\lambda}-\btheta^*}_{\Sigma_0} \\
    &\le
    c_2 \norm{\hat\theta_{\hat\lambda}-\btheta^*}_{\hat\Sigma_0}
    +
    c_2 x.
\end{align*}
Using the lower bound in \cref{in sample risk relat eq}, we get
\begin{equation}\label{out to in comparison}
    \sqrt{R(\hat f)-R(f^*)}
    \le
    \rho
    \sqrt{R_{\textit{in}}(\hat f)-R_{\textit{in}}(f^*)}
    +
    c_2 x.
\end{equation}

Plugging \cref{oracle eq} into \cref{out to in comparison}, we obtain
\begin{align*}
    \sqrt{R(\hat f)-R(f^*)}
    &\le
    \rho
    \left\{
    \min_{\lambda \in \Lambda}
    \sqrt{R_\textit{in}(\hat f_\lambda)-R_\textit{in}(f^*)}
    + b
    \right\}
    +
    c_2 x.
\end{align*}
Applying \cref{min in to min out comparison} yields
\begin{align*}
    \sqrt{R(\hat f)-R(f^*)}
    &\le
    \rho
    \left\{
    \rho
    \min_{\lambda\in\Lambda}
    \sqrt{R(\hat f_\lambda)-R(f^*)}
    +
    c_2 x
    +
    b
    \right\}
    +
    c_2 x \\
    &=
    \rho^2
    \min_{\lambda\in\Lambda}
    \sqrt{R(\hat f_\lambda)-R(f^*)}
    +
    c_2(1+\rho)x
    +
    \rho b.
\end{align*}
Since $\rho^2=K/k$, $c_2=\sqrt{K/2}$, and $\rho=\sqrt{K/k}$, we conclude that
\[
    \sqrt{R(\hat f)-R(f^*)}
    \le
    \frac{K}{k}
    \min_{\lambda\in\Lambda}
    \sqrt{R(\hat f_\lambda)-R(f^*)}
    +
    \sqrt{\frac K2}\left(1+\sqrt{\frac Kk}\right)x
    +
    \sqrt{\frac Kk}\,b.
\]
Substituting the definitions of $x$ and $b$ gives the desired result.

\section{Proofs of \texorpdfstring{\cref{generic GLM estimator}}{Section \getrefnumber{generic GLM estimator}}'s pointwise convergence rate}\label[appendix]{proofres}

\begin{remark}[Role of $\beta$]\label{rem: beta}
The parameter $\beta\in(0,8)$ in \cref{consistency rate} only trades off the two ends of the admissible window for $\lambda$, which is non-empty for $n$ large enough, and affects constants and logarithmic factors only. Since $2+\beta/2<6$, the penalty $\tilde\lambda=\mu^2\log^7(n)\log(n_0/\delta)/n$ of \cref{main theorem} satisfies the lower bound for every $\beta\in(0,8)$ simultaneously. At the lower end of the window, \eqref{eq: glm rate bias variance} gives a logarithmic factor of exponent at most $1+(2+\beta/2)/4\le 5/2$, reported uniformly as $(\log(n/\delta))^3$ in \eqref{eq: glm rate}, which is why $\beta$ does not appear there.
\end{remark}

Recall our notations: $k$ is strong convexity constant, $M_1$ is a uniform bound on the eigenfunctions of the RKHS, $M_2$ is a uniform bound on the covariates, $\resv$ is a bound on the sub-exponential norm of the noise vector, and $R$ is a bound on the true parameter's norm.

Using the isometry between the RKHS and the feature space \cref{eq: rkhs isometry}, an equivalent formulation of \cref{regularized glm bis} is

\begin{equation}\label{regularized glm}
    \min_{\btheta \in \mathbb{H}} \acc{\frac{1}{n} \sum_{i=1}^n (\logpar(\ip{\Phi(\bx_i)}{\btheta}) - y_i \ip{\Phi(\bx_i)}{\btheta}) + \frac{\lambda}{2}\norm{\btheta}_{\mathbb{H}}^2},
\end{equation}

In what follows, denote the regularized empirical GLM loss and its minimizer as:
\begin{equation*}L_{\lambda}(\btheta) = \frac{1}{n} \sum_{i=1}^n (\logpar(\ip{\Phi(\bx_i)}{\btheta}) - y_i \ip{\Phi(\bx_i)}{\btheta}) + \frac{\lambda}{2}\norm{\btheta}^2\,, \quad \hthetal \in \argmin_{\btheta \in \mathbb{H}} L_{\lambda}(\btheta),\end{equation*}

with gradient
\begin{equation*}\nabla L_{\lambda}(\btheta) = \frac{1}{n} \sum_{i=1}^n (\logpar'(\ip{\Phi(\bx_i)}{\btheta}) - y_i)\Phi(\bx_i)+ \lambda\btheta\,, \quad  \nabla L_{\lambda}(\btheta^*)= \frac{1}{n}\bX^{\top} \bres + \lambda \btheta^*,\end{equation*}
and Hessian
\begin{equation*}\nabla^2 L_{\lambda}(\btheta) = \frac{1}{n} \sum_{i=1}^n \logpar''(\ip{\Phi(\bx_i)}{\btheta})(\Phi(\bx_i) \otimes \Phi(\bx_i)) + \lambda \bI. \end{equation*}

We need the following definition in the rest of the proof:
\begin{definition}[Strong convexity region]
    Define $\Omega(\rho):=\{\btheta:\norm{\bX(\btheta-\btheta^*)}_{\infty}\le\rho C\}$ the domain over which the local strong convexity implies that $\forall \btheta \in\Omega(\rho)$, $\forall i \in [n], \logpar''(\ip{\Phi(\bx_i)}{\btheta}) \ge k_{\rho}$. Denote also $k:=k_1$, the smallest strong convexity constant when $\rho \le 1$.
\end{definition}

\begin{remark}
    It follows that $\forall \btheta \in\Omega(\rho), \nabla^2L_\lambda(\btheta) \succeq k_\rho \hSig + \lambda \bI$.
\end{remark}

To stay more general in this section, we will assume that the kernel eigenvalues in the $L^2(\cP, \cX)$ basis have a polynomial decay of exponent $\alpha \ge 1$: $\forall j, \mu_j \le Aj^{-2\alpha}$. Only at the end will we replace $\alpha$ by $1$, as per our \cref{kernel assumptions}.

\paragraph{GLM estimator properties}\label[appendix]{pestimator}
To characterize the estimation error of the estimator, we let $\btheta(t) = \btheta^* + t(\hthetal - \btheta^*)$ and write the Taylor expansion of $\nabla L_\lambda$ between $\hthetal$ and $\btheta^*$:

\begin{equation*}0= \nabla L_\lambda(\hthetal) = \nabla L_\lambda(\btheta^*) + \cro{\int_0^1\nabla^2 L_\lambda(\btheta(t))dt} \cdot(\hthetal - \btheta^*).\end{equation*}

Then take the inverse to get an expression of:
\begin{equation}\label{estimation error}
\hthetal - \btheta^* = - \paren{\frac{1}{n}\bX^{\top}\bD\bX + \lambda \bI}^{-1}\paren{\frac{1}{n}\bX^{\top} \bres + \lambda \btheta^*},
\end{equation}
where the diagonal matrix $\bD$ has components
\begin{equation*}D_{ii} = \int_0^1 \logpar''\paren{\ip{\Phi(\bx_i)}{\btheta(t)}}dt.\end{equation*}

However, to turn this non-linear equation into a usable expression, we need to know that $\btheta(t)$ is in the region of strong convexity to lower bound $\bD$ by $k\bI$. This motivates our first lemma which proves a crude consistency rate with high probability. This coarse-grained analysis will serve as a stepping stone to prove the main theorem.

\subsection{Coarse-grained analysis of the estimator consistency}

\begin{lemma}\label{crude bound}
    Let $g(n) = \frac{1}{\log^{\beta}(n)}\underset{n\rightarrow\infty}{\longrightarrow} 0, \quad \beta > 0$, and choose $\delta \in (0, 1/7]$. Then, under \Cref{setup,kernel assumptions,subexponentialassumption,convexityassumption,trueparam}, there exists constants $c, c_0, c_1 >0$ and $c_2\ge k$, such that when $\rho =  c_0g(n)^{1/4}$ and:
    \[
  \begin{cases}
    \lambda \le \frac{c_1 g(n)}{R^4} =  \frac{c_1}{R^4\log^{\beta}(n)}\\
    \lambda  \ge c_2 \paren{\frac{M_1^2\resv\sqrt{\log(4/\delta)}\log n}{\sqrt{n}}}^{2}\frac{1}{g(n)^{1/2}} = \frac{c_2M_1^2\resv^{2}\log(\frac{4}{\delta})(\log n)^{2+\beta/2}}{n}
  \end{cases}
\]

    We have:
    \begin{equation*}\Pro{\hthetal \in \mathring{\Omega}(\rho)} \ge 1- \delta,$$
    namely \wpd, $$\norm{\bX(\hthetal-\btheta^*)}_{\infty}\le c g(n)^{1/4}\end{equation*} 
\end{lemma}

\begin{proof}

By contradiction, assume that $\hthetal \notin \Omega(\rho)$. We need show that $\forall \btheta \notin \Omega(\rho), L_\lambda(\btheta)>L_\lambda(\btheta^*)$, for a well-chosen $\rho$, thus contradicting the optimality of $\hthetal$.
Inside $\Omega(\rho)$, we can use the local $k_\rho$-strong convexity. In this proof, we denote $\Sla:=k_\rho \hSig + \lambda \bI$.

\bigskip

Let the ellipsoid $E_\rho = \bigl\{\btheta:\|\Sla^{1/2}(\btheta-\btheta^*)\|\le2\,\|\Sla^{-1/2}\nabla L_\lambda(\btheta^*)\|\bigr\}$. 
We begin by proving the two following facts:

\begin{itemize}
    \item \textbf{Fact 1:} $\forall \btheta \in \mathring{\Omega}(\rho) \setminus E_\rho, L_\lambda(\btheta) > L_\lambda (\btheta^*)$.

    \begin{proof}
        Let $\btheta \in \mathring{\Omega}(\rho)$. Then we can apply Taylor Lagrange between $\btheta$ and $\btheta^*$ with a $\btheta_{in}=\btheta^* + t(\btheta-\btheta^*)\in \mathring{\Omega}(\rho)$ (for some $t\in(0,1)$) on the line segment connecting the two vectors. We can then use the strong convexity of $\logpar$ to write:
        \begin{align*}
            L_\lambda(\btheta) 
            &= L_\lambda(\btheta^*) + \ip{\nabla L_\lambda(\btheta^*)}{\btheta - \btheta^*} + \frac{1}{2}\ip{\btheta - \btheta^*}{\nabla^2L_\lambda(\btheta_{in})(\btheta - \btheta^*))} \\
            &\ge L_\lambda(\btheta^*) + \ip{\nabla L_\lambda(\btheta^*)}{\btheta - \btheta^*} + \frac{1}{2}\ip{\btheta - \btheta^*}{\Sla(\btheta - \btheta^*))},
        \end{align*}
        i.e. \begin{equation*}L_\lambda(\btheta) \ge L_\lambda(\btheta^*) + \ip{\Sla^{-1/2}\nabla L_\lambda(\btheta^*)}{\Sla^{1/2}(\btheta - \btheta^*)} + \frac{1}{2}\norm{\Sla^{1/2}(\btheta-\btheta^*)}^2.\end{equation*}
        Using Cauchy-Schwarz, we get:
        \begin{equation*}L_\lambda(\btheta)-L_\lambda(\btheta^*) \ge \norm{\Sla^{1/2}(\btheta-\btheta^*)}\paren{\frac{1}{2}\norm{\Sla^{1/2}(\btheta-\btheta^*)} - \norm{\Sla^{-1/2}\nabla L_\lambda(\btheta^*)}},\end{equation*}
        which yields the desired result.
    \end{proof}

    \item \textbf{Fact 2:} If \begin{equation}\label{inequality}
        \rho > \frac{2}{C}\max_i \acc{\norm{\Sla^{-1/2}\Phi(\bx_i)}\cdot\norm{\Sla^{-1/2}\nabla L_\lambda(\btheta^*)}},
    \end{equation}
    then $E_\rho$ is a proper subset of $\Omega(\rho)$.

    \begin{proof}
        Let $\btheta \in E_\rho, i\in [n]$. Then:
        \begin{equation*}\abs{\ip{\Phi(\bx_i)}{\btheta-\btheta^*}} = \abs{\ip{\Sla^{-1/2}\Phi(\bx_i)}{\Sla^{1/2}(\btheta-\btheta^*)}} \le \norm{\Sla^{-1/2}\Phi(\bx_i)}\cdot \norm{\Sla^{1/2}(\btheta-\btheta^*)}.\end{equation*}
        By definition of $E_\rho$, we get:
        \begin{equation*}\abs{\ip{\Phi(\bx_i)}{\btheta-\btheta^*}} \le 2\norm{\Sla^{-1/2}\Phi(\bx_i)}\cdot \norm{\Sla^{-1/2}\nabla L_\lambda(\btheta^*)}.\end{equation*}
        Thus, if $\rho > \frac{2}{C}\max_i \acc{\norm{\Sla^{-1/2}\Phi(\bx_i)}\cdot\norm{\Sla^{-1/2}\nabla L_\lambda(\btheta^*)}}$, $\btheta \in \mathring{\Omega}(\rho)$, and the result is proved. 
    \end{proof}
\end{itemize}

Together these two facts imply either that $\hthetal \notin \Omega(\rho)$ or that $\hthetal \in E_\rho \subset \mathring{\Omega}(\rho)$ by optimality of $\hthetal$. 
Let us prove that the first possibility is actually impossible. 

Let $\btheta \notin \Omega(\rho)$ and construct the line parametrized by $\btheta(t) = \btheta^* + t(\btheta-\btheta^*), t\in [0,1]$ linking $\btheta$ and $\btheta^*$. By continuity, it intersects $\Omega(\rho)$ and we define the intersection point as $t_b \in (0,1), \btheta(t_b) \in \partial \Omega(\rho)$. Given that $\btheta(t_b) \notin E_\rho$, we know from \textbf{Fact 1} that $L_\lambda(\btheta(t_b))\ge L_\lambda(\btheta^*)$. Then by convexity of $L_\lambda$ (true everywhere), $L_\lambda(\btheta) \ge \frac{1}{t_b}L_\lambda(\btheta(t_b)) + (1-\frac{1}{t_b})L_\lambda(\btheta^*)>L_\lambda(\btheta^*)$.

Finally, we have proved that with a well chosen $\rho$ satisfying \cref{inequality}, $\hthetal \in E_\rho \subset \mathring{\Omega}(\rho)$.

\bigskip

Let us now prove that the inequality \eqref{inequality} is feasible and extract the adequate $\rho$. This is a nonlinear inequality in $\rho$. However, we know that $k_\rho$ is decreasing as $\rho$ increases (since $\rho$ expands the region of strong convexity), which makes the RHS of \eqref{inequality} increasing in $\rho$. If we can assume that $\rho$ is upper bounded, say $\rho \le 1$, and then can show the existence of $\rho > \frac{2}{C}\max_i \acc{\norm{S_{\lambda, 1}^{-1/2}\Phi(\bx_i)}\cdot\norm{S_{\lambda, 1}^{-1/2}\nabla L_\lambda(\btheta^*)}}$ still smaller than $1$, we are done. We are going to prove that it is possible by computing explicitly the RHS for $\rho = 1$:

\paragraph{Computation of the RHS of \eqref{inequality} for $\rho=1$}

\begin{equation} \label{rhs}
\begin{split}
RHS & = \frac{2}{C}\max_i \acc{\norm{S_{\lambda, 1}^{-1/2}\Phi(\bx_i)}\cdot\norm{S_{\lambda, 1}^{-1/2}\nabla L_\lambda(\btheta^*)}} \\
 & = \frac{2}{C}\max_i \acc{\norm{(k\hSig +\lambda \bI))^{-1/2}\Phi(\bx_i)}\cdot\norm{(k\hSig +\lambda \bI)^{-1/2}(\frac{1}{n}\bX^{\top}\bres + \lambda \btheta^*))}} \\
 & \le \frac{2}{C}\max_i \acc{\norm{(k\hSig +\lambda \bI))^{-1/2}\Phi(\bx_i)}\cdot\paren{\frac{1}{n}\norm{(k\hSig +\lambda \bI)^{-1/2}\bX^{\top}\bres} + \lambda\norm{(k\hSig +\lambda \bI)^{-1/2}\btheta^*}}},
\end{split}
\end{equation}
where $k=k_1$.

\bigskip 

Define the event \begin{equation*}
    \mathcal{A} := \acc{\frac{1}{2}(k\bSigma + \lambda \bI)\preceq k\hSig + \lambda \bI \preceq \frac{3}{2}(k\bSigma + \lambda \bI)}.
    \end{equation*}

By \cref{covarariance concentration} applied with $\gamma = 1/2$, $\lambda = \lambda/k$ and $\delta = \delta/2$, which relates the sample and population covariance, since $\delta \in (0, 1/7] $ and our choice of $\lambda$ satisfies $\lambda \ge \frac{CM_2^2\log(2n/\delta)}{n}$ where $C\ge k > 0$ for $n$ large enough, we have

\begin{equation*}\Pro{\mathcal{A}} \ge 1-\frac{\delta}{2}.\end{equation*}

We consider the three norm terms of \cref{rhs} successively. We study the square norm to use inner products and then take the square-root. $\lesssim$ only hides universal multiplicative constants.

\begin{itemize}
    \item Under $\mathcal{A}$, $\ip{\Phi(\bx_i)}{(k\hSig + \lambda \bI)^{-1}\Phi(\bx_i)} \le 2\ip{\Phi(\bx_i)}{(k\bSigma + \lambda \bI)^{-1}\Phi(\bx_i)}$. We make use of \cref{effective dimension} to get: $\ip{\Phi(\bx_i)}{(k\bSigma + \lambda \bI)^{-1}\Phi(\bx_i)} \lesssim M_1^2\lambda^{-1/2\alpha}$. Therefore, taking the square-root, we obtain
    \begin{equation*}\norm{(k\hSig + \lambda \bI)^{-1/2}\Phi(\bx_i)} \lesssim M_1\lambda^{-1/4\alpha}.\end{equation*}

    \item Under $\mathcal{A}$, the bias term writes: $\lambda\norm{(k\hSig +\lambda \bI)^{-1/2}\btheta^*} \lesssim \lambda\norm{(k\bSigma +\lambda \bI)^{-1/2}\btheta^*}$. Then, using $(k\bSigma + \lambda \bI)^{-1} \preceq \bI/\lambda$, \begin{equation*}\norm{(k\bSigma + \lambda \bI)^{-1/2}\lambda \btheta^*} \le \sqrt{\lambda}\norm{\btheta^*} \lesssim \lambda^{1/2}\end{equation*} for bounded $\btheta^*$.

    \item For the squared variance term $\ip{\res}{\frac{1}{n^2}\bX(k\hSig +\lambda \bI)^{-1}\bX^{\top} \bres}$, conditioned on $\bX$, we use the quadratic form tail bound proved in \cref{subexp hanson wright} for the sub-exponential random vector $\res$ and PSD matrix $A=\frac{1}{n^2}\bX(k\hSig +\lambda \bI)^{-1}\bX^{\top}$, which bounds with high probability by the trace of the PSD operator. Since the $\delta$ defined earlier verifies $\delta \le 1/7 \le 4/e$, we have:
    
    \begin{equation*}\Pro{\ip{\res}{\frac{1}{n^2}\bX(k\hSig +\lambda \bI)^{-1}\bX^{\top} \bres} \lesssim \resv^2 \log^2(n)\log(4/\delta)\Tr(\frac{1}{n^2}\bX(k\hSig +\lambda \bI)^{-1}\bX^{\top}) \Bigm\vert \bX} \ge 1-\frac{\delta}{2},\end{equation*}
    and thus integrating over the whole event:
    \begin{equation*}\Pro{\ip{\res}{\frac{1}{n^2}\bX(k\hSig +\lambda \bI)^{-1}\bX^{\top} \bres} \lesssim \resv^2 \log^2(n)\log(4/\delta)\Tr(\frac{1}{n^2}\bX(k\hSig +\lambda \bI)^{-1}\bX^{\top}) \Bigm\vert \mathcal{A}} \ge 1-\frac{\delta}{2}.\end{equation*}
    
    Now, under $\mathcal{A}$ according to \cref{covarariance concentration}, 
    \begin{equation*}\Tr(\frac{1}{n^2}\bX(k\hSig +\lambda \bI)^{-1}\bX^{\top}) = \frac{1}{n}\Tr((k\hSig +\lambda \bI)^{-1}\hSig) \le \frac{2M_1^2}{n}\Tr((k\bSigma + \lambda \bI)^{-1}\bSigma).\end{equation*}
    Then applying again \cref{effective dimension}, we have that \begin{equation*}\Tr((k\bSigma + \lambda \bI)^{-1}\bSigma) = r(\lambda/k)/k \lesssim \lambda^{-1/2\alpha}.\end{equation*} Therefore:
    
    \begin{equation*}\Pro{\norm{\frac{1}{n}(k\hSig +\lambda \bI)^{-1/2} \bX^{\top} \bres} \lesssim \resv M_1\sqrt{\log(4/\delta)}n^{-1/2}\lambda^{-1/4\alpha}\log n \Bigm\vert \mathcal{A}} \ge 1-\frac{\delta}{2}.\end{equation*}

\end{itemize}

All in all, since our analysis holds for any $i$, there exists $c >0$, such that the event 
\begin{equation*}\mathcal{E}:=\acc{RHS \le c M_1\lambda^{-1/4\alpha}\paren{\lambda^{1/2}\norm{\btheta^*}+\resv M_1\sqrt{\log(4/\delta)}\lambda^{-1/4\alpha}\frac{\log n}{\sqrt{n}}}}\end{equation*}
has conditional probability
\begin{equation*}\Pro{\mathcal{E} \bigm\vert \mathcal{A}} \ge 1 - \frac{\delta}{2},\end{equation*}
and 
\begin{equation*}\Pro{\mathcal{E}} \ge \Pro{\mathcal{E} \bigm\vert \mathcal{A}} \Pro{\mathcal{A}}  \ge \left(1 - \frac{\delta}{2}\right)^2 \ge 1-\delta.\end{equation*}

\bigskip

\paragraph{Crude consistency rate.} Therefore, if we set 
\begin{equation}\label{sandwich lambda}
    \paren{\resv^{2}M_1^2\log(\frac{4}{\delta})}^{\alpha}\paren{\frac{\log n}{\sqrt{n}}}^{2\alpha}\frac{1}{g(n)^{\alpha/2}} \lesssim \lambda \lesssim \frac{g(n)}{R^4}
\end{equation}
for a slowly decaying $g(n) \to 0$ (and still respecting the condition of event $\mathcal{A}$ \footnote{For instance, we take $1/g(n)$ to be a polylog}), since $1/2 - 1/4\alpha \ge 1/4$, we have \wpd:
\begin{equation*}RHS \lesssim g(n)^{1/4} \to 0.\end{equation*}

We can simplify inequality \eqref{sandwich lambda} since it is still verified when setting $\alpha=1$, and with $g(n) = \frac{1}{\log^{\beta}(n)}, \beta>0$, so that it becomes

\begin{equation}\label{sandwich lambda bis}
    \frac{M_1^2\resv^{2}\log(\frac{4}{\delta})(\log n)^{2+\beta/2}}{n} \lesssim \lambda \lesssim \frac{1}{R^4\log^{\beta}(n)}.
\end{equation}

Namely, if we set $\rho = g(n)^{1/4}$, inequality \eqref{inequality} is verified and $\hthetal \in \Omega(\rho)$ so that the following consistency rate is established \wpd:
\begin{equation*}\norm{\bX(\hthetal - \btheta^*)}_{\infty} \lesssim g(n)^{1/4}=\frac{1}{\log^{\beta/4}(n)}.\end{equation*}

\end{proof}

\subsection{Proof of \texorpdfstring{\Cref{consistency rate}}{Theorem \getrefnumber{consistency rate}}}\label[appendix]{proof of consistency rate}

Building on the previous lemma, we refine the analysis of the consistency rate in the theorem.

\consistency*
\begin{remark}
    When union bounding on all the $\bx_0$ of a target $\bX_0$ of length $n_0$, we get with $g(n):=\frac{1}{\log^\beta(n)}$ that when \begin{equation*}\lambda  \asymp \paren{\frac{M_1\resv\sqrt{\log(8n_0/\delta)}\log n}{\sqrt{n}}}^{2}\frac{1}{g(n)^{1/2}},\end{equation*}
    \begin{equation*}\Pro{\norm{\bX_0(\hthetal-\btheta^*)}_{\infty} \lesssim \sqrt{M_1^3 \sigma}n^{-1/4}(\log (\frac{n}{\delta}))^3} \ge 1- \delta.\end{equation*}
\end{remark}
\begin{proof}
Recall \cref{estimation error}:

\begin{equation*}
\hthetal - \btheta^* = - \paren{\frac{1}{n}\bX^{\top}\bD\bX + \lambda \bI}^{-1}\paren{\frac{1}{n}\bX^{\top} \bres + \lambda \btheta^*}
\end{equation*}
where the diagonal matrix $\bD$ has components
\begin{equation*}D_{ii} = \int_0^1 \logpar''\paren{\ip{\Phi(\bx_i)}{\btheta^* + t(\hthetal - \btheta^*)}}dt.\end{equation*}

We denote $\bH:= \frac{1}{n}\bX^{\top}\bD\bX + \lambda \bI$ in what follows.

\medskip
Since $\bD$ and $\bH$ depend on $\hthetal$, they are random and correlated with $\res$, so let us consider their deterministic (in fixed-design) versions $\bar{D}_{ii}:= \logpar''\paren{\ip{\Phi(\bx_i)}{\btheta^*}}$ and $\bar{\bH}:= \frac{1}{n}\bX^{\top}\bar{\bD}\bX + \lambda \bI$. Let $\bx_0 \in \supp(P)$ and $\Phi_0:=\Phi(\bx_0)$ the associated vector in the RKHS. We can now write a sort of bias-variance decomposition plus perturbation from \eqref{estimation error}:

    \begin{equation}\label{bias variance perturbation}
        \abs{\ip{\Phi_0}{\hthetal - \btheta^*}} \le  \lambda \abs{\ip{\Phi_0}{\bH^{-1}\btheta^*}} + \abs{\ip{\Phi_0}{\frac{1}{n}\bar{\bH}^{-1}\bX^{\top} \bres}} +  \abs{\ip{\Phi_0}{\frac{1}{n}(\bH^{-1} -\bar{\bH}^{-1})\bX^{\top} \bres}}.
    \end{equation}

    We will show that the perturbation term is small.

    We let $k=k_1$ be the constant of strong convexity of the link function $\logpar$. This constant is the smallest allowed for the range of $\rho$ considered, i.e. associated to the largest region of strong convexity. $\logpar$ is thus $k$-strongly convex on $\Omega(\rho), \forall \rho \le 1$. As such, $\bar{\bD} \succeq k \bI$ and $\bar{\bH} \succeq k \hSig + \lambda \bI$, and if we let the event of \cref{crude bound} be:
    \begin{equation*}\mathcal{B}:=\acc{\norm{\bX(\hthetal-\btheta^*)}_{\infty}\lesssim g(n)^{1/4}},\end{equation*}
    which holds with probability larger than $1-\delta/2$, when 
    \[
  \begin{cases}
    \lambda \le \frac{c_1 g(n)}{\norm{\btheta^*}^4} \\
    \lambda  \ge c_2 \paren{\frac{M_1\resv\sqrt{\log(8/\delta)}\log n}{\sqrt{n}}}^{2\alpha}\frac{1}{g(n)^{\alpha/2}},
  \end{cases}
\]
 and $\delta \in (0,2/7]$, we also have a fortiori under $\mathcal{B}$ that $\bD \succeq k \bI$ and $\bH \succeq k \hSig + \lambda \bI$.
    Recall as well that under $\mathcal{B}$, the events from \cref{crude bound} 
    \begin{equation*}\mathcal{A}_1 = \acc{\frac{1}{2}(k\bSigma + \lambda \bI)\preceq k\hSig + \lambda \bI \preceq \frac{3}{2}(k\bSigma + \lambda \bI)},\end{equation*} and \begin{equation*}\mathcal{A}_2 = \acc{\norm{\frac{1}{\sqrt{n}}\paren{k\hSig + \lambda \bI}^{-1/2}\bX^{\top}\bres} \lesssim M_1\resv \sqrt{\log(8/\delta)}\log(n)\lambda^{-1/4\alpha}}\end{equation*} hold as well. \textbf{We place ourselves under that event $\mathcal{B}$ in the rest of the proof}, everything is conditional on $\bX$ in particular. We study the three terms of \eqref{bias variance perturbation} separately. Begin by noting the following fact.

    \begin{fact}[Relative error]\label{relative error}
   \begin{equation*}\forall i \in [n], \frac{\abs{D_{ii} - \bar{D}_{ii}}}{\bar{D}_{ii}} \le \frac{1}{2}\end{equation*} for $n$ large enough, the rank depending on the Lipschitz and strong convexity constant of the log-partition function $\logpar$ \footnote{E.g. for $g(n)=1/\log^4(n)$, with the logistic regression loss, $n$ must be larger than $28$. Compute other examples with explicit constants!}. It implies successively the following inequalities:
   \begin{equation*}-\frac{1}{2}\bar{\bH} \preceq \bH -\bar{\bH} \preceq \frac{1}{2}\bar{\bH}\,, \quad \norm{\bDelta} \le \frac{1}{2},\end{equation*}
   where $\bDelta := \bar{\bH}^{-1/2}\bH\bar{\bH}^{-1/2} -I$. We also get that $\norm{\bDelta}\le \frac{L}{2k}\norm{\bX(\hthetal-\btheta^*)}_{\infty}$.
    \end{fact}

    \begin{proof}
        \begin{align*}
            \abs{D_{ii} - \bar{D}_{ii}} 
            & = \abs{\int_0^1 \paren{\logpar''\paren{\ip{\Phi(\bx_i)}{\btheta^* + t(\hthetal - \btheta^*)}}-\logpar''\paren{\ip{\Phi(\bx_i)}{\btheta^*}}}dt} \\
            & \le \int_0^1 \abs{\logpar''\paren{\ip{\Phi(\bx_i)}{\btheta^* + t(\hthetal - \btheta^*)}}-\logpar''\paren{\ip{\Phi(\bx_i)}{\btheta^*}}}dt \\
            & \le \int_0^1 L t \abs{\ip{\Phi(\bx_i)}{\hthetal - \btheta^*}}dt
            \quad\text{by Lipschitz continuity of $\logpar''$}  \\
            & = \frac{L}{2} \abs{\ip{\Phi(\bx_i)}{\hthetal - \btheta^*}} \\
            & \lesssim L g(n)^{1/4}
            \quad\text{under event $\mathcal{B}$}.
        \end{align*}

        Since we also have $\bar{D}_{ii} \ge k$, we deduce
        \begin{equation*}\frac{\abs{D_{ii} - \bar{D}_{ii}}}{\bar{D}_{ii}} \lesssim \frac{L}{k}g(n)^{1/4} \le \frac{1}{2}\end{equation*} for $n$ large enough.

        Since they are diagonal matrices, it implies that
        \begin{equation*}-\frac{1}{2}\bar{\bD} \preceq \bD -\bar{\bD} \preceq \frac{1}{2}\bar{\bD},\end{equation*}
        and thus
        \begin{equation*}-\frac{1}{2}\bar{\bH} \preceq \bH -\bar{\bH} \preceq \frac{1}{2}\bar{\bH},\end{equation*}
        and 
        \begin{equation*}-\frac{1}{2}\bI \preceq \bDelta \preceq \frac{1}{2}\bI,\end{equation*}
        hence $\norm{\bDelta} \le \frac{1}{2}$.

        \medskip
        Also note that if we denote by $\bE = \bar{\bD}^{-1/2}(\bD-\bar{\bD})\bar{\bD}^{-1/2}$ the diagonal matrix of elements $\frac{D_{ii} - \bar{D}_{ii}}{\bar{D}_{ii}}$, and define the matrix $\bA = \frac{1}{\sqrt{n}}\bar{\bD}^{1/2}\bX\bar{\bH}^{-1/2}$, then 
        \begin{equation*}\bDelta = \bA^{\top} \bE \bA\,, \quad \text{and } \bA^{\top}\bA=\bI - \lambda \bar \bH\preceq\bI,\end{equation*}
        so that 
        \begin{equation*}\norm{\bDelta} \le \norm{\bE} \le \frac{L}{2k}\norm{\bX(\hthetal-\btheta^*)}_{\infty} \lesssim g(n)^{1/4}.\end{equation*}
    \end{proof}

Note also that the strong convexity property allows to lower bound $\bH, \bar{\bH}$ by $k\hSig + \lambda \bI$, so that using successively lemmas \ref{covarariance concentration}, \ref{effective dimension}, and \cref{subexp hanson wright} for the third point, we get the following fact:
\begin{fact}\label{operator bounds}
The matrix $\bA$ can be either $\bH$ or $\bar{\bH}$
    \begin{itemize}
        \item $\norm{\bA^{-1/2}\Phi_0} \lesssim M_1\lambda^{-1/4\alpha}.$
        \item $\norm{\frac{1}{\sqrt{n}}\bA^{-1/2}\bX^{\top}} \lesssim 1$.
        \item $\norm{\frac{1}{\sqrt{n}}\bA^{-1/2}\bX^{\top} \bres} \lesssim M_1\resv\sqrt{\log(8/\delta)}\lambda^{-1/4\alpha} \log n$. \footnote{Thanks to the symmetric form, we can bound $\bH$ and $\bar{\bH}$ immediately, which makes the use of the subexponential tail bound valid, no need to wait for taking the trace}
    \end{itemize}
\end{fact}

\begin{proof}
For the first item, write:
        \begin{align*}
            \norm{\bA^{-1/2}\Phi_0}^2 
            & = \ip{\Phi_0}{\bA^{-1}\Phi_0} \\
            & \le \ip{\Phi_0}{(k \hSig + \lambda \bI)^{-1}\Phi_0}
            \quad \text{under }\mathcal{B} \\
            & \le 2\ip{\Phi_0}{(k \bSigma + \lambda \bI)^{-1}\Phi_0}
            \quad \text{via \cref{covarariance concentration}} \\
            & \lesssim M_1^2\lambda^{-1/2\alpha} \quad\text{via \cref{effective dimension}}.
        \end{align*}

For the second item, we have:
        \begin{align*}
            \norm{\frac{1}{\sqrt{n}}\bA^{-1/2}\bX^{\top}}^2
            & = \frac{1}{n}\norm{\bX\bA^{-1}\bX^{\top}} \\
            & \le \frac{1}{n}\norm{\bX(k \hSig + \lambda \bI)^{-1}\bX^{\top}}
            \quad \text{under }\mathcal{B} \\
            & = \norm{(k \hSig + \lambda \bI)^{-1}\hSig} \\
            & \lesssim 1.
        \end{align*}

And for the third item:
\begin{align*}
    \norm{\frac{1}{\sqrt{n}}\bA^{-1/2}\bX^{\top} \bres}
    & \le \norm{\frac{1}{\sqrt{n}}(k \hSig + \lambda \bI)^{-1/2}\bX^{\top} \bres}\quad \text{under }\mathcal{B} \\
    & \le M_1\resv\sqrt{\log(8/\delta)}\lambda^{-1/4\alpha} \log n
    \quad\text{directly from event }\mathcal{A}_2.
\end{align*}

\end{proof}

We also have the simple bound where $\bA$ can be either $\bH$ or $\bar{\bH}$:
\begin{equation*}\bA^{-1/2} \preceq \lambda^{-1/2} \bI.\end{equation*}

\paragraph{Computation of the bias term of \eqref{bias variance perturbation}}

\begin{align*}
    \lambda \abs{\ip{\Phi_0}{\bH^{-1}\btheta^*}}
    &= \lambda \abs{\ip{\bH^{-1/2}\Phi_0}{\bH^{-1/2}\btheta^*}} \\
    & \le \lambda \norm{\bH^{-1/2}\Phi_0}\cdot \norm{\bH^{-1/2}\btheta^*}
    \quad\text{(by Cauchy-Schwarz)} \\
    & \lesssim \lambda \cdot M_1\lambda^{-1/4\alpha} \cdot \lambda^{-1/2} \norm{\btheta^*}
    \quad\text{(by \cref{operator bounds})} \\
    & = M_1\lambda^{(2-1/\alpha)/4} \norm{\btheta^*} \\
    & \le M_1\lambda^{1/4}\norm{\btheta^*} \quad\text{since $\alpha \ge 1$ and $\lambda$ is small}.
\end{align*}

\paragraph{Computation of the variance term of \eqref{bias variance perturbation}}
Conditionally on $\bX$, $\bres$ is sub-exponential vector with $\norm{\bres}_{\psi_1} \le \resv$. Hence, automatically for any vector $\ba \in \RR^n$, $\ip{\ba}{\bres}$ is a sub-exponential variable with $\norm{\ip{\ba}{\bres}}_{\psi_1} \le \norm{a}_2 \resv$, such that the following tail bound holds $\forall \delta' \in (0,1)$ with probability at least $1-\delta'$: \begin{equation*}\abs{\ip{\ba}{\bres}} \le \resv\norm{a}_2 \log(2/\delta').\end{equation*}
Therefore, the variance term can be bound conditionally on the event $\mathcal{B}$ with probability at least $1-\delta'$, as:
\begin{align*}
    \abs{\ip{\Phi_0}{\frac{1}{n}\bar{\bH}^{-1}\bX^{\top} \bres}}
    & \le \resv\log(2/\delta')\frac{1}{n}\norm{\bX\bar{\bH}^{-1}\Phi_0}_2 \\
    & \le \resv\log(2/\delta')\frac{1}{\sqrt{n}} \norm{\bar{\bH}^{-1/2}\Phi_0}\cdot \norm{\frac{1}{\sqrt{n}}\bar{\bH}^{-1/2}\bX^{\top}} \\
    & \lesssim \resv\log(2/\delta')\frac{M_1}{\sqrt{n}} \lambda^{-1/4\alpha}\cdot 1
    \quad\text{(by \cref{operator bounds})} \\
    &= \resv\log(2/\delta')\frac{M_1}{\sqrt{n}} \lambda^{-1/4\alpha}.
\end{align*}

\paragraph{Computation of the perturbation term of \eqref{bias variance perturbation}}

First note that $(\bI + \bDelta)^{-1} = \bI - \bDelta (\bI + \bDelta)^{-1}$ and 
\begin{align*}
    \bH^{-1} - \bar{\bH}^{-1} 
    &= \bar{\bH}^{-1/2}\cro{\paren{\bar{\bH}^{-1/2}\bH\bar{\bH}^{-1/2}}^{-1}-\bI} \bar{\bH}^{-1/2} \\
    &= \bar{\bH}^{-1/2}\cro{\paren{\bI + \bDelta}^{-1}-\bI} \bar{\bH}^{-1/2} \\
    &= - \bar{\bH}^{-1/2}\cro{\bDelta\paren{\bI + \bDelta}^{-1}} \bar{\bH}^{-1/2},
\end{align*}
and since $\norm{\bDelta} \le 1/2$, 
\begin{equation*}\norm{\paren{\bI + \bDelta}^{-1}} \le \sum_{n=0}^{\infty}\norm{\bDelta}^n = \frac{1}{1-\norm{\bDelta}} \le 2\end{equation*} hence the perturbation term becomes:

\begin{align*}
    \abs{\ip{\Phi_0}{\frac{1}{n}(\bH^{-1} -\bar{\bH}^{-1})\bX^{\top} \bres}}
    &= \abs{\ip{\Phi_0}{\frac{1}{n}\cro{\bar{\bH}^{-1/2}\cro{\bDelta\paren{\bI + \bDelta}^{-1}} \bar{\bH}^{-1/2}}\bX^{\top} \bres}} \\
    & \le \norm{\bar{\bH}^{-1/2}\Phi_0} \cdot 2\norm{\bDelta}\cdot \frac{1}{\sqrt{n}}\norm{\frac{1}{\sqrt{n}}\bar{\bH}^{-1/2}\bX^{\top} \bres} \\
    & \lesssim M_1\lambda^{-1/4\alpha} \cdot \norm{\bDelta} \cdot \frac{1}{\sqrt{n}} M_1\resv\sqrt{\log(8/\delta)}\lambda^{-1/4\alpha} \log n\\
    &= \frac{\log n}{\sqrt{n \lambda^{1/\alpha}}}\norm{\bDelta}M_1^2\resv\sqrt{\log(8/\delta)}.
\end{align*}

So far, we have proved that there exists a constant $c>0$ such that for any $\bx_0 \in \supp(P)$:
\begin{equation*}
    \Pro{\abs{\ip{\Phi(\bx_0)}{\hthetal - \btheta^*}} \le c \paren{\frac{\log n}{\sqrt{n \lambda^{1/\alpha}}}\norm{\bDelta}M_1^2\resv\sqrt{\log(\frac{8}{\delta})} + \resv\log(\frac{2}{\delta'})\frac{M_1}{\sqrt{n}} \lambda^{-1/4\alpha} + M_1\lambda^{1/4}\norm{\btheta^*}}\Bigm\vert \mathcal{B}} \ge 1-\delta'.
\end{equation*}

In particular, by union bound on all the $\bx_i$ and $\bx_0$, we have simultaneously with probability larger than $1-(n+1)\delta'$ conditionally on $\mathcal{B}$:

\begin{equation}\label{inf bound}
    \norm{\bX(\hthetal - \btheta^*)}_{\infty} \le c \paren{\frac{\log n}{\sqrt{n \lambda^{1/\alpha}}}\norm{\bDelta}M_1^2\resv\sqrt{\log(\frac{8}{\delta})} + \resv\log(2/\delta')\frac{M_1}{\sqrt{n}} \lambda^{-1/4\alpha} + M_1\lambda^{1/4}\norm{\btheta^*}},
\end{equation}

and 

\begin{equation}\label{proba x0 B}
    \abs{\ip{\Phi(\bx_0)}{\hthetal - \btheta^*}} \le c \paren{\frac{\log n}{\sqrt{n \lambda^{1/\alpha}}}\norm{\bDelta}M_1^2\resv\sqrt{\log(\frac{8}{\delta})} + \resv\log(\frac{2}{\delta'})\frac{M_1}{\sqrt{n}} \lambda^{-1/4\alpha} + M_1\lambda^{1/4}\norm{\btheta^*}}.
\end{equation}

\bigskip

But by construction from \cref{relative error}, we also know that $\norm{\bDelta} \le C_1 \norm{\bX(\hthetal - \btheta^*)}_{\infty}$ where $C_1>0$ is a constant depending on the Lipschitz and strong convexity constant.

We can plug this expression in \cref{inf bound} and regroup the terms to get:
\begin{equation*}
    \norm{\bX(\hthetal - \btheta^*)}_{\infty} \paren{1-\frac{cC_1\log n}{\sqrt{n \lambda^{1/\alpha}}}M_1^2\resv\sqrt{\log(\frac{8}{\delta})}} \le cM_1 \paren{\resv\log(2/\delta')\frac{1}{\sqrt{n}} \lambda^{-1/4\alpha} + \lambda^{1/4}\norm{\btheta^*}}.
\end{equation*}

Thus, if $\frac{cC_1\log n}{\sqrt{n \lambda^{1/\alpha}}}M_1\resv\sqrt{\log(\frac{8}{\delta})} \le \frac{1}{2}$, 
\begin{equation}\label{final infinity bound}
    \norm{\bX(\hthetal - \btheta^*)}_{\infty} \le 2cM_1 \paren{\resv\log(2/\delta')\frac{1}{\sqrt{n}} \lambda^{-1/4\alpha} + \lambda^{1/4}\norm{\btheta^*}}.
\end{equation}

With $\lambda \gtrsim \paren{\frac{\log n}{\sqrt{n}}}^{2\alpha}\frac{1}{g(n)^{\alpha/2}}$, \begin{equation*}\frac{cC_1\log n}{\sqrt{n \lambda^{1/\alpha}}} \lesssim cC_1 g(n)^{1/4} \to 0\end{equation*} so for $n$ large enough, we do have $\frac{cC_1 M_1\resv\sqrt{\log(\frac{8}{\delta})}}{\sqrt{n \lambda^{1/\alpha}}} \le \frac{1}{2}$.

\bigskip

Finally, we can plug back \cref{final infinity bound} into \cref{proba x0 B} with probability larger than $1-(n+1)\delta'$ conditionally on $\mathcal{B}$:

\begin{multline*}
    \abs{\ip{\Phi(\bx_0)}{\hthetal - \btheta^*}} \le c \left(\frac{\log n}{\sqrt{n \lambda^{1/\alpha}}}2cC_1 M_1^3\resv\sqrt{\log(\frac{8}{\delta})}\cro{\resv\log(2/\delta')\frac{1}{\sqrt{n}} \lambda^{-1/4\alpha} + \lambda^{1/4}\norm{\btheta^*}}\right. \\ \left. + \resv\log(2/\delta')\frac{M_1}{\sqrt{n}} \lambda^{-1/4\alpha} + M_1\lambda^{1/4}\norm{\btheta^*}\right),
\end{multline*}
i.e. there is a constant $c'>0$ such that, for $n$ large enough and $\delta' = \frac{\delta}{2(n+1)}\le \frac{\delta}{6} \in (0,1)$:

\begin{equation}
    \Pro{\abs{\ip{\Phi(\bx_0)}{\hthetal - \btheta^*}} \le c'M_1 \left(\resv\log(n/\delta)\frac{1}{\sqrt{n}} \lambda^{-1/4\alpha} + \lambda^{1/4}\norm{\btheta^*}\right) \Bigm\vert \mathcal{B}} \ge 1-\frac{\delta}{2},
\end{equation}
and thus
\begin{equation}\label{clean bound}
    \Pro{\abs{\ip{\Phi(\bx_0)}{\hthetal - \btheta^*}} \le c' M_1\left(\resv\log(n/\delta)\frac{1}{\sqrt{n}} \lambda^{-1/4\alpha} + \lambda^{1/4}\norm{\btheta^*}\right)} \ge (1-\frac{\delta}{2})^2\ge 1-\delta.
\end{equation}

If we set \begin{equation*}\lambda = c_2 \paren{\frac{M_1\resv\sqrt{\log(8/\delta)}\log n}{\sqrt{n}}}^{2\alpha}\frac{1}{g(n)^{\alpha/2}}= c_2 \paren{M_1\resv\sqrt{\log(8/\delta)}}^{2\alpha}\paren{\log(n)}^{\alpha(2 + \beta/2)}n^{-\alpha},\end{equation*}
the probabilistic bound in \eqref{clean bound} becomes:
\begin{equation}
    M_1\resv\log(n/\delta)\frac{n^{-1/4}}{\log(n)^{1/2 + \beta/8}}\paren{M_1\resv\sqrt{\log(8/\delta)}}^{-1/2}  + \paren{M_1\resv\sqrt{\log(8/\delta)}}^{\alpha/2}\frac{\log(n)^{\alpha(1/2+\beta/8)}}{n^{\alpha/4}}M_1\norm{\btheta^*}.
\end{equation}

With $\alpha = 1$ and $\beta \in (0,8)$, it yields, \wpd \begin{equation*}\abs{\ip{\Phi(\bx_0)}{\hthetal - \btheta^*}} \lesssim \sqrt{M_1^3\sigma}n^{-1/4} (\log (\frac{n}{\delta}))^{3/2} \lesssim \sqrt{M_1^3\sigma} n^{-1/4} (\log (\frac{n}{\delta}))^3.\end{equation*}

(We only care about the $n$ exponent, so we chose $3$ as an upper bound to the true $3/2$ for the logarithm exponent to lighten the expression.)
    
\end{proof}

\section{Proof of \texorpdfstring{\cref{generic oracle inequality}}{Section \getrefnumber{generic oracle inequality}}'s in-sample oracle inequality}\label[appendix]{poracle}

\oraclenonapprox*

\begin{remark}
    Under that event, we also have this intermediate inequality that holds \wpd:
\begin{equation*}\sqrt{\cRjhat} \le \sqrt{\cRjstar} + \frac{C_2\sqrt{K}}{\sqrt{nk}}\paren{\sqrt{K}\variv \log(\frac{m}{\delta}) + \sqrt{K}b + \frac{\sqrt{n}\eta^2}{\sqrt{k}}},\end{equation*}
where $C_2>0$.
\end{remark}

\begin{proof}
    We start by introducing some notations for the rest of the proof. 
    For any function $f$, let $\bm{f} =(f(\bx_1),\cdots, f(\bx_n))$ be the evaluation vector over the samples. We also denote for all $j$, $\cR_j:=\cR_{in}(f_j)$ and $\widehat{\cR}_j:=\widehat{\cR}_{in}(f_j)$.

    \medskip
    We use the definition of the selected index $j^* = \argmin_{j\in[m]} \cR_j$. We first need to establish a decomposition result between $\cR$ and $\hat \cR$:
    \begin{fact}\label{true risk decomposition}
    Let $f$ be any candidate model. Then
    \begin{equation*}\cR_{in}(f) =  \widehat{\cR}_{in}(f) - \widehat{\cR}_{in}(f^*) + \frac{1}{n}\sum_{i=1}^n(\logpar'(\tilde f(\bx_i))-\logpar'(f^*(\bx_i)))(f(\bx_i)-f^*(\bx_i)).\end{equation*}
    \end{fact}

    \begin{proof}
        We simply apply the law of cosines of Bregman divergence since 
        \begin{equation*}\cR_{in}(f) =\frac{1}{n}\sum_{i=1}^n D_{\logpar}(f(\bx_i), f^*(\bx_i))\,,\quad \widehat{\cR}_{in}(f) =\frac{1}{n}\sum_{i=1}^n D_{\logpar}(f(\bx_i), \tilde f(\bx_i)).\end{equation*}
    \end{proof}

    Applying \cref{true risk decomposition} to $\cR_{\hat j}$, it follows that:
    \begin{align*}
        \cR_{\hat{j}}  & = \widehat \cR_{\hat{j}} - \widehat \cR(f^*) + \frac{1}{n}\sum_{i=1}^n(\logpar'(\tilde f(\bx_i))-\logpar'(f^*(\bx_i)))(f_{\hat{j}}(\bx_i)-f^*(\bx_i)) \\
        & \le \widehat \cR_j - \widehat \cR(f^*) + \frac{1}{n}\sum_{i=1}^n(\logpar'(\tilde f(\bx_i))-\logpar'(f^*(\bx_i)))(f_{\hat{j}}(\bx_i)-f^*(\bx_i))\,, \quad \forall j \in [m] \\
        & = \cR_j + \frac{1}{n}\sum_{i=1}^n(\logpar'(\tilde f(\bx_i))-\logpar'(f^*(\bx_i)))(f_{\hat{j}}(\bx_i)-f_j(\bx_i)),
    \end{align*}
    where the first inequality stems from the minimality of $\widehat \cR_{\hat{j}}$, and the last equality is obtained by re-applying \cref{true risk decomposition} to $\cR_{j}$.

    By the Mean Value Theorem, we can write 
    \begin{equation*}\logpar'(\tilde f(\bx_i))-\logpar'(f^*(\bx_i)) = \logpar''(q_i)(\tilde f(\bx_i)-f^*(\bx_i))\end{equation*}
    for some $q_i$ between $f^*(\bx_i)$ and $\tilde f(\bx_i)$. Denote $\widetilde \bD$ the positive diagonal matrix of entries $\logpar''(q_i)$ that induces an inner product and a (weighted) norm so that the previous inequality becomes:

    \begin{equation}\label{eq: model selection inequality}
        \cR_{\hat{j}} \le \cR_j + \frac{1}{n}\ip{\bm{f}_{\hat{j}}-\bm{f}_j}{\bias + \vari}_{\widetilde{\bD}}.
    \end{equation}
    
    We cannot directly use a sub-exponential tail bound for the inner product because $\vari$ and $\widetilde \bD$ are correlated. Thus we introduce the intermediate weight matrix $\overline \bD$ with components $\overline D_{ii}=\logpar''(f^*(\bx_i))$ and use the fact that the perturbation introduced is small thanks to the consistency rate assumed. 
    
    In the remainder of this proof, we will use repeatedly the following inequalities linking the $\widetilde \bD$-weighted norm, the $\overline \bD$-weighted norm, and the standard Euclidean norm:

    \begin{fact}\label{fact: norm inequalities}
        For all $z \in \RR^n$:
        \begin{enumerate}
            \item $\norm{z}_2 \le \frac{1}{\sqrt{k}}\norm{z}_{\widetilde\bD}\,, \quad \norm{z}_{\widetilde\bD} \le \sqrt{K}\norm{z}_{2}$.
            \item $\norm{z}_{\overline \bD} \le \sqrt{2}\norm{z}_{\widetilde\bD}$.
        \end{enumerate}
    \end{fact}

    \begin{proof}
        Since, we assume that the $\norm{\tilde{\bm{f}} - \bm{f}^*}_{\infty} \le \eta \le C$, and $\norm{\bm{f}^*}_\infty \le C$, we can leverage our assumptions on the local properties of the log-partition function $\logpar$ in $[-2C, 2C]$. 
        Part 1 directly follows from the definition of the weighted norm and boundedness of $\logpar''$. For part 2, we obtain a stronger bound by leveraging the Lipschitzness assumption on $\logpar''$ and the uniform consistency of the imputation model:
        \begin{equation*}\abs{\norm{z}_{\overline \bD}^2-\norm{z}_{\widetilde \bD}^2} \le \sum_i z_i^2\abs{\logpar''(q_i)-\logpar''(f^*(\bx_i))} \le \frac{L\eta}{k} \sum_i kz_i^2,\end{equation*}
        hence, via the strong convexity assumption,
        \begin{equation*}\abs{\norm{z}_{\overline \bD}^2-\norm{z}_{\widetilde \bD}^2} \le \frac{L\eta}{k} \sum_i \logpar''(q_i)z_i^2 \le \frac{L\eta}{k} \norm{z}_{\widetilde \bD}^2.\end{equation*}
        Therefore, since $\frac{L\eta}{k}<1$ according to the assumptions, \begin{equation*} \norm{z}_{\overline \bD}^2\le 2\norm{z}_{\widetilde \bD}^2.\end{equation*}
    \end{proof}

    To make progress on inequality \eqref{eq: model selection inequality}, we decompose the inner product into three terms that we can tackle independently:
    \begin{equation*}\cR_{\hat{j}} \leq \cR_j + \frac{1}{n} \cro{\ip{\fjhat - \bm{f}_j}{\bias}_{\widetilde \bD} + \ip{\fjhat - \bm{f}_j}{\vari}_{\overline \bD} + \ip{\fjhat - \bm{f}_j}{\vari}_{\widetilde\bD-\overline \bD}}.\end{equation*}

\begin{itemize}
    \item Cauchy-Schwarz for the first term yields:
    \begin{equation*}\ip{\fjhat - \bm{f}_j}{\bias}_{\widetilde\bD} \le \norm{\fjhat - \bm{f}_j}_{\widetilde\bD}\norm{\bias}_{\widetilde\bD} \le \sqrt{K}\norm{\fjhat - \bm{f}_j}_{\widetilde\bD}\norm{\bias}_2 \le \sqrt{K}\,b\,\norm{\fjhat - \bm{f}_j}_{\widetilde\bD},\end{equation*}
    where the last inequality holds on the event of item 3 in \cref{assumptions oracle}.

    \item Cauchy-Schwarz for the third perturbation term gives:
    \begin{equation*}\ip{\fjhat - \bm{f}_j}{\vari}_{\widetilde\bD-\bar \bD} \le \norm{\widetilde\bD-\bar \bD}\norm{\vari}_2\norm{\fjhat - \bm{f}_j}_2 
    \le \frac{L}{\sqrt{k}} \max_{i\in[n]}\abs{\tilde f(\bx_i)-f^*(\bx_i)}\cdot\norm{\vari}_2 \cdot \norm{\fjhat - \bm{f}_j}_{\widetilde\bD},\end{equation*}
    where we used $\norm{\widetilde\bD-\bar \bD} = \max_i\abs{\logpar''(q_i)-\logpar''(f^*(\bx_i))}\le L \max_i\abs{\tilde f(\bx_i)-f^*(\bx_i)}$.
    We will see that the consistency rate of our estimator is good enough to make this term of the same order as the two others.

    We can improve the $\max_{i\in[n]}\abs{\tilde f(\bx_i)-f^*(\bx_i)}\cdot\norm{\vari}_2$ factor by noting that:
    
    By assumption, $\forall i \in [n], \abs{B_i + \vari_i} \le \eta$, hence $\abs{\vari_i} \le \eta + \abs{B_i}$ and $\norm{\vari}_2^2 \le 2(n\eta^2 + \norm{\bias}_2^2)\le 2(n\eta^2 + b^2)$, pathwise on the event of item 3 in \cref{assumptions oracle}.

    Therefore,
    \begin{equation}\label{improve equation oracle}
        \max_{i\in[n]}\abs{\tilde f(\bx_i)-f^*(\bx_i)}^2\cdot\norm{\vari}_2^2 \le 2(n\eta^4 + \eta^2b^2).
    \end{equation}

    \item For the second term, first bound its sub-exponential norm:
    \begin{align*}
        \norm{\ip{\fjhat - \bm{f}_j}{\vari}_{\overline \bD}}_{\psi_1} 
        & = \norm{\overline \bD (\fjhat - \bm{f}_j)}_2 \norm{\ip{\frac{\overline \bD(\fjhat - \bm{f}_j)}{\norm{\overline \bD (\fjhat - \bm{f}_j)}_2}}{\vari}_{\overline \bD}}_{\psi_1} \\
        & \le \sqrt{K}\norm{\fjhat - \bm{f}_j}_{\overline \bD}\norm{\ip{\frac{\overline \bD(\fjhat - \bm{f}_j)}{\norm{\overline \bD (\fjhat - \bm{f}_j)}_2}}{\vari}_{\overline \bD}}_{\psi_1} \\
        & \le \variv \sqrt{K}\norm{\fjhat - \bm{f}_j}_{\overline \bD} \\
        & \le \variv \sqrt{2K}\norm{\fjhat - \bm{f}_j}_{\widetilde\bD}.
    \end{align*}
    Hence, the sub-exponential property and union bound gives that the event 
    \begin{equation*}\mathcal{A} = \left \{\forall j \in [m], \ip{\fjhat - \bm{f}_j}{\vari}_{\overline \bD} \le \sqrt{2K} \variv \log(4m/\delta)\norm{\fjhat - \bm{f}_j}_{\widetilde\bD}\right \}\end{equation*}
    holds with probability at least $1-\delta/2$.
    We place ourselves under $\mathcal{A}$. 
    
\end{itemize}

    So we obtain with probability at least $1-\delta$:

    \begin{equation}
        \cR_{\hat{j}} \le \cR_j + \frac{2\xi}{n}\norm{\bm{f}_{\hat{j}}-\bm{f}_j}_{\tilde{\bD}},
    \end{equation}
    where $\xi = (\sqrt{2K}\variv \log(4m/\delta) + \sqrt{K}b + \frac{L\max_{i\in[n]}\abs{\tilde f(\bx_i)-f^*(\bx_i)}\cdot\norm{\vari}_2}{\sqrt{k}})/2$.

    Then, using part 1 of \cref{fact: norm inequalities}, the triangle inequality and \cref{in-sample excess risk bound}, we get:
    \begin{align*}
        \frac{1}{\sqrt{n}}\norm{\bm{f}_{\hat j}-\bm{f}_j}_{\tilde \bD}
        & \le \sqrt{\frac{K}{n}}\norm{\bm{f}_{\hat j}-\bm{f}_j}_2 
         \le \sqrt{\frac{K}{n}}\paren{\norm{\bm{f}_{\hat j}-\bm{f}^*}_2 + \norm{\bm{f}_j-\bm{f}^*}_2} 
         \le \sqrt{\frac{2K}{k}} \paren{\sqrt{\cR_j} + \sqrt{\cR_{\hat j}}}.
    \end{align*}
    
    Hence, with $\tilde \xi = \sqrt{\frac{2K}{k}} \xi / \sqrt{n}$, we get:
    \begin{equation*}\cR_{\hat{j}} \le \cR_j + 2\tilde \xi\paren{\sqrt{\cR_j} + \sqrt{\cR_{\hat j}}},\end{equation*}
    and we can now complete the square:

\begin{itemize}
    \item Rearranging:
\begin{equation*}\cR_{\hat{j}} - 2\tilde\xi \sqrt{\cR_{\hat{j}}} \le \cR_j + 2\tilde\xi \sqrt{\cR_j}
\Leftrightarrow (\sqrt{\cR_{\hat{j}}}-\tilde\xi)^2 - \tilde\xi^2  \le (\sqrt{\cR_j} + \tilde\xi)^2 - \tilde\xi^2.\end{equation*}
Thus, \begin{equation*}\sqrt{\cR_{\hat{j}}} \le \sqrt{\cR_j} + 2\tilde\xi.\end{equation*}

Since this is valid for all $j$, an intermediate result is thus:
\begin{equation*}\sqrt{\cR_{\hat{j}}} \le \sqrt{\cR_{j^*}} + 2\tilde\xi.\end{equation*}

    \item Squaring on both sides, we obtain: 
\begin{equation*}\cR_{\hat{j}} \le \cR_{j^*} + 2\sqrt{4\tilde\xi^2\cR_{j^*}} + 4\tilde\xi^2.\end{equation*}
Using the simple inequality $2ab \le a^2 + b^2$ and the "Rob Peter to pay Paul" trick with any $\epsilon >0$, i.e. $2ab=2\frac{a}{\epsilon}\epsilon b \le \frac{a^2}{\epsilon^2} + \epsilon^2b^2$, we get:
    \begin{equation*}\cR_{\hat{j}} \le (1+\epsilon)\cR_{j^*} + 4\tilde\xi^2(1 + \frac{1}{\epsilon}).\end{equation*}

Finally, to get the final formula of the theorem, we leverage \cref{improve equation oracle} to bound $\tilde \xi^2$ as:
\begin{equation*}\tilde \xi^2 \lesssim \frac{K}{kn}(K\variv^2 \log^2(m/\delta) + b^2(K + 2L^2\eta^2/k) + \frac{L^2n\eta^4}{k}).\end{equation*}

Using $\eta \le \sqrt{kK}/L$ and taking the infimum on $\epsilon$ yields the final result.

\end{itemize}

\end{proof}

\begin{remark}[The argument is pathwise in $\bias$]\label{rem: pathwise}
The bias $\bias$ enters the proof above only through the Cauchy--Schwarz bound on $\ip{\fjhat-\bm{f}_j}{\bias}_{\widetilde\bD}$ and the bound $\norm{\vari}_2^2\le 2(n\eta^2+\norm{\bias}_2^2)$, which hold for every realization. Neither the law of $\bias$ nor its dependence on $\vari$ or on $\hat j$ plays any role; only $\norm{\bias}_2\le b$ is used. Concentration is only needed for $\vari$, through the event $\mathcal{A}$, whose union bound over $j\in[m]$ covers the data-dependent $\hat j$.
\end{remark}

We deduce a bound on the expectation:

\begin{corollary}\label{exoracle}
Under the same assumptions as \cref{oracle general non approx}, and assuming in addition that the bounds of item 3 in \cref{assumptions oracle} hold almost surely, we deduce:
    \begin{equation*}\displaystyle \Ex{\cRjhat} \le \inf_{\epsilon>0} \left \{(1+\epsilon)\cRjstar + \frac{C_1K}{nk}\paren{1+\frac{1}{\epsilon}}\paren{Kb^2 + \frac{L^2n\eta^4}{k}+ K\variv^2(1+(\log(m)+1)^2)}\right \}.\end{equation*}
\end{corollary}

\begin{proof}
Denoting $W=[\cRjhat - (1+\epsilon)\cRjstar - \frac{C_1K}{nk}(1+\frac{1}{\epsilon})(Kb^2+L^2n\eta^4/k)]/\frac{C_1K}{nk}(1+\frac{1}{\epsilon})K\variv^2$, we have from the theorem:
\begin{equation*}\Pro{W \ge \log^2(m/\delta)} \le \delta, \forall \delta \in (0,1).\end{equation*}
Let $t=\log^2(m/\delta)$ such that:
\begin{equation*}\Pro{W \ge t} \le me^{-\sqrt{t}}, \forall t \ge \log^2(m).\end{equation*}
Integrating to get the expectation yields:
\begin{align*}\Ex{W}=\int_0^{\infty}\Pro{W>t}dt &\le \int_0^{\log^2(m)}1dt + \int_{\log^2(m)}^{\infty}me^{-\sqrt{t}}dt \\ &=\log^2(m)+2(\log(m)+1)=(1+\log(m))^2 + 1,\end{align*}
since $\int_{u^2}^{\infty}e^{-\sqrt{t}}dt=2(u+1)e^{-u}$.
Rearranging the terms and taking the infimum gives the desired result.
\end{proof}

\section{Useful lemmas}\label[appendix]{lemmas}

\subsection{Lemmas for the GLM estimator}

\begin{lemma}\label{effective dimension}
Let the kernel regularity \cref{kernel assumptions} hold, with spectral eigendecay of the form $\mu_j \le A j^{-2\alpha}, \alpha \ge 1$, and let $\lambda \in (0,1]$. Then the effective dimension $\displaystyle r(\lambda):= \sum_{j=1}^{\infty}\frac{\mu_j}{\mu_j + \lambda}$ satisfies for a universal constant $C >0$:
\begin{equation*}r(\lambda) \le C \lambda^{-1/2\alpha}.\end{equation*}

In particular, if the trace class operator $\bSigma$ has eigenvalues $\mu_j$, then $\Tr((\bSigma + \lambda \bI)^{-1}\bSigma) = r(\lambda) \lesssim \lambda^{-1/2\alpha}$.

Likewise, for a vector $\Phi(\bx)$ of the RKHS with basis decomposition $\Phi(\bx)=(\sqrt{\mu_1}\phi_1(\bx), \sqrt{\mu_2}\phi_2(\bx)...)$, where the basis functions are uniformly bounded by $M_1$, the quadratic form $\ip{\Phi(\bx)}{(\bSigma + \lambda \bI)^{-1}\Phi(\bx)} \le M_1^2 r(\lambda) \lesssim M_1^2\lambda^{-1/2\alpha}$.
\end{lemma}

\begin{proof}
    Split the sum before and after $j=\lambda^{-1/2\alpha}$ and use the asymptotic tail of the series. Without loss of generality, we assume $A=1$ (for $A\neq 1$, use $\lambda = \lambda/A$ in the formula). 
    \begin{align*}
        \sum_{j=1}^{\infty}\frac{\mu_j}{\mu_j + \lambda} 
        & = \sum_{j\le \lambda^{-1/2\alpha}}\frac{1}{1 + \lambda j^{2\alpha}} + \sum_{j> \lambda^{-1/2\alpha}}\frac{1}{1 + \lambda j^{2\alpha}} \\
        & \le \sum_{j\le \lambda^{-1/2\alpha}}1 + \sum_{j> \lambda^{-1/2\alpha}}\frac{1}{\lambda j^{2\alpha}} \\
        & \le \lambda^{-1/2\alpha} + \frac{1}{\lambda (2\alpha -1)\lambda^{-(2\alpha -1)/2\alpha}} \\
        & \le \frac{2\alpha}{2\alpha-1}\lambda^{-1/2\alpha}.
    \end{align*}
    We find that $C = \frac{2\alpha A^{1/2\alpha}}{2\alpha -1}$ in general.
\end{proof}

\begin{remark}\label{remark clambda}
    By the eigendecay for the source distribution, applying \cref{effective dimension} with $\alpha=1$ directly gives for $\lambda \in (0,1)$
    \begin{equation*}\Tr(\bC_\lambda) \lesssim \lambda^{-1/2}.\end{equation*}
\end{remark}

\subsection{Technical lemmas}

\begin{lemma}\label{subexp hanson wright}
Let $\bx$ be a centered sub-exponential vector with $\norm{x_i}_{\psi_1} \le 1, \forall i \in [n]$, $\bA \succeq 0$ a fixed PSD matrix in $\RR^{n\times n}$. Then there exists a universal constant $C>0$, such that $\forall \delta \in (0, \frac{2}{e}]$, \wpd, 
\begin{equation*}\bx^{\top} \bA \bx \le C \Tr(\bA)\log^2(n)\log(\frac{2}{\delta})\end{equation*}
when $n \ge \paren{\frac{4}{\delta}}^{1/3}$.

\end{lemma}

\begin{proof}
    We will truncate the sub-exponential variables to get sub-Gaussian variables and bound the gap between the quadratic forms. Then use lemma D.1 in \citep{wang2023pseudo}. 
    \begin{itemize}
        \item First by definition of sub-exponential, \begin{equation*}\forall i, \Pro{\abs{x_i}>C\log n} \le 2 n^{-C}\end{equation*}
    for any constant $C>0$. Hence if we define the truncated version $\bar{\bx}$ of components $\bar{x}_i := x_i \mathds{1}(\abs{x_i}\le C\log n)$, by union bound, $\bx=\bar{\bx}$ with probability larger than $1-2n^{1-C}$. Denote this event $\mathcal{A}$. Also, by definition of $\bar{\bx}$, we have immediately that $\norm{\bar{\bx}} \le C \sqrt{n}\log n$ almost surely.

        \item $\bar{\bx}$ is not centered anymore so let its mean be $\bmu$. Then $\forall i, \abs{\mu_i} = \Ex{x_i \mathds{1}(\abs{x_i}\le C\log n)} = \Ex{x_i \mathds{1}(\abs{x_i}> C\log n)}$ since $\Ex{x_i}=0$. So Cauchy-Schwarz yields:
        \begin{equation*}\abs{\mu_i} \le \sqrt{\Ex{\abs{x_i}^2}\Pro{\abs{x_i}> C\log n}} \le 2\sqrt{2} n^{-C/2}\lesssim n^{-C/2},\end{equation*}
        where the first factor is bounded by a constant by the sub-exponential property, and the second obeys the tail bound mentioned before. Hence:
        \begin{equation*}\norm{\bmu} \lesssim \sqrt{n}n^{-C/2} = n^{(1-C)/2}.\end{equation*}

        \item Next define the centered version $\tilde{\bx}:=\bar{\bx}-\bmu$. It is centered and bounded almost surely: $\norm{\tilde{\bx}}_{\infty} \le 2C\log n$ hence it is sub-Gaussian with $\norm{\tilde{\bx}}_{\psi_2} \lesssim C\log n$. Thus its quadratic form satisfies a tail bound as proposed in lemma D.1 of \citep{wang2023pseudo} for a universal constant $C'>0$, $\forall \delta \in (0, \frac{2}{e}]$, with probability at least $1-\frac{\delta}{2}$: \begin{equation*}\tilde{\bx}^{\top} \bA \tilde{\bx} \le C' \norm{\tilde{\bx}}_{\psi_2}^2 \Tr(\bA)\log(2/\delta) \lesssim \log^2(n) \Tr(\bA) \log(2/\delta).\end{equation*}
        Denote this event $\mathcal{B}$.

        \item Finally, we can bound the difference $\abs{\tilde{\bx}^{\top} \bA \tilde{\bx} - \bar{\bx}^{\top} \bA \bar{\bx}} = \abs{\bmu^{\top} \bA (2\bar{\bx}-\bmu)} \le \norm{\bA}(2\norm{\bmu}\norm{\bx} + \norm{\bmu}^2) \le \norm{\bA}(n^{1-C} + n^{1-C/2}\log n)$.
    \end{itemize}

Putting everything together, under the event $\mathcal{A}\cap \mathcal{B}$ which holds with probability larger than $1-2n^{1-C}-\frac{\delta}{2}$ by union bound, we get:
$\bx^{\top} \bA \bx = \tilde{\bx}^{\top} \bA \tilde{\bx} + (\bx^{\top} \bA \bx - \bar{\bx}^{\top} \bA \bar{\bx}) + (\bar{\bx}^{\top} \bA \bar{\bx} - \tilde{\bx}^{\top} \bA \tilde{\bx}) \lesssim \norm{\bA}(n^{1-C} + n^{1-C/2}\log n) + \log^2n \Tr(\bA)\log(2/\delta)$.

\bigskip
Choosing $C=4$, since $\log(2/\delta) \ge 1$ for $\delta \le 2/e$, and $\norm{\bA} \le \Tr(\bA)$, yields:
\begin{equation*}\bx^{\top} \bA \bx \lesssim \Tr(\bA)\log(\frac{2}{\delta})\paren{\log^2(n)+\frac{1}{n^3} + \frac{\log (n)}{n}}\end{equation*}
with probability larger than $1-\frac{2}{n^3}-\frac{\delta}{2}$. 
For $n \ge \paren{\frac{4}{\delta}}^{1/3}$, the probability is larger than $1-\delta$ as wanted. For $\delta = 0.05$, this threshold is just $n\ge 5$.
Removing the negligible terms yields the desired result.
    
\end{proof}

\begin{lemma}\label{covarariance concentration}
Let \Cref{setup,kernel assumptions} hold, with $M_2$ the bound on the covariates $\Phi(\bx_i)$. Choose $\gamma \in (0,1)$, $\delta \in (0,1/e]$, and $\lambda \ge \frac{4M_2^2\log(14n/\delta)}{\gamma^2n}$. Then the following event holds \wpd:

\begin{equation*}\Pro{(1-\gamma)(\bSigma + \lambda \bI)\preceq \hSig + \lambda \bI \preceq (1+\gamma)(\bSigma + \lambda \bI)} \ge 1-\delta.\end{equation*}

Under that event, the following inequalities are true:
\begin{enumerate}
    \item For any vector $\ba \in \mathbb{H}$, $\ip{\ba}{(\hSig + \lambda \bI)^{-1} \ba} \le (1-\gamma)^{-1}\ip{\ba}{(\bSigma + \lambda \bI)^{-1} \ba}$.
    
    \item $\Tr((\hSig + \lambda \bI)^{-1}\hSig) \le M_1^2(1-\gamma)^{-1}\Tr((\bSigma + \lambda \bI)^{-1}\bSigma)$, where $M_1>0$ is the uniform bound on the eigenfunctions.
\end{enumerate}

\end{lemma}

\begin{proof}
    The first part is just a repetition of corollary D.1 in \citep{wang2023pseudo}. Then, the first inequality is a direct corollary from the main sandwich inequality, by definition of the PSD order.
    Finally, for the second inequality, observe that the sandwich inequality implies: 
    \begin{equation*}\Tr((\hSig + \lambda \bI)^{-1}\hSig) \le \frac{1}{1-\gamma}\Tr((\bSigma + \lambda \bI)^{-1}\hSig).\end{equation*}
    Define $\tilde{\Phi}_i := (\bSigma + \lambda \bI)^{-1/2}\Phi(\bx_i)$ such that \begin{equation*}\Tr((\bSigma + \lambda \bI)^{-1}\hSig) = \Tr((\bSigma + \lambda \bI)^{-1/2}\hSig (\bSigma + \lambda \bI)^{-1/2}) = \frac{1}{n}\sum_{i=1}^n \norm{\tilde{\Phi}_i}^2.\end{equation*}
    Since by uniform boundedness of the eigenfunctions, \begin{equation*}\norm{\tilde{\Phi}_i}^2 = \sum_{j=1}^{\infty}\frac{\mu_j}{\mu_j + \lambda}\abs{\phi_j(\bx_i)}^2 \le M_1^2 \sum_{j=1}^{\infty}\frac{\mu_j}{\mu_j + \lambda} = M_1^2\Tr((\bSigma + \lambda \bI)^{-1}\bSigma),\end{equation*} we get:
    \begin{equation*}\Tr((\bSigma + \lambda \bI)^{-1}\hSig) \le M_1^2 \Tr((\bSigma + \lambda \bI)^{-1}\bSigma),\end{equation*}
    and the result follows.
    
\end{proof}

\subsection{Relating the source covariance operator to the target's via the effective sample size}\label[appendix]{sec: use of neff}

Through the effective dimension defined as $n_{\text{eff}} = \sup \bigg\{ t\le n | t\bSigma_0 \preceq n\bSigma + \frac{\mu^2}{k}\bI \bigg\}$ in equation \cref{effective sample size}, we can relate expressions depending on $\bSigma$ to $\bSigma_0$-only expressions. Take $\lambda \ge \frac{\mu^2}{n}$, and let $r=\frac{n_{\text{eff}}}{n}$. On the one hand we have
\begin{equation*} 2( k \bSigma + \lambda \bI) \succeq  (k \bSigma + \mu^2/n \bI) +  \lambda\bI \succeq rk\bSigma_0 +  \lambda\bI = r(k\bSigma_0 +  r^{-1}\lambda\bI),\end{equation*} hence
\begin{equation}\label{eq: sigma sigma0 neff}
    (k \bSigma + \lambda \bI)^{-1} \preceq 2r^{-1}(k\bSigma_0 +  r^{-1}\lambda\bI)^{-1},
\end{equation}

and 
\begin{equation}\label{eq: trace s_lambda}
    \Tr(\bS_\lambda) \le 2r^{-1}\Tr((k\bSigma_0 + \lambda r^{-1}\bI)^{-1}\bSigma_0).
\end{equation}

On the other hand, we have simply:
\begin{equation*}k \bSigma + \lambda \bI \succeq rk\bSigma_0,\end{equation*} hence:
\begin{equation}\label{eq: operator norm s_lambda}
    \norm{\bS_\lambda} \le r^{-1}k^{-1}.
\end{equation}

\section{Complement to the numerical experiments from \texorpdfstring{\cref{sec: experiments}}{Section \getrefnumber{sec: experiments}}}\label[appendix]{app: numericals}
\subsection{Implementation details}\label[appendix]{app: algo details}

We optimize the ridge-regularized kernel GLM \cref{regularized glm bis} in function space $\textbf{f} = \bK\balpha$ via Fisher scoring \citep{mccullagh1989generalized}, i.e.\ Newton's method with the Hessian replaced by the expected Hessian (Fisher information). At each iteration, with current evaluations $\eta_i = f(\bx_i)$, mean $\mu_i=\mathbb E[y_i\mid \eta_i]$, and variance $V_i=\text{Var}[y_i\mid \eta_i]$, we compute the standard GLM weights $\bW = \diag((\frac{d\mu_i}{d\eta_i})^2/V_i)$ and pseudo-responses $z_i = \eta_i + (y_i-\mu_i)/\frac{d\mu_i}{d\eta_i}$. The resulting iteratively reweighted least squares (IRLS) update requires solving $(\bK \bW + n\lambda \bI)\textbf{f} = \bK \bW \bz$. $\bK\in\mathbb R^{n\times n}$ is the kernel matrix introduced in equation \cref{eq: alpha space}.

To ensure stable convergence of our conjugate gradient \citep{hestenes_stiefel_1952} (CG) solver, we avoid extreme weight scalings by defining $\bS = \bW^{1/2}$ and substituting $\bu = \bS\textbf{f}$ to solve the equivalent symmetric positive definite system:
\[
(\bS \bK \bS + n\lambda \bI) \bu = \bS \bK \bW \bz
\]
We then recover $\textbf{f} = \bS^{-1}\bu$. After convergence, the kernel coefficients needed to evaluate $\hat f=\sum_i\alpha_iK(\bx_i,\cdot)$ at new (e.g., target) points are obtained without solving another linear system: rearranging $(\bK \bW + n\lambda \bI)\textbf{f} = \bK \bW \bz$ as $\textbf{f}=\bK\,\bW(\bz-\textbf{f})/(n\lambda)$ gives $\balpha=\frac{1}{n\lambda}\bW(\bz-\textbf{f})$. 

For large-scale problems, forming the dense $n \times n$ matrix $\bK$ introduces an $\mathcal{O}(n^2)$ memory bottleneck. We circumvent this by implementing the inner CG solver using KeOps \citep{keops} for on-the-fly GPU kernel-vector products. This strategy reduces the memory footprint to $\mathcal{O}(nd)$ without altering the exact IRLS update equations. Full algorithmic details and solvers are available at \repourl{}.

\subsection{Details about the real data experiments}\label[appendix]{app: real data exp}

Let us explain in more detail the experiments of \cref{sec: experiments real}. After sub-sampling, we have labeled in-distribution (ID) data $\mathcal D=\{(x_i,y_i)\}_{i=1}^n$ and labeled out-of-distribution (OOD) data $\mathcal D_0=\{(x_{0,i},y_{0,i})\}_{i=1}^{n_0}$ (OOD labels are used \emph{only} for oracle benchmarking and final evaluation).
We first split the OOD sample into a selection set and a test set,
$\mathcal D_0=\mathcal D_{0,\mathrm{sel}}\cup \mathcal D_{0,\mathrm{test}}$, and report all final metrics on
$\mathcal D_{0,\mathrm{test}}$.

Recall that for a score function $f:\mathcal X\to\mathbb R$, the empirical logistic risk on a finite set $S$ with targets $t\in[0,1]$ is
\[
R_S(f;t)\;:=\;\frac{1}{|S|}\sum_{x\in S}\Bigl(a(f(x)) - t(x)\,f(x)\Bigr),
\qquad a(u)=\log(1+e^u),
\]
where $t(x)\in\{0,1\}$ yields standard log-loss and $t(x)\in[0,1]$ yields the corresponding \emph{soft-label} risk.

\medskip
As mentioned in \cref{sec: experiments real}, due to the small sample size, we use a slightly modified selection procedure compared to \cref{alg:pseudo-krr} and \cref{sec: experiments synth}, to deflate the variance of the final estimator. The chosen cross-validation procedure is designed so as to minimize the bias of the imputation model, while keeping its variance reasonable; in other words, it is intended to minimize its pointwise error. It is detailed in \cref{alg:practical-cv}. Each fitting in steps 2 and 5 is a ridge logistic regression task. 

\begin{algorithm}[ht]
\caption{Repeated $K$-Fold Cross-Validated Pseudo-Labeling (Practical Variant)}
\label{alg:practical-cv}
\begin{algorithmic}[1]
\STATE \textbf{Input}: Labeled ID data $\cD = \{(x_i, y_i)\}_{i=1}^n$, labeled OOD data $\mathcal{D}_0 = \{(x_{0i}, y_{0i})\}_{i=1}^{n_0}$ split into $\mathcal{D}_{0,\mathrm{sel}}$ and $\mathcal{D}_{0,\mathrm{test}}$. OOD labels are \emph{not} used by the pseudo-labeling method; they serve only for the oracle benchmark and final evaluation. Candidate grid $\Lambda = \{\lambda_1, \dots, \lambda_m\}$, imputer penalty $\tilde{\lambda}$, number of folds $K$, number of repeats $R$.

\STATE \textbf{Step 1 (Repeated $K$-fold partition).} For each repeat $r \in [R]$, draw a fresh stratified (i.e. preserving class proportions) $K$-fold partition of the ID indices:
\[
\{1,\dots,n\} = \bigsqcup_{k=1}^K F_k^{(r)}, \qquad F_k^{(r)} \cap F_{k'}^{(r)} = \emptyset.
\]
Define the candidate and imputer index sets for each fold $(r,k)$:
\[
I_{\mathrm{cand}}^{(r,k)} := F_k^{(r)}, \qquad I_{\mathrm{imp}}^{(r,k)} := \{1,\dots,n\} \setminus F_k^{(r)}.
\]

\STATE \textbf{Step 2 (Per-fold training and evaluation).} For each fold $(r,k)$:

\STATE \quad (a) \textbf{Imputer.} Fit $\tilde{f}^{(r,k)}$ on $\{(x_i, y_i)\}_{i \in I_{\mathrm{imp}}^{(r,k)}}$ with penalty $\tilde{\lambda}$. Compute soft pseudo-labels:
\[
p^{(r,k)}(x) := \sigma(\tilde{f}^{(r,k)}(x)), \quad x \in \mathcal{D}_{0,\mathrm{sel}}, \quad \sigma(u) = (1+e^{-u})^{-1}.
\]

\STATE \quad (b) \textbf{Candidates.} For each $\lambda_j \in \Lambda$, fit $f_{\lambda_j}^{(r,k)}$ on $\{(x_i, y_i)\}_{i \in I_{\mathrm{cand}}^{(r,k)}}$.

\STATE \quad (c) \textbf{Risk curves.} For each $\lambda_j$, compute:
\begin{align*}
\widehat{R}_{\mathrm{naive}}^{(r,k)}(\lambda_j) &:= R_{\{x_i : i \in I_{\mathrm{imp}}^{(r,k)}\}}\bigl(f_{\lambda_j}^{(r,k)};\; y\bigr), \\
\widehat{R}_{\mathrm{pseudo}}^{(r,k)}(\lambda_j) &:= R_{\mathcal{D}_{0,\mathrm{sel}}}\bigl(f_{\lambda_j}^{(r,k)};\; p^{(r,k)}\bigr), \\
\widehat{R}_{\mathrm{oracle}}^{(r,k)}(\lambda_j) &:= R_{\mathcal{D}_{0,\mathrm{sel}}}\bigl(f_{\lambda_j}^{(r,k)};\; y_0\bigr).
\end{align*}

\STATE \textbf{Step 3 (Aggregation).} Average over all folds and repeats:
\begin{equation}\label{eq: selection risk curve}
\widehat{R}_{\star}(\lambda_j) := \frac{1}{KR} \sum_{r=1}^R \sum_{k=1}^K \widehat{R}_{\star}^{(r,k)}(\lambda_j), \qquad \star \in \{\mathrm{naive},\, \mathrm{pseudo},\, \mathrm{oracle}\}.
\end{equation}

\STATE \textbf{Step 4 (Selection).} Select regularization parameters independently:
\[
\hat{\lambda}_{\star} \in \arg\min_{\lambda \in \Lambda} \widehat{R}_{\star}(\lambda), \qquad \star \in \{\mathrm{naive},\, \mathrm{pseudo},\, \mathrm{oracle}\}.
\]

\STATE \textbf{Step 5 (Refit).} For each $\star \in \{\mathrm{naive},\, \mathrm{pseudo},\, \mathrm{oracle}\}$, refit a final model $\hat{f}_{\star}$ on \emph{all} labeled ID data $\cD$.

\STATE \textbf{Step 6 (Evaluation).} Evaluate each final model on the held-out OOD test set:
\[
R_{\mathrm{test}}(\hat{f}_{\star}) := R_{\mathcal{D}_{0,\mathrm{test}}}\bigl(\hat{f}_{\star};\; y_0\bigr).
\]

\STATE \textbf{Output}: Test risks $R_{\mathrm{test}}(\hat{f}_{\mathrm{naive}})$, $R_{\mathrm{test}}(\hat{f}_{\mathrm{pseudo}})$, $R_{\mathrm{test}}(\hat{f}_{\mathrm{oracle}})$.
\end{algorithmic}
\end{algorithm}

\paragraph{Refitting on all labeled ID data.} Step 5 of \cref{alg:practical-cv} refits the final model at the selected penalty on all labeled ID data, as is standard after cross-validation \citep{james2021isl,friedman2010glmnet}. This is purely an empirical enhancement: our guarantees cover the model trained on a single split (\cref{alg:pseudo-krr}), and we do not prove that refitting preserves them. We use it only for Raisin, where a single split of the 900 instances leaves too little data for a stable final fit, and apply it identically to all methods, so that it cannot create the reported gaps.

\paragraph{Hyperparameters of the Raisin experiment.}
In the specific case of the Raisin dataset, we chose $\lambda_1=10^{-4}$ to $\lambda_m=10^2$ for the candidates grid, and $\tilde \lambda = 10^{-4}$ for the imputer, i.e., $\tilde\lambda=\min\Lambda$. Then we ran \cref{alg:practical-cv} over 100 random seeds. In \cref{sec: experiments real}, we reported the results for $K=2, R=6$. \cref{tab:risk_results_5_fold} provides the results for the same experiment, but using $K=5, R=2$ instead. In this setting, we also compare with the KRR pseudo-labeling method of \citet{wang2023pseudo}, i.e., the same pipeline with square-loss candidates and imputer, evaluated with the same target log-loss. It improves on naive validation and on the density-ratio methods, but remains behind our GLM version (0.406 versus 0.383): even with oracle selection, its square-loss candidates only reach 0.443, reflecting their poor calibration for classification. The CI of this row is computed from its standard error.

\begin{table}[t]
  \caption{Risk mean estimate across seeds with confidence interval (CI).}
  \label{tab:risk_results_5_fold}
  \centering
  \small
  \begin{tabular*}{\columnwidth}{@{\extracolsep{\fill}} lcc @{}}
    \toprule
    Method & Mean & 95\% CI \\
    \midrule
    Naive           & 0.442 & [0.414, 0.469] \\
    KMM + IW-CV          & 0.449 & [0.431, 0.468] \\
    KLIEP + IW-CV          & 0.434 & [0.412, 0.465] \\
    KRR pseudo-labeling \citep{wang2023pseudo} & 0.406 & [0.388, 0.424] \\
    Pseudo-labeling & 0.383 & [0.372, 0.394] \\
    Oracle          & 0.376 & [0.362, 0.389] \\
    \bottomrule
  \end{tabular*}
\end{table}

\bigskip
Finally for this 5-fold (2 repeats) cross-validation experiment, we plotted the risk curves used for model selection for each of the three methods, described in step 3's equation \eqref{eq: selection risk curve}, averaged over seeds in \cref{fig: section risk curves}. The pseudo-labeling method constructs a risk curve consistent with the oracle, thus achieving good model selection for the target. In contrast, the naive risk curve stays flat.

\begin{figure}[ht]
    \centering    \includegraphics[width=0.7\linewidth]{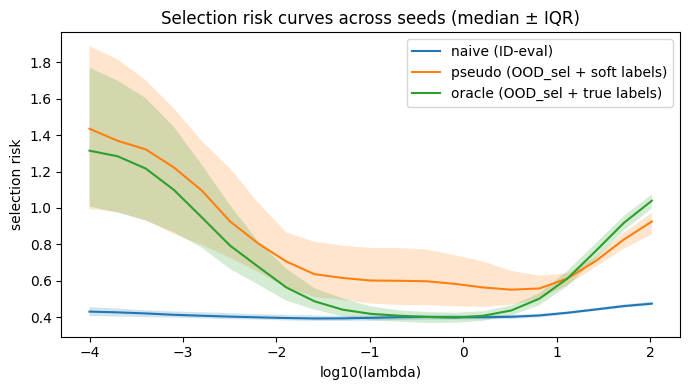}
    \caption{Selection risk curves of the three methods.}
    \label{fig: section risk curves}
\end{figure}

\subsection{WILDS benchmark evaluation}\label[appendix]{app: wilds exp}

\paragraph{Dataset and domain shift.} Camelyon17 \citep{bandi2018detection} consists of histopathology patch images for tumor detection, collected across five hospitals using different scanners and staining protocols, as depicted in \cref{fig:camelyon}. Following the standard WILDS \citep{koh2021wildsbenchmarkinthewilddistribution} splits and the hospital labels of \cref{fig:camelyon}, we use hospitals 1--3 as the labeled source domain and hospital 5 as the unlabeled target domain; hospital 4 (the WILDS out-of-distribution validation split) is not used. The source contains 302,436 training patches and 33,560 in-distribution validation patches, and the target contains 85,054 patches. The domain shift is structural --- different tissue preparation and imaging conditions --- and persists even with strong pretrained representations, making it a challenging and realistic UDA benchmark.

\paragraph{Model and features.} We fine-tune a DINOv2 ViT-B/14 backbone \citep{oquab2023dinov2} pretrained on LVD-142M, with a linear classification head, end-to-end on source data using AdamW (weight decay $0.01$, cosine annealing schedule, batch size $256$). The epoch number $t$ serves as the regularization parameter: early epochs correspond to strong regularization (small deviation from pretrained weights), late epochs to weak regularization (full source adaptation). The first run reported below used a learning rate of $10^{-5}$ and $t \in \{1,\dots,8\}$; the two additional training seeds used a learning rate of $4\times10^{-5}$ with a cosine schedule over 50 epochs, and were stopped after 19 and 22 epochs, respectively, beyond all epochs selected by the pseudo-labeling and oracle rules. Colab A100 GPU with high RAM runtime was used to run these experiments.

\paragraph{Experimental procedure.} Following \cref{alg:pseudo-krr}, candidates are trained on the labeled source data $\cD_1$, while the imputer is a separate DINOv2 instance fine-tuned to convergence on the source validation dataset $\cD_2$, then frozen. Soft pseudo-labels are generated once on a target selection set ($\cD_{\text{sel}}$, a random half of hospital 5) and kept fixed throughout candidate training. At each epoch $t$, three selection signals are computed: naive (source validation NLL on $\cD_2$), pseudo (pseudo-target risk on $\cD_{\text{sel}}$), and oracle (true target NLL on $\cD_{\text{sel}}$, not available in practice). Each method selects the checkpoint minimizing its signal. Final evaluation is performed on the held-out target test set ($\cD_{\text{test}}$, remaining half of hospital 5), which is never used during selection. No refitting is performed.

\paragraph{Why density ratio methods are intractable here.} KMM and KLIEP estimate importance weights by solving a kernel optimization problem whose complexity scales with the number of target samples. With roughly 42k target selection set patches in a 768-dimensional DINOv2 feature space, the kernel matrix (source, target) has billions of entries, causing memory overflow even on high-RAM instances. Subsampling the target to 2000 patches alleviates the memory issue but renders the density ratio estimate unreliable in 768 dimensions - RBF kernels with median bandwidth become flat (all pairwise distances concentrate around the same value), collapsing weights to near-uniform and providing no useful reweighting signal. This high-dimensional intractability is a known limitation of kernel density ratio methods \citep{sugiyama_direct_2008} and motivates our pseudo-labeling approach, which avoids density estimation entirely.

\paragraph{Results for the first run.} \Cref{tab:early stopping epochs} reports the three selection signals across epochs for the first run. The naive source validation loss keeps decreasing as the model fits the source data, reaching its minimum at epoch 8, by which point the model has overfit to source-specific features. Both the pseudo-target and the oracle target risks are minimized at epoch 1 and are much larger at epoch 8, so that pseudo-labeling selects the oracle checkpoint. On the target test set, the selected checkpoints reach a negative log-likelihood (NLL) of $1.942$ and an accuracy of $0.531$ for naive validation, against $0.655$ and $0.732$ for both pseudo-labeling and the oracle.

\begin{table}[ht]
  \caption{First run: selection signals across epochs (naive: source validation NLL on $\cD_2$; pseudo: pseudo-target risk on $\cD_{\text{sel}}$; oracle: target NLL on $\cD_{\text{sel}}$).}
  \label{tab:early stopping epochs}
  \centering
  \small
  \begin{tabular*}{\columnwidth}{@{\extracolsep{\fill}} lcccccccc @{}}
    \toprule
    Epoch & 1 & 2 & 3 & 4 & 5 & 6 & 7 & 8 \\
    \midrule
    Naive  & 0.111 & 0.071 & 0.056 & 0.057 & 0.041 & 0.045 & 0.045 & \textbf{0.040} \\
    Pseudo & \textbf{0.351} & 0.406 & 0.559 & 0.800 & 0.527 & 0.479 & 0.510 & 1.269 \\
    Oracle & \textbf{0.659} & 0.968 & 1.162 & 1.318 & 1.213 & 1.128 & 1.102 & 1.953 \\
    \bottomrule
  \end{tabular*}
\end{table}

\paragraph{Stability across training seeds.} We separate two independent sources of randomness: a \emph{training seed}, controlling the head initialization, the minibatch order and the data augmentation; and a \emph{split seed}, controlling the bisection of hospital 5 into $\cD_{\text{sel}}$ and $\cD_{\text{test}}$. cuDNN is deterministic and the data loader workers are seeded, so that each pair of seeds is exactly reproducible. We repeated the full fine-tuning and selection procedure for two additional training seeds. \Cref{tab:early stopping seeds} in \cref{sec: wilds} reports the selected epochs, the target test NLL, and the regret relative to oracle selection, $\mathrm{NLL}(\text{method})-\mathrm{NLL}(\text{oracle})$.

Exact epoch-level agreement with the oracle is not robust to the initialization: pseudo-labeling selects epochs 1, 2 and 13, against oracle epochs 1, 5 and 2. Nevertheless, the practically relevant conclusion is stable. Naive source validation always selects late epochs (8, 18 and 21) at which the model has overfit the source, whereas pseudo-labeling achieves a lower target NLL than naive validation in every run, and its mean regret is more than four times smaller ($0.39$ versus $1.66$). For Run 3, the target test accuracies are $0.695$ (naive), $0.757$ (pseudo-labeling) and $0.861$ (oracle).

\paragraph{Stability across target splits.} Keeping the imputer and the checkpoints of Run 3 fixed, we redrew the bisection of hospital 5 for 10 split seeds. This axis matters much less than the training seed, consistent with the size of each half ($\approx 42$k patches): the pseudo-labeling stopping epoch lags the oracle one by exactly 11 epochs in all 10 splits, the regret, measured on the selection half, is $0.399\pm0.004$ (range $[0.394, 0.405]$), and the Spearman rank correlation between the pseudo and oracle risk curves across epochs is $0.746\pm0.012$ (mean $\pm$ standard deviation).

\paragraph{Scope of this experiment.} As stated in \cref{sec: wilds}, this experiment lies outside the scope of our theory. We do not attempt to derive a bound for this setting, as we have not carried the conditional-shift case through the estimator analysis. Note also that, with a high-capacity pretrained representation, approximation error is unlikely to be the binding constraint: what the theory idealizes here is the shift structure rather than the model class.

\end{appendices}


\end{document}